\documentclass[11pt]{article}

\usepackage[utf8]{inputenc}
\usepackage[T1]{fontenc}
\usepackage{hyperref}
\usepackage{url}
\usepackage{booktabs}
\usepackage{amsfonts}
\usepackage{microtype}
\usepackage{xcolor}
\usepackage{framed}
\usepackage{float}

\usepackage{tikz}
\usetikzlibrary{patterns, arrows.meta}
\usepackage{listings}
\usepackage{pdflscape}
\usepackage{epsfig}
\usepackage{enumitem}

\setlist[itemize]{leftmargin=*}
\usepackage{framed}
\usepackage{lscape}
\usepackage{longtable}
\usepackage{amsmath,amssymb,amsthm}
\usepackage{graphicx}
\usepackage{subfigure}
\usepackage{wrapfig}
\usepackage{placeins}
\usepackage{cases}
\usepackage{verbatim}
\usepackage{multirow}

\usepackage{algorithm}
\usepackage[noend]{algorithmic}
\usepackage{thmtools}

\usepackage{cleveref}

\oddsidemargin
0pt
\evensidemargin
0pt
\marginparwidth
10pt
\marginparsep
10pt
\topmargin
-20pt

\textwidth
6.5in
\textheight
8.5in
\parindent
=
20pt

\definecolor{lightbluebg}{RGB}{235,245,255}
\definecolor{lightorangebg}{RGB}{245,225,210}

\newcommand{\R}{\mathbb{R}}
\newcommand{\Z}{\mathbb{Z}}

\DeclareMathOperator*{\argmin}{argmin}

\def\approxleq{ \kern3pt \mbox{\raisebox{.6ex}{$<$}} \kern-8pt
  \mbox{\raisebox{-.6ex}{$\sim$}} \kern5pt}

   \def\cT{{\cal T}}
     
\def\cQ{{\cal Q}} \def\cM{{\cal M}}

\def\Diag{\textup{Diag}}

\def\dist{{\textup{dist}}}

\newlength{\len}
\declaretheoremstyle[bodyfont=\sl]{normalbody}

\newtheorem{theorem}{Theorem}

\newtheorem{lemma}{Lemma}
\newtheorem{corollary}{Corollary}
\newtheorem{remark}{Remark}
\newtheorem{assumption}{Assumption}

\newtheorem{example}{Example}
\allowdisplaybreaks[4]

\title{Understanding Quantization-Aware Training: Gradients at Quantized Weights Bias to the Low-Loss Basin}

\author{Hanyang Li
\thanks{Department of IEOR, University of California, Berkeley (\href{mailto:hanyang_li@berkeley.edu}{hanyang\_li@berkeley.edu})}
\and Jianhao Ma
\thanks{Department of Statistics and Data Science, University of Pennsylvania (\href{mailto:jianhao@umich.edu}{jianhao@umich.edu})}
\and Ying Cui
\thanks{Department of IEOR, University of California, Berkeley (\href{mailto:yingcui@berkeley.edu}{yingcui@berkeley.edu})}
}

\date{June 7, 2026}

\begin{document}

\bibliographystyle{plain}
\maketitle

\begin{abstract}
Post-training quantization (PTQ) converts a trained full-precision model into low-bit weights without task-level retraining, while quantization-aware training (QAT) incorporates quantization into the training loop. Although PTQ is efficient and often accurate at moderate bitwidths, it can fail sharply at aggressive bitwidths; QAT is more expensive but can often recover the lost accuracy. We propose a unified geometric framework that explains both PTQ failure and QAT recovery. We model full-precision training as following a low-loss \emph{river} inside a wider \emph{valley}: a normal neighborhood of the river forms a nearly flat \emph{basin}, while leaving this basin incurs a sharp loss increase. When the quantization grid is comparable to the basin width, local PTQ objectives, including rounding and Hessian-based second-order reconstruction, can select a high-loss deployed quantized point outside the basin even when nearby low-loss quantized points exist. In this regime, straight-through-estimator-based QAT has a useful bias: it evaluates gradients at the deployed quantized weights while updating latent full-precision weights, causing the gradient to sense the valley wall and acquire an inward component that steers subsequent quantized iterates back into the basin. We formalize this mechanism through a local landscape model, construct a geometric PTQ failure mode, and prove finite-time QAT recovery under local quantizer-compatibility assumptions. Experiments across vision and language models under multiple neural-network quantization schemes corroborate the predicted basin-crossing failure of PTQ and the corresponding recovery mechanism of QAT.
\end{abstract}

\section{Introduction}

Quantization is a standard technique for reducing inference memory footprint, latency, and energy consumption of large-scale machine learning models \cite{han2016deep,sze2017efficient,jacob2018quantization}. A common deployment pipeline first trains a model in full precision and then converts it to a low-bit representation. The most economical variant is post-training quantization (PTQ), which takes a pretrained model as input and determines quantization scales and discrete weights without task-level fine-tuning, typically via local rounding rules or layer-wise reconstruction objectives \cite{nahshan2021loss,nagel2020up,frantar2023gptq}. Modern PTQ methods are highly effective at moderate bitwidths, particularly around 8 bits, making them attractive for practical deployment \cite{krishnamoorthi2018quantizing,nagel2020up,yuan2022ptq4vit}. However, PTQ often becomes fragile around the 4-bit regime, especially when both weights and activations are quantized, where even well-trained full-precision models can suffer sharp performance degradation \cite{catalan2025training,liu2025paretoq,dremov2026compute}.

A common remedy is quantization-aware training (QAT), which incorporates
quantization into the training loop.
In standard QAT based on the straight-through estimator (STE), training keeps a
latent full-precision copy of the weights, but evaluates the forward pass using
their quantized values. The gradient is therefore measured at the quantized
weights, while the update is applied to the continuous full-precision weights.
This STE update bypasses the fact that quantization is discrete and
has zero or undefined derivatives almost everywhere
\cite{bengio2013estimating,courbariaux2015binaryconnect,jacob2018quantization}.
\footnote{There are many variants of QAT beyond STE-based methods. In this
paper, we focus on the standard STE-based formulation.} Practitioners often
view PTQ and QAT as occupying different points on a cost--accuracy frontier:
PTQ is fast, calibration-efficient, and easy to apply, but can incur substantial
accuracy degradation at low precision, whereas QAT is usually more robust in
low-precision regimes but requires additional data, optimizer state,
hyperparameter tuning, and training compute. This tradeoff has motivated
practical pipelines in which PTQ is first used as a low-cost quantization
attempt; if the resulting model does not meet the target accuracy, practitioners
switch to QAT fine-tuning to recover performance under the quantized forward
pass.

Despite the practical success of PTQ--QAT pipelines, their underlying mechanism remains poorly understood. When can PTQ fail sharply even after full-precision training has found a good model? What does full-precision pretraining contribute before QAT begins? And, under the same total training budget, why should QAT fine-tuning differ from continuing full-precision training and then applying PTQ again? These questions matter for both theory and practice: if QAT merely provides additional optimization steps, then its advantage should be reproducible by equal-budget full-precision fine-tuning followed by PTQ; if instead QAT corrects quantization-specific errors that remain invisible to full-precision training, then QAT has a quantization-specific benefit: it adapts the latent full-precision weights to perform well after quantization, which PTQ alone may fail to achieve.

\begin{figure}[t]
\centering
\subfigure[River--valley--basin landscape.]{
    \includegraphics[width=0.35\linewidth]{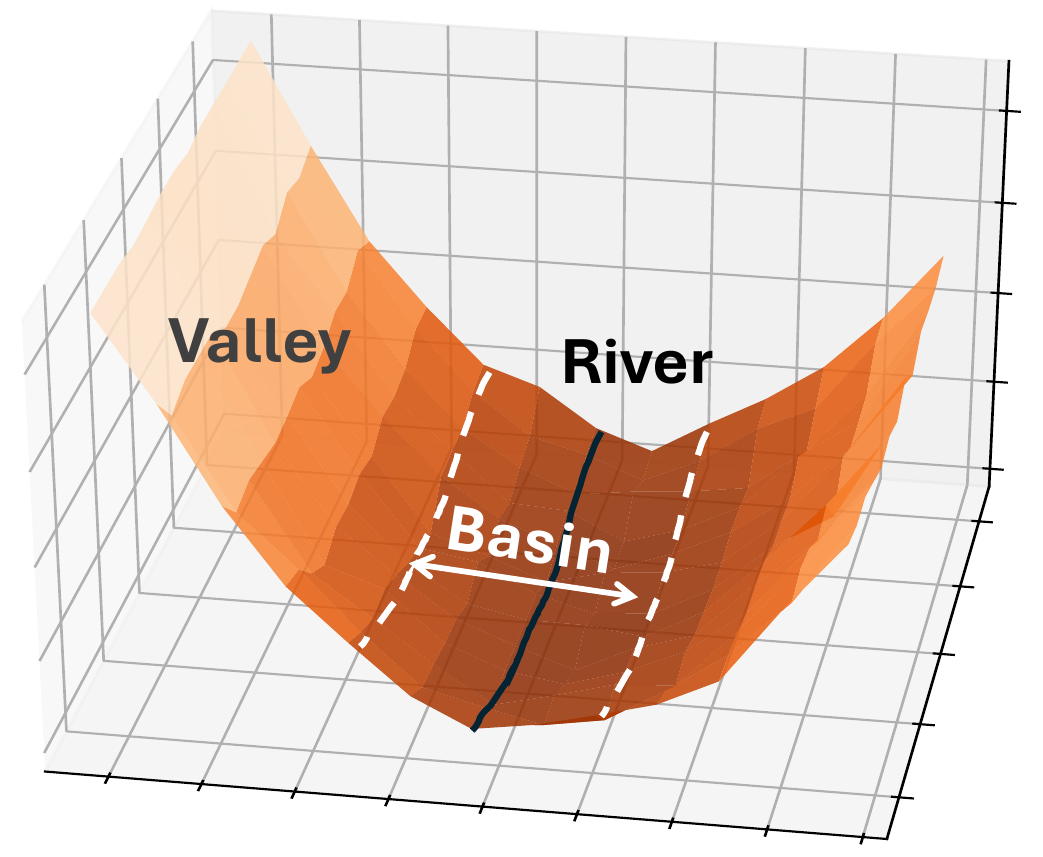}
    \label{fig:intro-river-basin}
}
\hspace{0.8cm}
\subfigure[Quantization compatibility.]{
    \includegraphics[width=0.35\linewidth]{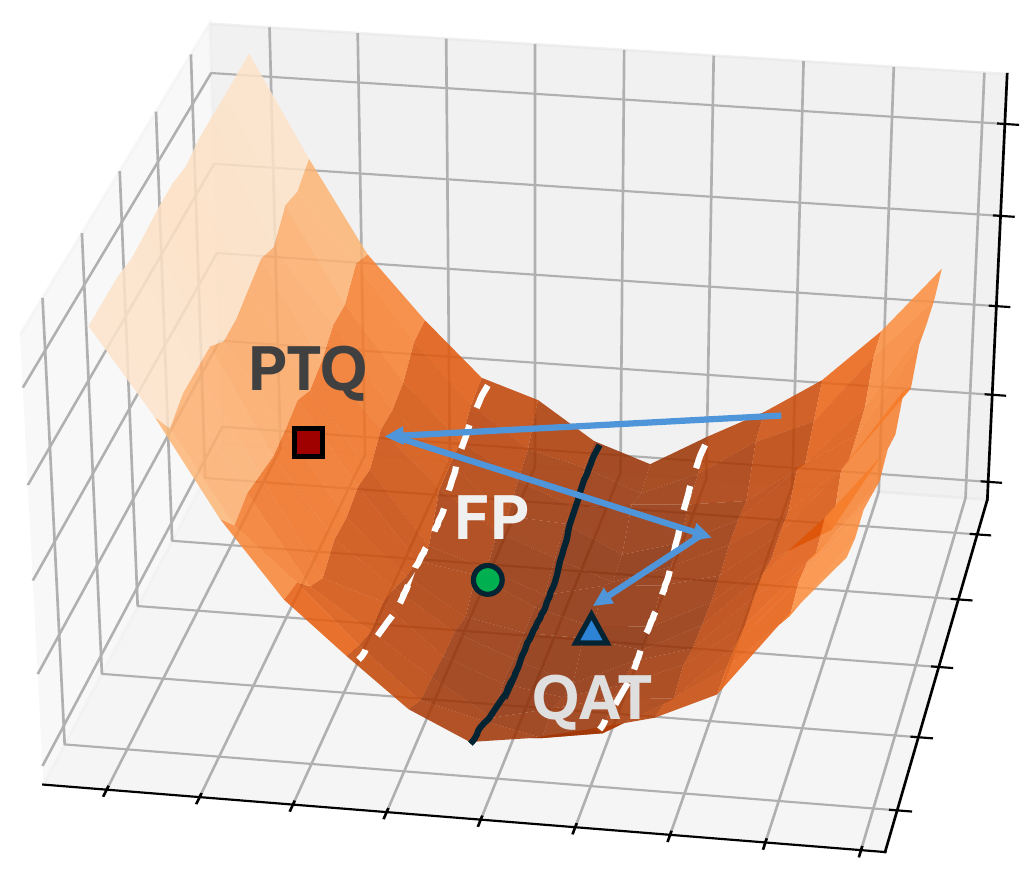}
    \label{fig:intro-ptq-qat}
}
\caption{
River--valley--basin geometry of the loss and quantization compatibility induced by the gradient bias of QAT.
(a) A loss landscape from a real 213M Llama model trained on the SlimPajama dataset, showing a low-loss river inside a broader valley and basin.
(b) Illustration of the PTQ--QAT contrast: PTQ may round the full-precision solution to a quantized point outside the low-loss basin, whereas QAT can move the quantized model back into the basin.
}
\label{fig:intro-river-landscape}
\vspace{-0.5em}
\end{figure}

We argue that these questions can be understood through \emph{quantization
compatibility}: whether the deployed low-bit weights  sit in a region of parameter space where loss remains low
relative to the pretrained full-precision point.
Our analysis is built on a \emph{river-valley-basin} view of the loss
landscape (see Figure~\ref{fig:intro-river-basin}), supported by recent empirical and theoretical evidence~\cite{ma2022beyond,wen2025understanding,chen2025unveiling,begout2026gradient}. 
The \emph{river} is a low-loss manifold or trajectory along which full-precision training can move without large cost increases. Around the river lies a tube-shaped \emph{basin} (dashed lines in Figure~\ref{fig:intro-river-basin}) in which normal perturbations cause little loss change. Outside the basin, \emph{valley walls}
produce steep increases in normal directions. 

PTQ can fail when the deployed quantized point falls outside the low-loss basin and lands on the valley wall, where the loss increases sharply and performance deteriorates, as illustrated in Figure~\ref{fig:intro-ptq-qat}. This can happen when the quantization grid resolution is comparable to the basin width: PTQ may cross the basin boundary even though another nearby quantized point remains inside the basin and has low loss. The mechanism is that the Hessian can be nearly flat inside the basin, so local quadratic surrogates may underestimate the boundary-crossing cost and fail to distinguish basin-compatible quantized points from incompatible ones. By contrast, QAT evaluates gradients at the quantized weights. When these weights fall outside the basin, the gradients probe the steep valley-wall regime and can contain a substantial inward normal component. Updates to the latent full-precision weights through the STE can therefore shift subsequent quantized iterates back toward the low-loss basin, restoring quantization compatibility. Ordinary full-precision fine-tuning, however, may keep improving the unquantized model along the river while leaving the normal quantization error uncorrected. Thus, QAT can help not merely because it uses additional compute, but because it optimizes the deployed quantized model and provides a quantization-specific correction that standard full-precision training does not.
Our contributions are as follows:
\begin{itemize}
    \item We formalize a river--valley--basin landscape together with
    quantization-compatibility conditions. This framework identifies a geometric
    PTQ failure mode: local rounding or second-order reconstruction objectives
    can select a high-loss quantized point outside the low-loss basin, even
    though nearby low-loss quantized points exist. %

    \item We analyze QAT dynamics and show that, when the deployed quantized
    iterate lies outside the basin, the gradient evaluated at the quantized
    weights can contain an inward normal component that drives subsequent
    deployed quantized iterates back into the low-loss basin in finite time.

    \item We support the theory through experiments on low-dimensional
    landscapes, matrix factorization, and neural network quantization tasks,
    testing the predicted basin-crossing failure, inward-gradient correction,
    and QAT recovery behavior.
\end{itemize}

\subsection{Related Work}
\label{sec:related}
\paragraph{PTQ and QAT.}
 Early and widely used PTQ pipelines rely on
rounding, calibration, and per-layer or per-channel scaling~\cite{krishnamoorthi2018quantizing,banner2019post}. More advanced methods
choose rounding decisions by minimizing local reconstruction objectives: loss-aware
PTQ optimizes quantization decisions with a local loss proxy~\cite{nahshan2021loss},
AdaRound learns rounding decisions through a relaxed reconstruction
objective~\cite{nagel2020up}, and GPTQ applies a layerwise second-order
reconstruction procedure for transformers~\cite{frantar2023gptq}. 
Architecture-specific PTQ methods have
also been developed for vision transformers, including twin uniform quantization
in PTQ4ViT~\cite{yuan2022ptq4vit} and Hessian-based reconstruction in
APHQ-ViT~\cite{wu2025aphq}. Unlike PTQ, QAT includes
quantization in the training loop, typically by applying quantization in the
forward pass and using an STE in the backward pass \cite{bengio2013estimating,courbariaux2015binaryconnect,jacob2018quantization}. Empirically, QAT is often more
robust at low precision, but it is more expensive than PTQ. Recent work studies
hybrid PTQ--QAT pipelines and compute allocation between full-precision training
and QAT fine-tuning \cite{catalan2025training,liu2025paretoq,dremov2026compute}. Our work is
complementary: rather than proposing a new quantization algorithm, we study a
mechanism explaining when and why QAT can correct a PTQ failure.

\paragraph{Theory of straight-through estimators.}
Theoretical understanding of STE-based training remains limited. Existing analyses
typically impose structural assumptions on the quantizer, architecture, objective,
or data distribution. For binary weights, Li et al.~\cite{li2017training} analyze
STE-based SGD for finite-sum objectives and show convergence of averaged iterates
to a neighborhood of the minimizer under strong convexity and smoothness
conditions. Yin et al.~\cite{yin2019understanding} study STE for two-layer
networks with binary activations under Gaussian inputs and show that the expected
coarse gradient can be a descent direction for the population objective. More
recent work extends aspects of this analysis beyond population objectives
\cite{jeong2025beyond}. A complementary line of work gives optimization and
regularization interpretations of quantized training and STE-like updates.
Bai et al.~\cite{bai2018proxquant} relate straight-through training to dual
averaging for quantization-constrained optimization and propose ProxQuant, which
formulates quantized training as regularized learning solved by proximal gradient
updates. Dockhorn et al.~\cite{dockhorn2021demystifying} further interpret
BinaryConnect through dual averaging and generalized conditional gradient
methods, and introduce ProxConnect using proximal maps as a principled class of
quantizers. More recently, PARQ~\cite{jin2025parq} uses convex piecewise-affine
regularization to induce weights toward discrete values, optimizes the resulting
objective with an aggregate proximal stochastic-gradient method, and shows that
STE can be viewed as an asymptotic form of this regularized proximal framework.
These results clarify when STE updates can behave as useful biased gradients, or
how STE-like quantized training can be interpreted through optimization
principles. They do not, however, directly address the PTQ--QAT pipeline: why a
pretrained full-precision model may be fragile under quantization, why QAT should
start from such a model, or why QAT may be more useful than equal-budget
full-precision fine-tuning. Our analysis focuses on these pipeline-specific
questions.

\paragraph{Neural network loss landscapes.}
Our landscape model builds on recent work on anisotropic geometric structures in neural network loss functions. Ma et al.~\cite{ma2022beyond} argue that loss landscapes exhibit
multiscale structure with subquadratic growth rates along some directions, so a single local quadratic approximation can miss
qualitatively different behavior at larger scales.  
Wen et al.~\cite{wen2025understanding} propose a river-valley model to explain
warmup-stable-decay learning rate schedules: high learning rates make progress
along the river while oscillating across sharp valley walls, and learning rate
decay reduces the off-river component. Chen et al.~\cite{chen2025unveiling}
observe basin-like stability regions in large language models, where performance
is nearly unchanged inside a basin but collapses outside it.  These empirical
observations are consistent with optimization-theoretic analyses. In particular, Davis et al.~\cite{davis2025gradient} show that a smooth function with subquadratic growth rate admits a
smooth {ravine} manifold on which 
growth normal to the ravine is at least quadratic near the minimum. Most recently, B\'egout et al.~\cite{begout2026gradient} establish that, near the
solution, gradient descent aligns with the eigenvector corresponding to the
smallest eigenvalue of the Hessian.
We combine these perspectives to study  PTQ and QAT.

\section{River-valley-basin Landscape}
\label{sec:landscape}

Motivated by the empirical and theoretical evidence of the loss landscape of neural networks reviewed in Section \ref{sec:related}, we now introduce a local
model that isolates the geometry relevant to quantization. The model has three
components: a low-loss \emph{river} that represents the full-precision training
trajectory, a surrounding \emph{basin} in which moderate normal perturbations do
not substantially change the loss, and sharper \emph{valley walls} outside this
basin that generate an inward gradient. This abstraction is not intended as a
global description of the full neural-network objective. Rather, it is a local
description near a pretrained point, designed to separate tangential motion along
the trained solution set from normal displacements induced by quantization. This
separation is the key structure used later to explain both PTQ failure and QAT
recovery.

We first introduce the local geometric notation. Let $w_{\rm fp}$ be the
full-precision pretrained point, and let $\cM\subset\R^d$ be an
$m$-dimensional $C^{2,1}$ embedded manifold. For each $\pi\in\cM$, denote by
$T_{\cM}(\pi)$ and $N_{\cM}(\pi)$ the tangent and normal spaces at $\pi$.
Let $\mathsf H$ be a $d\times d$ dimensional symmetric positive definite matrix and define
$\|x\|_{\mathsf H}=(x^\top \mathsf Hx)^{1/2}$. Denote the maximal and minimal eigenvalues of $\mathsf H$ by $\lambda_1$ and $\lambda_d$. We assume $\mathsf H N_{\cM}(\pi)\subseteq N_{\cM}(\pi)$ for all $\pi\in\cM$, so that the weighted normal metric does not mix normal and tangential directions.
Let $U=\{w\in\R^d \mid \dist(w,\cM)<R_U\}$ be a tubular neighborhood of $\cM$
containing $w_{\rm fp}$, on which the nearest-point projection
$P_{\cM}: U\to\cM$ is well defined and assumed to be $L_P$-Lipschitz continuous. 
For a radius $R>0$, define the anisotropic normal tube
$\cT(\pi)=\{\pi+\xi \mid \xi\in N_{\cM}(\pi),\ \|\xi \|_{\mathsf H} \le R\}$ and
$\cT=\bigcup_{\pi\in\cM}\cT(\pi)$. Equivalently,
$\cT=\{w\in U \mid \|w-P_{\cM}(w)\|_{\mathsf H} \le R\}$.

\begin{assumption}[Local river-valley geometry]\label{ass:river-geometry}
The function $f$ is $L$-smooth on $U$.  Moreover, there exist constants
$\epsilon_{\rm flat}$ and $c_\perp>0$ such that the following hold.
\begin{itemize}
    \item {\bf (Anisotropic Basin)}
    For any $\pi\in\cM, w \in \cT(\pi)$, we have $
    |f(w)-f(\pi)| \le \epsilon_{\mathrm{flat}}$.
    \item {\bf (River center)}
    For every $\pi \in \cM$, it holds that $\nabla f(\pi) \in T_{\cM}(\pi)$.
    \item {\bf (Sharp valley outside the basin)}
    For every $w \in U \setminus \cT$, defining the weighted normal direction $\nu_{\mathsf H}(w) = \mathsf H (w-P_{\cM}(w))/\|\mathsf H(w-P_{\cM}(w))\|$, we have
 $\langle \nabla f(w), \nu_{\mathsf H}(w)\rangle \ge c_{\perp}$.
\end{itemize}
\end{assumption}

The \emph{Anisotropic Basin} models the basin-like near-flatness observed in neural-network
loss landscapes: within the anisotropic tube $\cT(\pi)$ of radius $R$ around a
river point $\pi\in\cM$, the loss changes by at most
$\epsilon_{\mathrm{flat}}$. The metric $\mathsf H$ encodes anisotropic normal
scaling, allowing some normal directions to be flatter than others; this is
consistent with multiscale landscape structure~\cite{ma2022beyond} and with
results that distinguish mild growth along a ravine from sharper transverse
growth~\cite{davis2025gradient}. The \emph{River Center} condition requires the gradient on
$\cM$ to be tangential, formalizing the idea that full-precision training can
continue moving along a low-loss river~\cite{wen2025understanding} and relating
to talweg- and direction-based descriptions near minimizers~\cite{begout2026gradient}.
The \emph{Sharp Valley} condition imposes sharp valley walls outside the basin: for every
$w\in U\setminus\cT$, the normal pairing
$\langle\nabla f(w),\nu_{\mathsf H}(w)\rangle$ is uniformly bounded below by
$c_{\perp}>0$. This condition captures the sharp regime beyond basin stability
observed in large language models~\cite{chen2025unveiling} and provides the
geometric input used later to prove finite-time recovery of the quantized
iterate.

\begin{figure}[htp]
\centering
\subfigure[ResNet on CIFAR-10]{\includegraphics[width=0.32\linewidth]{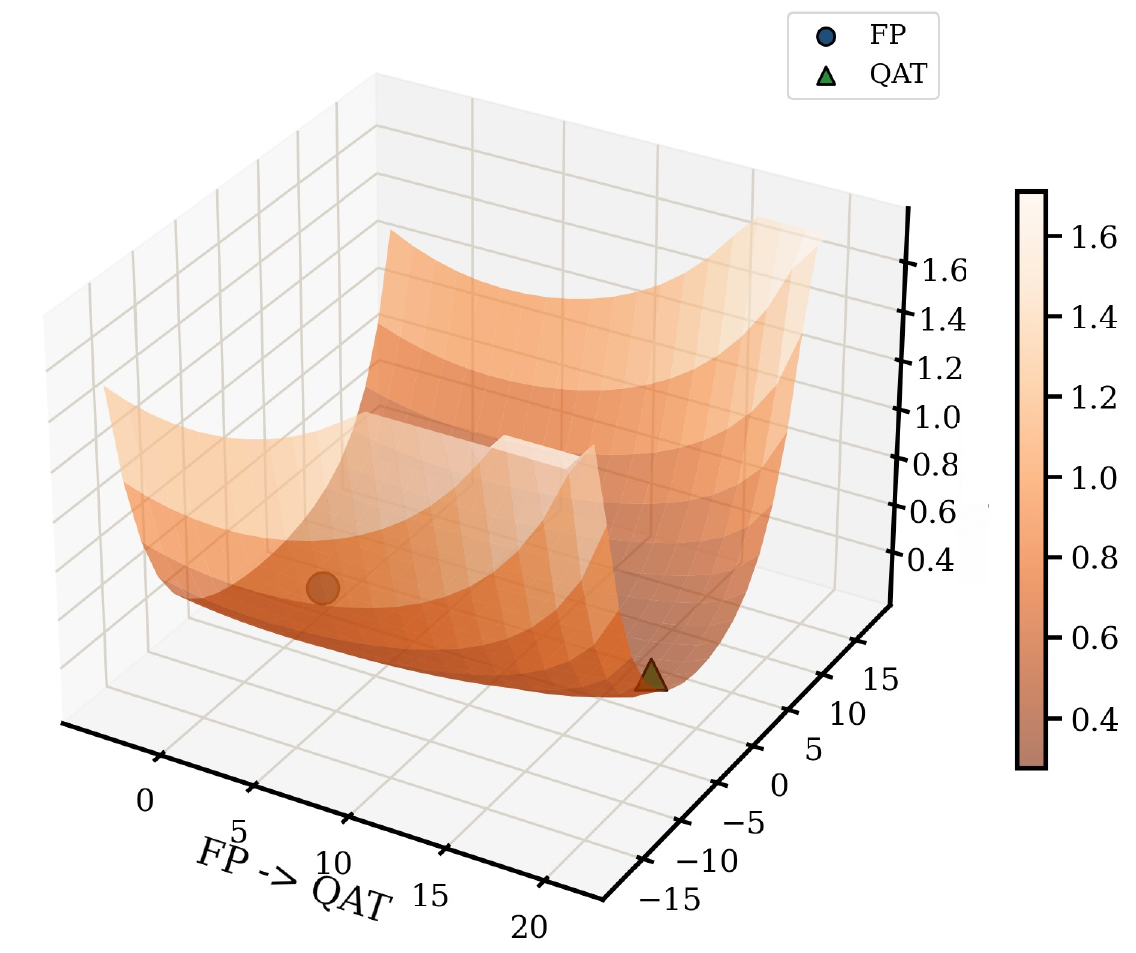}}
\hfill
\subfigure[ViT on ImageNet]{\includegraphics[width=0.31\linewidth]{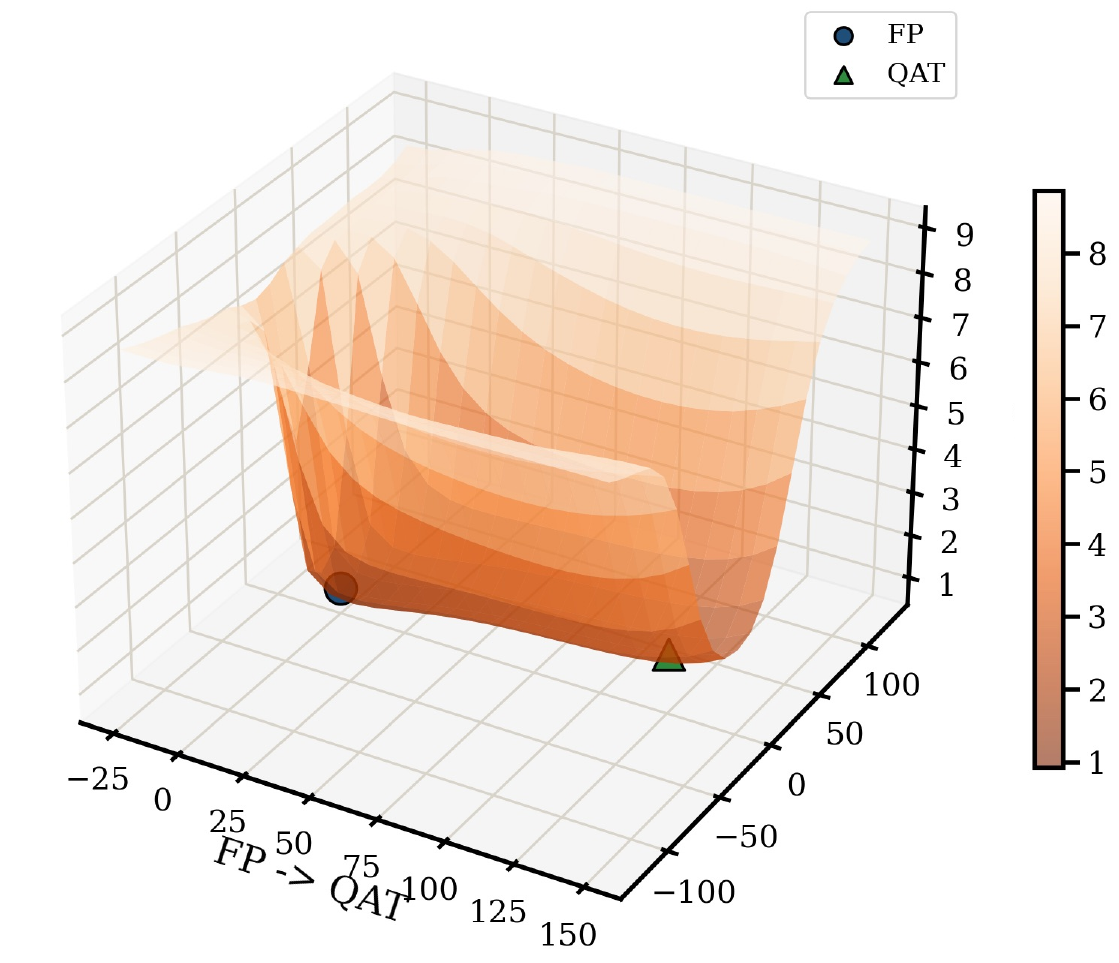}}
\hfill
\subfigure[Llama on SlimPajama]{\includegraphics[width=0.34\linewidth]{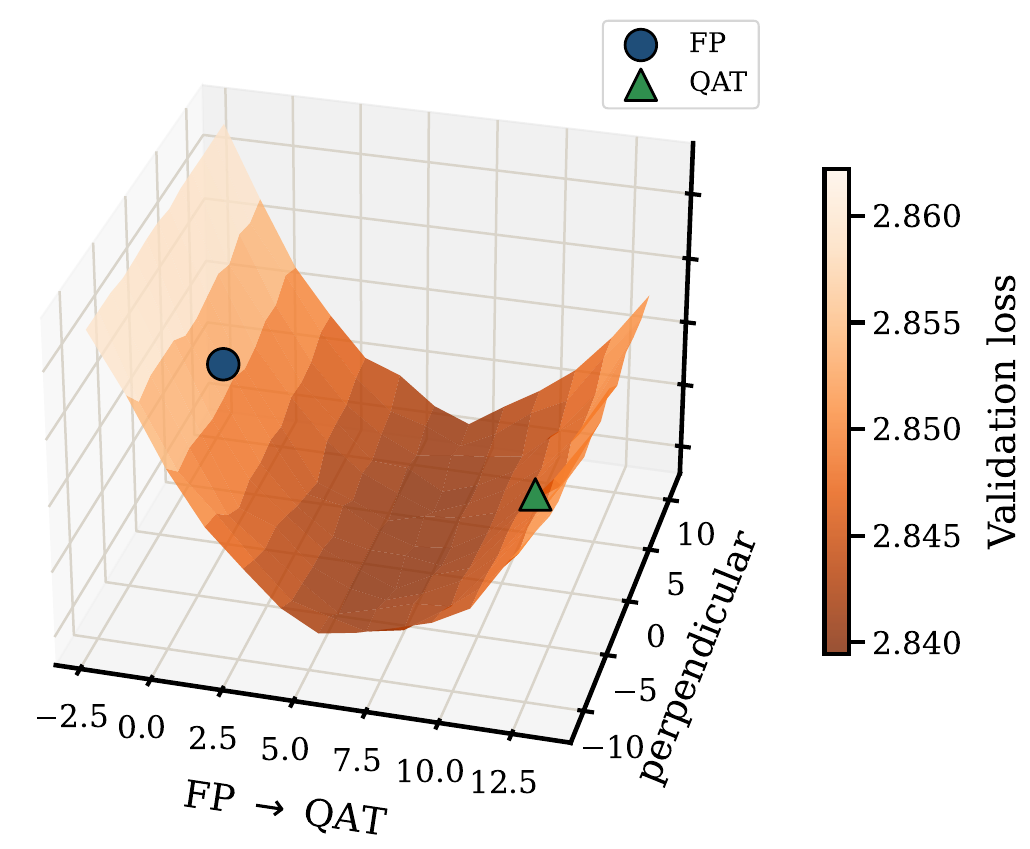}}
\caption{Local landscape diagnostics around pretrained checkpoints.
Each panel plots the loss along one tangential and one normal direction
relative to the QAT trajectory. The profiles are qualitatively
consistent with Assumption~\ref{ass:river-geometry}: a near-flat basin
in the normal direction and sharper growth outside it.}
\label{fig:assumption1-empirical}
\end{figure}

Figure~\ref{fig:assumption1-empirical} gives local landscape diagnostics around pretrained checkpoints for three architectures. The observed profiles are qualitatively consistent with the river-valley-basin geometry of Assumption~\ref{ass:river-geometry}. However, the two-dimensional slices in Figure~\ref{fig:assumption1-empirical} can only probe the \emph{Anisotropic Basin} and \emph{Sharp Valley} conditions along sampled directions; they do not provide a quantitative verification of other conditions such as \emph{River Center}, nor do they determine the constants $\epsilon_{\rm flat}$, $c_\perp$, or $R$ in the full parameter space. 

To complement this empirical evidence, we present two analytically tractable
examples. The first is a two-dimensional landscape that makes the
river-valley-basin geometry visually transparent. The second is an
over-parameterized matrix factorization problem, which shows why an anisotropic
normal metric naturally arises: some normal directions enter the loss
quadratically, while others appear only through higher-order terms. For both
examples, the detailed verification of Assumption~\ref{ass:river-geometry} and
the explicit constants are deferred to Appendices~\ref{app:toy-example-assum1}
and~\ref{app:mf-assum1}.
\vspace{0.05in}

\noindent
\begin{minipage}[c]{0.67\linewidth}
\begin{example}[A two-dimensional river-valley-basin loss]
\label{ex:simple-running}
Let $w=(x,y)\in\mathbb R^2$. Fix parameters $u\in\R$, $\mu, r, \epsilon>0$, and define
\[
    f(w)=\tfrac12 (x+u)^2+\tfrac{\mu}{2}\bigl(\max\{|y|-r, 0\}\bigr)^2 .
\]
The river center is the horizontal line $\cM=\{(x,0)\mid x\in\mathbb R\}$,
and the basin is $\cT=\{(x,y)\in U\mid |y|\le R\}$ with $R =  r+\sqrt{2\epsilon/\mu}$ and $U=\{(x,y)\in\mathbb R^2\mid |y|<R_U\}$ for any $R_U > R$.
This example satisfies Assumption~\ref{ass:river-geometry} on $U$ with
$\mathsf H=\mathbb{I}$, $L_P=1$, $L=\max\{1,\mu\}$, $\epsilon_{\rm flat}= \epsilon$, and $c_\perp=\sqrt{2\mu\epsilon}$.
\end{example}
\end{minipage}\hfill
\begin{minipage}[c]{0.32\linewidth}
\centering
\includegraphics[width=\linewidth]{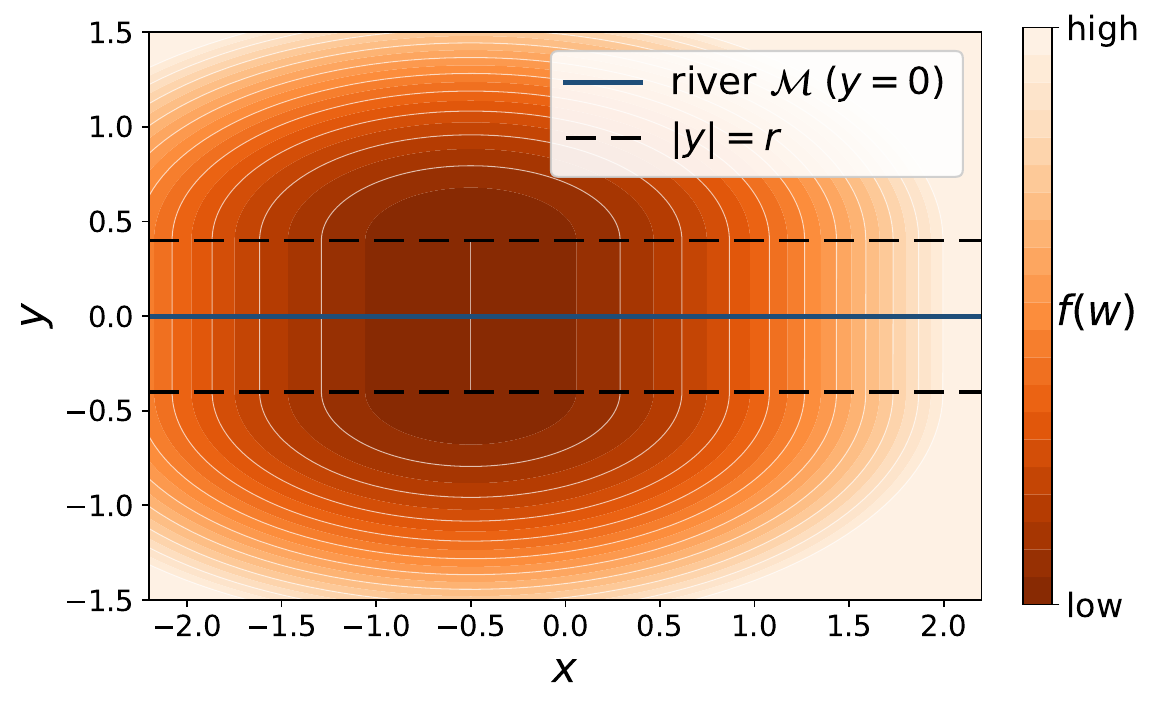}
\end{minipage}

\vspace{0.05in}
Notice that for fixed $\epsilon_{\rm flat}=\epsilon$, increasing $\mu$ steepens the valley $c_{\perp} \propto \sqrt{\mu}$. That motivates calling $\mu$ the sharpness parameter of $f$. In the next section, we use a variant of this example to illustrate why PTQ fails while QAT succeeds under the geometry of Assumption~\ref{ass:river-geometry}.
The matrix factorization example below shows that the same geometry also appears in a standard over-parameterized nonconvex model, but with a genuinely anisotropic normal basin.

\begin{example}[Over-parameterized matrix factorization]
\label{ex:mf_sec2}
Assume $M^\star=\Diag(D,\mathbf 0)\in\R^{d\times d}$, where
$D\in\R^{r\times r}$ is positive definite and zero blocks have conformal
dimensions. Consider
\[
    f(X)=\|XX^\top-M^\star\|_F^2
\]
with $X=\binom{P}{Z}$, where $P\in\R^{r\times k}$,
$Z\in\R^{(d-r)\times k}$, and $k\ge r$.
Fix a full-row-rank matrix $P_0$ and a bounded open neighborhood
$\Omega$ of $P_0$ such that all $P$ in $\overline\Omega$ remain full-row-rank. Define
\[
    \cM=\left\{\binom{P}{\mathbf 0} \,\middle|\, P\in\Omega\right\},
    \qquad
    U=\left\{\binom{P}{Z} \,\middle|\, P\in\Omega,\ \|Z\|_F<R_U\right\}.
\]
Expanding the objective gives
$
    f(P,Z)
    =
    \|PP^\top-D\|_F^2
    +2\|PZ^\top\|_F^2
    +\|ZZ^\top\|_F^2
$.
Thus the normal geometry is anisotropic: the part of $Z$ seen by the row space of
$P$ contributes quadratically through $\|PZ^\top\|_F^2$, whereas the row-null part
first appears through the quartic term $\|ZZ^\top\|_F^2$. Freezing this row-space
decomposition at $P_0$, one can choose a fixed positive definite metric
$\mathsf H$ that weights the quadratic normal directions more heavily than the
quartic ones. After possibly shrinking $\Omega$, Assumption~\ref{ass:river-geometry}
holds on $U$ with an anisotropic basin $\cT$, tolerance $\epsilon_{\rm flat}=\epsilon$, and some $c_\perp>0$. A precise construction of $\mathsf H$, $\cT$, and $c_\perp$ is given in Appendix~\ref{app:mf-assum1}.
\end{example}

\section{Quantization under the River-valley-basin Landscape}
\label{sec:Quant_under_River-valley-basin}

In this section, we analyze both PTQ and QAT under the river--valley--basin landscape introduced
in Section~\ref{sec:landscape}. We first study how PTQ can fail when its deployed codeword
crosses the basin boundary, then show how STE-based QAT can recover by updating
latent full-precision weights using gradients evaluated at the quantized model. 

\subsection{Failure of Hessian-based PTQ}
\label{sec:ptq-failure}

Many PTQ methods are motivated by a local second-order approximation around a
fixed full-precision model. Given a pretrained point $w_{\mathrm{fp}}$ and a
quantization codebook $\mathcal Q$, the idealized Hessian-based PTQ objective is
\begin{equation}\label{eq:hessian-ptq}
    q_{\ast}
   =
    \argmin_{q \in \mathcal Q}
    \tfrac{1}{2}(q - w_{\mathrm{fp}})^\top
    \nabla^2 f(w_{\mathrm{fp}})
    (q - w_{\mathrm{fp}}),
\end{equation}
which is the
second-order Taylor proxy for the loss function, with the linear term omitted because
$\nabla f(w_{\mathrm{fp}})$ is small near a well-trained model. In practice, directly solving \eqref{eq:hessian-ptq} with the exact Hessian is typically computationally impractical: the Hessian is often replaced by a cheaper approximation, and the resulting discrete optimization over the quantization codebook is handled through relaxations or greedy/blockwise heuristics. For example,
AdaRound~\cite{nagel2020up} derives a quadratic rounding objective from a Taylor expansion and then
optimizes a layerwise reconstruction relaxation, while GPTQ~\cite{frantar2023gptq} uses a layerwise
output-reconstruction Hessian estimated from calibration data, together with
damping and blockwise inverse-Hessian compensation. Thus, we treat \eqref{eq:hessian-ptq} as a stylized model representing the core Hessian-based principle underlying these PTQ techniques. The failure mode we describe below, however, is not limited to this exact global objective; rather, it captures a broader limitation inherent to local PTQ surrogates.

When $w_{\mathrm{fp}}$ lies inside the flat basin, this quadratic proxy can be highly misleading.
Within this band, the loss remains nearly constant in the normal direction, meaning the Hessian $\nabla^2 f(w_{\mathrm{fp}})$ has very small eigenvalues along that direction.
As a result, the proxy treats normal-direction displacements as inexpensive, even though crossing the band boundary on the true landscape incurs a sharp increase in loss.
At the same time, the Hessian may have substantial curvature along the tangent direction of the river, causing the proxy to penalize tangent movements toward downstream grid points that actually have low true loss. This also explains why simply spending more budget on
ordinary full-precision training need not fix the PTQ failure. Once the
full-precision iterate is inside the flat band, its gradient may mainly improve
the along-river coordinate and need not provide an inward normal correction for
the deployed quantized point. As a result, after additional full-precision
steps, the final PTQ step can still select a quantized codeword outside the
low-loss basin.

We rotate the function in Example~\ref{ex:simple-running} to show that the selected point of \eqref{eq:hessian-ptq} can have arbitrarily larger loss than a nearby grid point inside the basin.

\vspace{0.2cm}
\noindent
\begin{minipage}[c]{0.68\linewidth}
\paragraph{Revisiting \Cref{ex:simple-running}: failure of Hessian-based PTQ.}
For any $\theta\in[\pi/6,\,\pi/4)$, let $R_\theta = \bigl(\begin{smallmatrix} \cos\theta & \sin\theta \\ -\sin\theta & \cos\theta \end{smallmatrix}\bigr)$ denote the clockwise rotation matrix by $\theta$ and set $w_0=(1,0)$.  Define
$
    f_\theta(w) = f(R_\theta(w-w_0))
$,
where $f$ is the loss in Example~\ref{ex:simple-running} with $\alpha=\cos\theta-\sin\theta$, $u\in(\frac{\alpha}{2},\alpha)$, and $R < \cos\theta+\sin\theta$.
The river center becomes
$\cM=\{w_0+\lambda(\cos\theta,\sin\theta) \mid \lambda\in\R\}$ and the basin
$\cT=\{w\mid\dist(w,\cM)\le R\}$.
Consider uniform quantizer $\cQ=\rho\{0,\pm1,\cdots,\pm(2^{B-1}\!-\!1)\}^2$ for $\rho=1$ and $B\ge 2$.
Fixing $w_{\rm fp}$ to be any stationary point of $f_\theta$ at which the Hessian exists, the minimizer $q_{\ast}$ of the Hessian proxy~\eqref{eq:hessian-ptq} lies outside the basin $\cT$, even though $q_{\rm g}=(1,0)\in\cT$ is nearby. The loss gap $f_\theta(q_{\ast})-f_\theta(q_{\rm g})$ grows linearly in the sharpness parameter $\mu$.
\end{minipage}\hfill
\begin{minipage}[c]{0.30\linewidth}
\centering
\includegraphics[width=\linewidth]{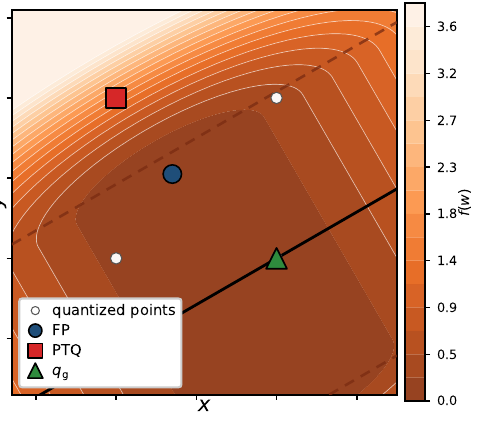}
\end{minipage}

\begin{remark}
(i) This failure holds throughout an open strip of stationary points, not just a single $w_{\rm fp}$. (ii) The unit quantizer scale $\rho=1$ is for ease of presentation. Appendix~\ref{app:toy-example-PTQ} proves the same failure over an open interval of scale, so the failure mechanism is also robust to $\rho$.
\end{remark}

While the Hessian model \eqref{eq:hessian-ptq} is intentionally stylized, it successfully isolates the local second-order logic common to practical PTQ algorithms. The previous example illustrates that this logic incorrectly ranks quantized candidates when the local curvature is nearly flat in the normal direction, completely ignoring the sharp rise in true loss that occurs just beyond the basin boundary.

\subsection{QAT Recovers from PTQ Failure}
\label{sec:ste-recovery}

Next, we demonstrate how QAT can recover from the aforementioned PTQ failure.
Let $w_{\mathrm{fp}}$ denote the full-precision pretrained checkpoint. Starting from
$w_0 = w_{\mathrm{fp}}$, we consider the following STE update:
\[
    w_{k+1} = w_k - \eta \nabla f(Q(w_k)).
\]
Here the gradient is evaluated at the quantized weights $Q(w_k)$ but applied to
the underlying full-precision weights $w_k$. After $T$ iterations, the deployed
model is the quantized checkpoint $Q(w_T)$.

To proceed, we introduce the assumptions on the quantizer. 
Recall the notation for $\mathsf{H}$ and $P_{\cM}$ in  Assumption \ref{ass:river-geometry}. 
We denote the maximal and minimal eigenvalues of $\mathsf H$ by $\lambda_1$ and $\lambda_d$. 
\begin{assumption}[Quantizer compatibility]
\label{ass:quantizer-compatibility}
There exist constants $\rho$, $G$, $\kappa_{\cM} > 0$ such that: \\[0.05in]
{(i)} \textbf{(Quantizer rounding error)} For every $w\in U$, $\|Q(w) - w\| \le \rho$. \\[0.05in]
(ii) \textbf{(Projection regularity and bounded gradients)} For every $w \in U$, $\|\nabla f(w)\| \leq G$. Moreover, $P_{\cM}$ is $C^{1,1}$ on $U$, i.e., there exists $\kappa_{\cM} > 0$ such that
    \[
        \|\nabla P_{\cM}(w) - \nabla P_{\cM}(w^\prime)\|_{\rm op}
        \le \kappa_{\cM} \|w-w^\prime\|,
        \qquad \forall w, w^\prime \in U.
    \]
(iii) \textbf{(River width)}
    $
        \rho(1+L_P)\sqrt{\lambda_1} \max\left\{\frac{G}{c_{\perp}} \sqrt{\frac{\lambda_1}{\lambda_d}}, 1\right\}
        \,<\, R
        \,<\, \sqrt{\lambda_d} (R_U - \rho),
    $ where the parameters $L_P$, $c_\perp$, $R$ and $R_U$ are defined in Assumption \ref{ass:river-geometry}.
\end{assumption}

\begin{theorem}[Recovery from PTQ failure under STE]\label{thm:QAT-recovery}
Suppose Assumptions~\ref{ass:river-geometry} and \ref{ass:quantizer-compatibility} hold. Assume $w_0 = w_{\mathrm{fp}} \in \cT$ and $Q(w_0) \notin \cT$. 
Let $T = \inf\{k\ge 0 \mid Q(w_k)\in\mathcal T\}$ be the first time the quantized iterate enters $\cT$.
If 
$
    \eta \le \min\left\{\frac{c_\perp \sqrt{\lambda_d}R - \lambda_1(1+L_P)\rho G}{\lambda_1(1+L_P)G^2}, \frac{\rho}{G}\right\}
$, then we have
\[
    T \le
    1 + \frac{\max\left\{\|w_0 - P_{\cM}(w_0)\|^2_{\mathsf H} - \left(R-\sqrt{\lambda_1}\rho(1+L_P) \right)^2, \; 0\right\}}
    {\eta \left( c_{\perp}\sqrt{\lambda_d}R - \lambda_1(1+L_{P}) \rho G\right)}.
\]
\end{theorem}

The first-entry-time bound in Theorem~\ref{thm:QAT-recovery} gives a cautious interpretation of scaling and bitwidth effects.  If larger models have wider effective low-loss basins, as suggested by recent basin-visualization studies~\cite{chen2025unveiling}, then the bound predicts a shorter recovery phase, provided the other constants in the drift margin do not deteriorate.  This offers a possible geometric explanation for empirical observations that larger models can be more tolerant to aggressive quantization~\cite{dremov2026compute}. Conversely, for a fixed quantizer family and dynamic range, lowering the bitwidth typically increases the quantization scale $\rho$, and the bound then predicts a longer first-entry time, and in the extreme case the drift margin may become nonpositive, in which case this theorem no longer guarantees recovery.

To further interpret the quality of the recovered point, we
separate into two regimes. 

\begin{corollary}[Near-optimal pre-trained solution]
\label{cor:QAT_loss1}
    Assume the conditions of Theorem~\ref{thm:QAT-recovery}. Suppose additionally that $\|\nabla f(\pi)\| \le \epsilon$ for all $\pi \in \cM$.
    Defining $\Delta =\kappa_{\mathcal M} \left(R/\sqrt{\lambda_d} + \frac12\eta G\right)$, we have
    \[
        f(Q(w_T))
        \le f(w_{\rm fp})
        + 2\epsilon_{\rm flat}
        + \underbrace{\epsilon\rho\bigl(\kappa_{\mathcal M}(R/\sqrt{\lambda_d} + \rho)+1\bigr)}_\text{final quantization error at $w_T$}
        + \underbrace{\eta T G (\Delta+1) \left(\epsilon + \tfrac{L}{2} \eta G (\Delta+1)\right)}_\text{accumulated tangent error along the river}.
    \]
    In particular, since $T=\mathcal O(1/\eta)$, fixing $\eta = \mathcal O(\epsilon)$ gives
    $
        f(Q(w_T)) \le f(w_{\rm fp}) + 2 \epsilon_{\rm flat} + \mathcal O(\epsilon(\rho+1))
    $.
\end{corollary}

Corollary~\ref{cor:QAT_loss1} describes the common fine-tuning regime where the
pretrained model is already close to stationary on the river: after QAT brings
the quantized iterate back into the basin, the recovered quantized loss is
controlled by the basin flatness and the final quantization error.  The next
case covers a less optimized checkpoint, where the river direction still offers
objective decrease.

\begin{corollary}[Sub-optimal pre-trained solution]
\label{cor:QAT_loss2}
    Assume the conditions of Theorem~\ref{thm:QAT-recovery}. Suppose that there exists $c_\parallel>0$ such that $\langle \nabla f(P_{\cM}(w)), \nabla f(w)\rangle > c_\parallel \|\nabla f(P_{\cM}(w))\|$ for all $w \in U$.
    For $\pi \in \cM$, denote the normalized gradient along the river by $g(\pi) = \nabla f(\pi)/\|\nabla f(\pi)\|$. If in addition, $\eta \le {c_\parallel^2}/\big[{2LG^2(\kappa_\cM(R/\sqrt{\lambda_d} + \rho) + 1)^2}\big]$, $\kappa_{\cM} \le \frac{c_\parallel}{4G(R/\sqrt{\lambda_d} + \rho)}$, and there exists $\kappa$ satisfying $\kappa \leq \frac{c_\parallel}{4\rho L_P G}$ such that $\|g(\pi) - g(\pi^\prime)\| \le \kappa \|\pi - \pi^\prime\|$ for any $\pi, \pi^\prime \in \cM$. Then,
    \[
        f(Q(w_T))
        \le f(w_{\rm fp})
        + 2\epsilon_{\rm flat}
        + {\rho L_P G}
        - {\tfrac14 \eta T c_\parallel^2}.
    \]
\end{corollary}

The negative term in Corollary~\ref{cor:QAT_loss2} shows that, in this
suboptimal regime, QAT can act as a descent method starting from $w_{\rm fp}$, up to the flat-basin and quantization-error terms.  This
connects our geometric view to prior understanding of STE~\cite{yin2019understanding} as a descent direction for the population loss in two-layer networks with activation quantization; our result isolates a complementary mechanism in which the gradient is evaluated at the deployed quantized weights to correct quantization incompatibility.

\section{Experiments}

We evaluate the PTQ-failure and QAT-recovery mechanism predicted by our landscape analysis on both vision and language tasks. Simulation results for the matrix factorization in Example~\ref{ex:mf_sec2} are deferred to Appendix~\ref{app:mf-simulations}.

\subsection{ResNet and DeiT on Image Classification Benchmarks}
\label{sec:exp-vision}

We first study image classification models. We quantize the convolutional and linear weights of ResNet-20/56~\cite{he2016deep} on CIFAR-10, and the transformer-block linear weights of the 5M-parameter DeiT-Tiny model~\cite{touvron2021training} on ImageNet. We compare round-to-nearest (RTN), AdaRound~\cite{nagel2020up} or GPTQ, QAT, and equal-budget full-precision fine-tuning followed by PTQ. In all experiments, RTN and QAT use the same PTQ-calibrated per-channel or groupwise quantization grid. Thus, the comparison isolates the effect of applying the grid once as a post-training projection versus keeping the same grid fixed during STE-based QAT.
Detailed settings and quantitative comparisons across bitwidths and random seeds are reported in Appendix~\ref{app:exp-details} and Tables~\ref{tab:cifar-main}--\ref{tab:deit-main}. The advantage of QAT over equal-budget full-precision fine-tuning followed by PTQ is most pronounced at low bitwidths, where the PTQ perturbation is large enough to leave the low-loss basin.

\begin{figure}[t]
\centering
\setlength{\tabcolsep}{0pt}
\begin{tabular}{@{}c@{\hspace{0.02\linewidth}}c@{\hspace{0.02\linewidth}}c@{}}
\multicolumn{3}{c}{\small{ResNet-56}}\\%[0.05em]
\includegraphics[height=0.17\textheight]{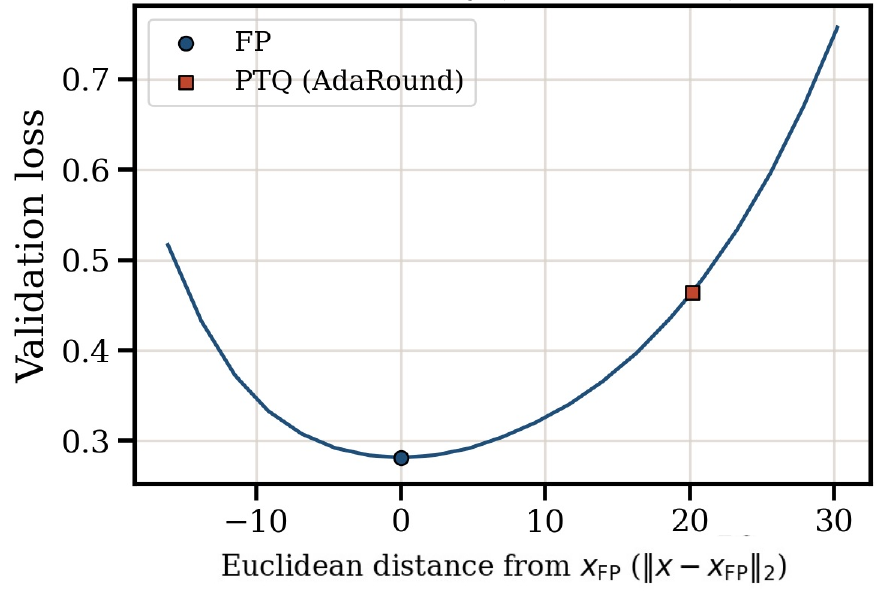} &
\includegraphics[height=0.17\textheight]{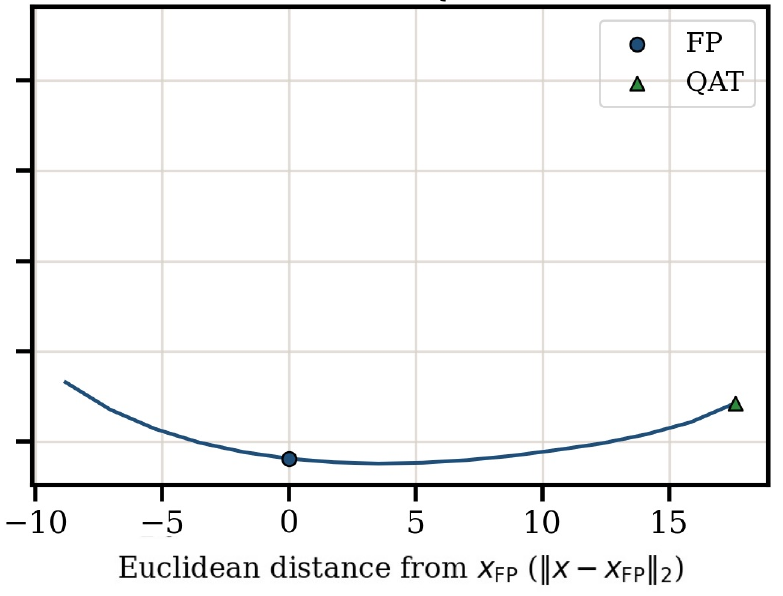} &
\includegraphics[height=0.17\textheight]{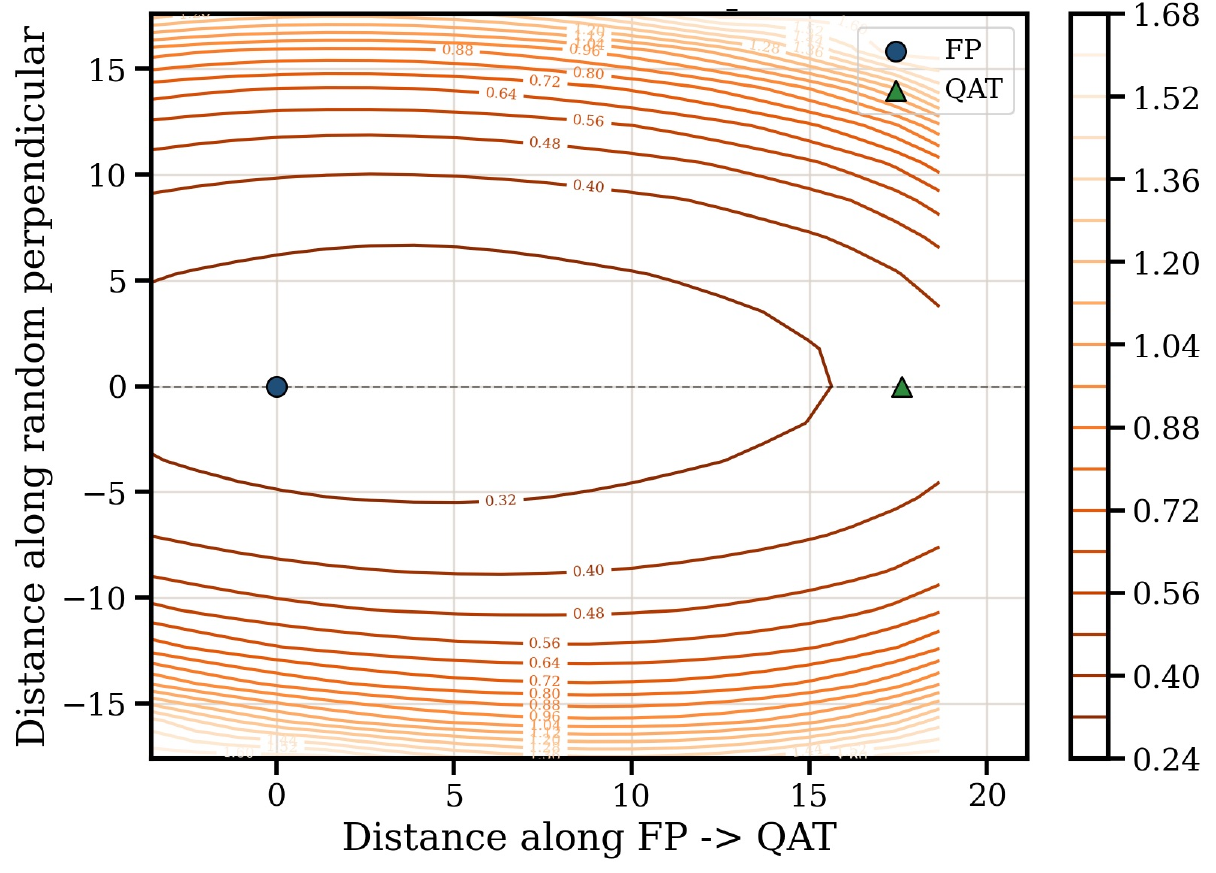}\\[-0.1em]
{\tiny (a) FP $\to$ PTQ} & {\tiny (b) FP $\to$ QAT} & {\tiny (c) 2D landscape}\\[0.5em]
\multicolumn{3}{c}{\small{DeiT-Tiny}}\\%[0.05em]
\includegraphics[height=0.17\textheight]{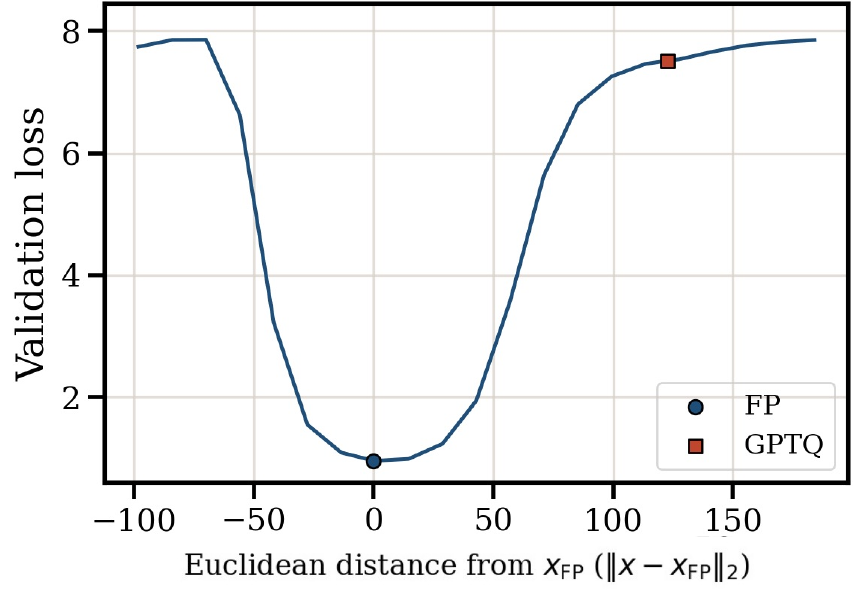} &
\includegraphics[height=0.17\textheight]{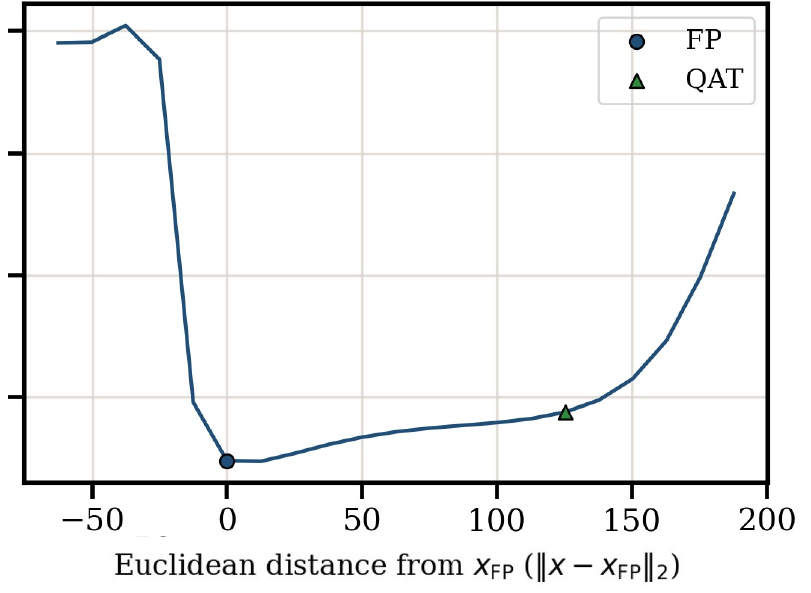} &
\includegraphics[height=0.17\textheight]{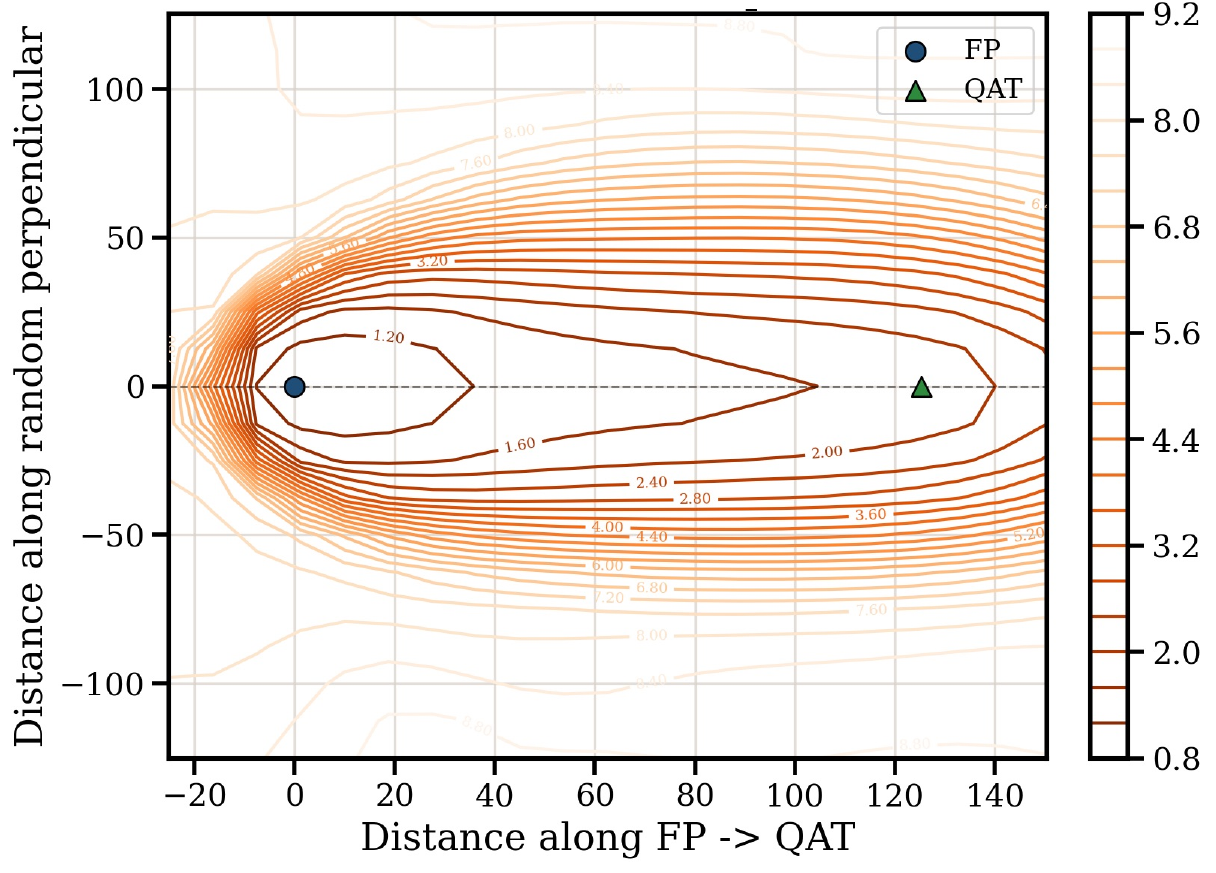}\\[-0.1em]
{\tiny (d) FP $\to$ PTQ} & {\tiny (e) FP $\to$ QAT} & {\tiny (f) 2D landscape}
\end{tabular}
\caption{\textbf{ResNet/CIFAR-10 and DeiT/ImageNet landscape diagnostics under 2-bit quantization.}
For ResNet-56, one-dimensional loss profiles from the FP checkpoint toward (a)~the AdaRound anchor and (b)~the QAT anchor; (c) Two-dimensional loss contour in the plane spanned by the FP$\to$QAT direction and a randomly generated orthogonal direction.
For DeiT-Tiny, (d)--(f) show the same diagnostics.
}
\label{fig:vision-w2-profiles}
\end{figure}

\FloatBarrier

Figure~\ref{fig:vision-w2-profiles} shows representative landscape diagnostics under 2-bit quantization. For each model, the first two panels interpolate from the full-precision checkpoint toward the PTQ and QAT endpoints, while the third panel shows a two-dimensional slice spanned by the FP$\to$QAT direction and a randomly generated orthogonal direction. The same flat-basin-plus-sharp-valley profile is observed across five independently sampled normal directions. The PTQ direction quickly leaves the low-loss region, whereas the QAT direction follows a substantially lower-loss quantized path. This contrast, together with the fact that equal-budget full-precision fine-tuning followed by PTQ does not close the gap, is consistent with the regime-specific prediction of Section~\ref{sec:Quant_under_River-valley-basin}: QAT is beneficial when the latent iterate has already settled into the basin, so that the quantized gradient probes the valley wall, rather than merely continuing the cosine learning rate tail along the river. %

\subsection{Llama on SlimPajama}
\label{sec:exp-Llama}

We next test the same mechanism in language modeling. We pretrain a 213M-parameter Llama-style decoder on SlimPajama for $60{,}000$ iterations with a cosine learning rate schedule. We then deploy weight-only FP4 quantization with the E2M1 format on every linear layer inside the transformer blocks, while keeping the token embedding, language-model head, and RMSNorm parameters in bf16. GPTQ is used as the PTQ baseline. As in Section~\ref{sec:exp-vision}, QAT uses the GPTQ-fitted FP4 quantizer as a fixed STE quantizer. Thus, PTQ and QAT are evaluated on the same FP4 grid, and any difference reflects weight adaptation rather than a different quantizer fit. We initialize QAT from the full-precision checkpoint and run a short constant-learning-rate STE phase. Detailed settings are provided in Appendix~\ref{app:exp-details}.

Figure~\ref{fig:Llama-landscape} reports one-dimensional loss profiles from the converged full-precision checkpoint toward the GPTQ and QAT anchors, together with a two-dimensional contour plot around the FP$\to$QAT direction. The QAT displacement is much shorter than the GPTQ displacement in parameter-space Euclidean norm. Moreover, the two displacements are nearly orthogonal, forming an angle of approximately $92^\circ$. Therefore, QAT is not a small correction along the GPTQ rounding direction; instead, it finds a distinct and much shorter adaptation direction.
Panel~\ref{fig:Llama-1d-ptq} shows that the FP$\to$GPTQ profile rises monotonically and has no local minimum near the GPTQ deployment point, which is consistent with the signature of rounding outside the basin. By contrast, Panel~\ref{fig:Llama-1d-qat} shows that the FP$\to$QAT profile first drops below the full-precision baseline over an interior interval and rises only after passing the QAT anchor. This indicates that QAT moves into a lower-loss region before eventually leaving it. Panel~\ref{fig:Llama-2d} further visualizes the local geometry in the plane defined by the QAT direction and a random orthogonal direction: the loss remains relatively flat along the QAT path and increases away from this low-loss band.

\begin{figure}[H]
\centering
\subfigure[FP $\to$ PTQ]{
  \includegraphics[width=0.3\linewidth]{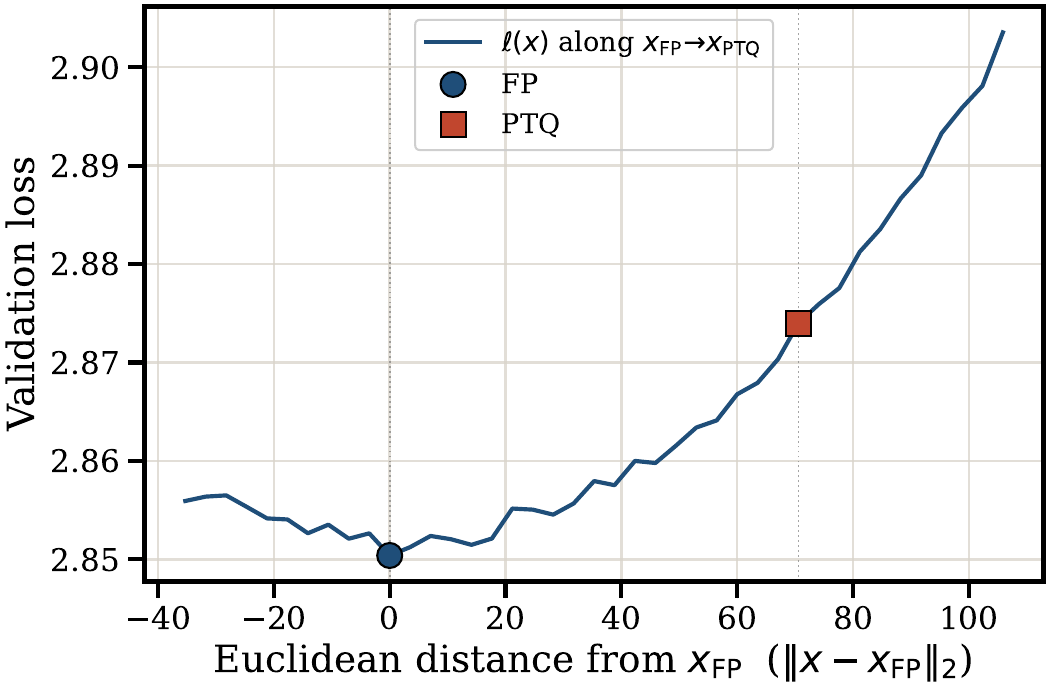}
  \label{fig:Llama-1d-ptq}
}
\subfigure[FP $\to$ QAT]{
  \includegraphics[width=0.3\linewidth]{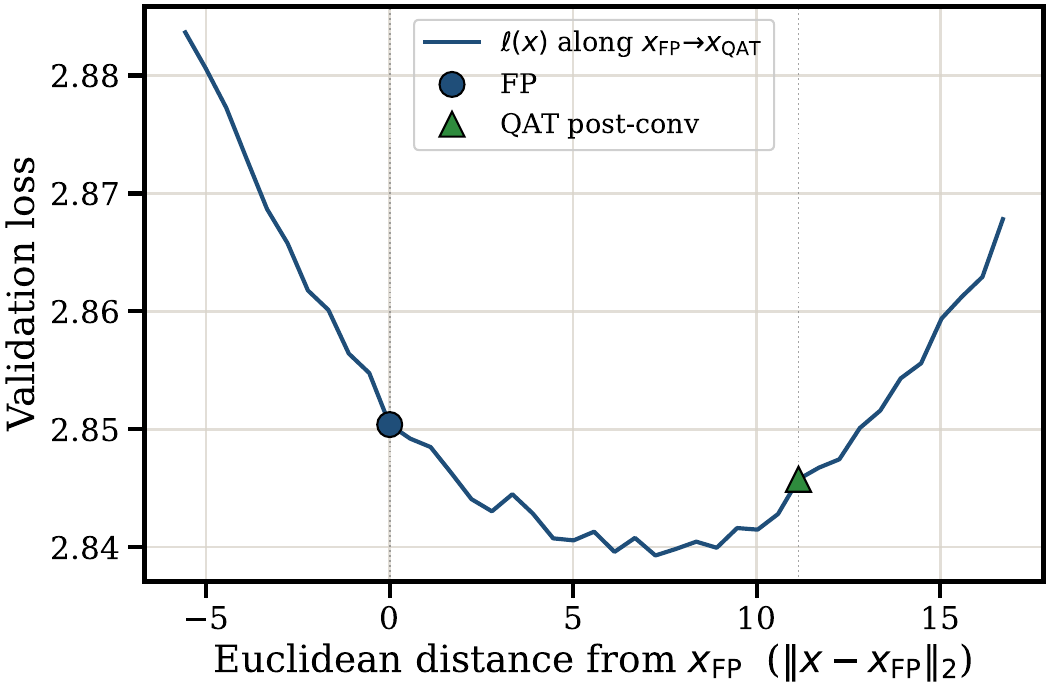}
  \label{fig:Llama-1d-qat}
}
\subfigure[2D landscape]{
  \includegraphics[width=0.3\linewidth]{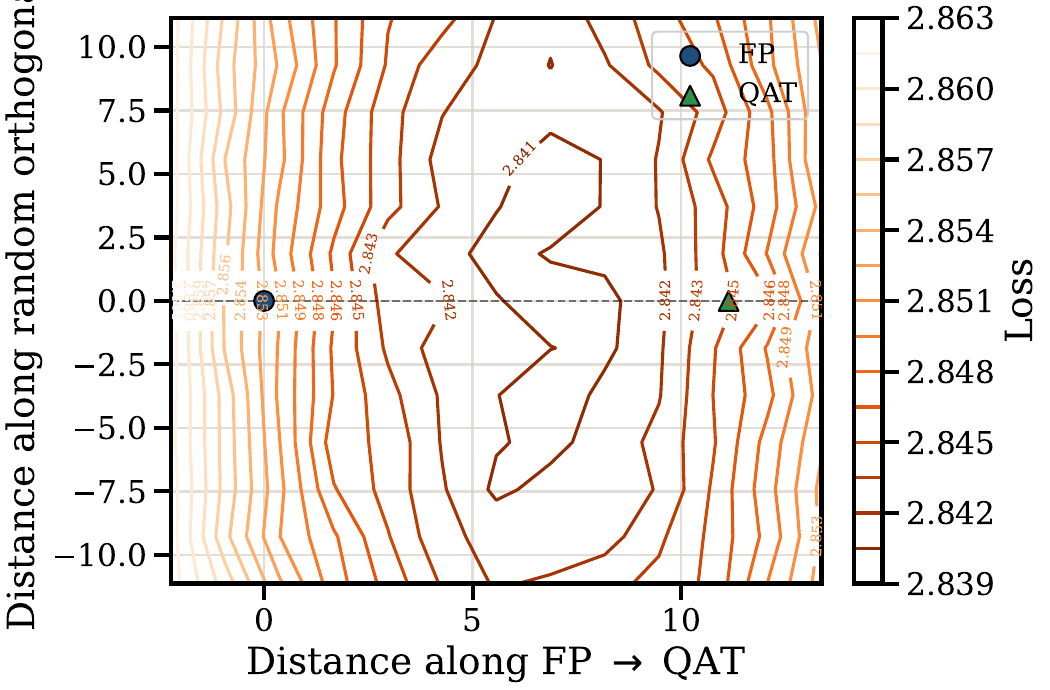}
  \label{fig:Llama-2d}
}
\caption{\textbf{Llama/SlimPajama landscape diagnostics at FP4.}
One-dimensional loss profiles from the converged FP checkpoint toward
(a)~the GPTQ anchor and (b)~the post-convergence QAT anchor.
(c)~Two-dimensional loss contours in the plane spanned by the FP$\to$QAT
direction and a randomly generated orthogonal direction.}
\label{fig:Llama-landscape}
\end{figure}

\FloatBarrier

These diagnostics extend the PTQ-failure/QAT-recovery pattern to language modeling at LLM scale. The qualitative behavior is again consistent with the river--valley--basin geometry: GPTQ moves the deployed model toward a high-loss region, whereas QAT follows a different, basin-aligned direction and reaches a nearby low-loss quantized solution. We emphasize that this Llama experiment is a single-seed diagnostic study intended to support the generality of the proposed mechanism, rather than a controlled comparison across bitwidths and seeds.

\FloatBarrier
\section{Discussion}
Our theoretical and empirical analysis makes several testable predictions on the model quantization. First, PTQ failure
should be most visible when the quantization grid is comparable to the basin
width: coarser grids are more likely to cross the boundary, while sufficiently
fine grids should remain in the flat band. Second, during successful QAT, the
loss of the deployed quantized model should improve before one observes a large
change in the latent full-precision loss, reflecting a correction of
quantization compatibility rather than ordinary full-precision optimization.
Third, equal-budget full-precision fine-tuning followed by PTQ can remain fragile when its updates move mainly along the river and do not reduce the normal quantization error. 

\textbf{Limitations.}
Our landscape model is local, and our theory only gives sufficient rather than necessary conditions for QAT recovery. Empirically, our experiments rely on slice-based diagnostics of the loss landscape, which illustrate the phenomenon but do not offer a complete verification of the global, high-dimensional geometry. Furthermore, our analysis of Hessian-based PTQ abstracts complex, practical layerwise techniques into a simplified local second-order proxy. Consequently, this model does not account for practical deployment factors such as activation quantization, dynamic scale learning, gradient clipping, or specific hardware constraints. Lastly, while the Llama experiments confirm that our proposed mechanism operates at the scale of large language models, they are intended as proof-of-concept evidence rather than a comprehensive study of scaling laws across diverse model sizes and bitwidths.

\bibliography{quant}

\newpage
\appendix

\section{Experimental Details}
\label{app:exp-details}

\paragraph{ResNet/CIFAR-10.}
Each full-precision ResNet is trained for 200 epochs with SGD, momentum 0.9,
batch size 128, weight decay $2\cdot 10^{-4}$, initial learning rate 0.1, and
cosine annealing.  We quantize all convolutional and linear weights with
per-output-channel signed uniform scales.  AdaRound uses 32 calibration batches of size 128, i.e.\ 4096 calibration images, and 10,000 rounding-optimization iterations.  RTN uses the same AdaRound-fitted per-channel scales but no learned rounding.  QAT starts from the epoch-180
checkpoint, keeps the AdaRound quantizer fixed, and trains for 20 epochs with
SGD at learning rate $10^{-4}$, dropping by 0.5 at epochs 10 and 16.  The
equal-budget control continues full-precision training from epoch 180 to 200 and
then applies AdaRound.  All ResNet-20 and ResNet-56 runs use a single NVIDIA
RTX A5000 GPU per run.

\paragraph{DeiT/ImageNet.}
For DeiT-Tiny, we quantize the transformer-block linear layers at W2/W3/W4 and
leave the classification head full precision.  GPTQ uses 128 calibration images,
group size 128, and no activation ordering.  RTN uses the same fitted grid/group
scales without GPTQ error correction.  QAT starts from the pretrained
full-precision checkpoint with the GPTQ quantizer frozen and runs for 10 epochs
with AdamW at learning rate $5\cdot 10^{-4}$.  All DeiT/ImageNet runs use four
NVIDIA A40 GPUs per run.

\paragraph{Llama/SlimPajama-6B.}
We pretrain a $213$M-parameter Llama-style decoder ($24$ layers, $12$ heads,
hidden size $768$, RMSNorm) on SlimPajama-6B for $60{,}000$ iterations with
AdamW (weight decay $0.1$, $\beta\!=\!(0.9,0.95)$, gradient clip $1.0$), a
cosine schedule with base learning rate $10^{-3}$ and $300$ warmup steps,
sequence length $512$, batch size $50\!\times\!4$ ($50$ per-device with
gradient-accumulation $4$, total batch $200$) on a single H200, and bf16.
GPTQ is applied weight-only to every \texttt{nn.Linear} inside the
transformer blocks — the
\texttt{q\_proj}/\texttt{k\_proj}/\texttt{v\_proj}/\texttt{o\_proj}
attention projections and the
\texttt{gate\_proj}/\texttt{up\_proj}/\texttt{down\_proj} MLP projections
— at FP4 (E2M1, with levels
$\{0,\pm 0.5,\pm 1,\pm 1.5,\pm 2,\pm 3,\pm 4,\pm 6\}$) with
per-output-channel signed scales, input-column groups of size $128$, GPTQ
block size $128$, and percent-damp $0.01$; the token embedding, the
language-model head, and the RMSNorm parameters remain in bf16.
Calibration draws $128$ sequences of length $512$ from the SlimPajama train
split with a fixed seed.  The GPTQ-fitted per-Linear FP4 quantizer (scales
and group structure) is persisted and reused as the STE fake-quant in both
QAT regimes, so PTQ and QAT share an identical FP4 grid.  Continuation QAT
resumes from the iter-$48{,}000$ ($80\%$) FP checkpoint and runs the same
cosine schedule (and the same data ordering, controlled by the seed) to
iter $60{,}000$, i.e.\ $12{,}000$ STE iterations along the cosine LR tail.
Post-convergence QAT initializes weights from the iter-$60{,}000$ FP
checkpoint with a fresh optimizer, uses a constant learning rate of
$10^{-5}$ with no warmup or decay, and trains for $3{,}000$ STE iterations
($5\%$ of pretraining).  All runs use a single seed; for the river-cross diagnostic
we additionally average over three random perpendicular directions.

\section{Additional ResNet/CIFAR-10 Diagnostics}
\label{app:cifar-extra}

\begin{table}[H]
\centering
\caption{{ResNet/CIFAR-10 equal-budget comparison.}
Each method entry reports mean validation loss / top-1 accuracy over five seeds,
with nonzero standard deviations in parentheses.  The depth column reports the
epoch-180 full-precision result.  FP-200+AdaRound is the equal-budget
full-precision control followed by AdaRound PTQ.  RTN and QAT use the same
AdaRound-calibrated weight grid as the AdaRound baseline.  Bold marks the best
quantized result in each row, separately for loss and accuracy.}
\label{tab:cifar-main}
\resizebox{\linewidth}{!}{
\begin{tabular}{cc|cccc}
\toprule
Depth (FP-180) & Bits & RTN & AdaRound & QAT & FP-200 + AdaRound \\
\midrule
\multirow{3}{*}{\begin{tabular}{c}20\\{\scriptsize (0.329 / 91.96)}\end{tabular}}
& W2 & 9.365 / 10.41 & 0.497 (0.017) / 86.62 (0.49) & {\bf 0.407} (0.022) / {\bf 87.92} (0.67) & 0.576 (0.030) / 84.51 (0.85) \\
& W3 & 1.201 / 73.96 & 0.358 (0.004) / {\bf 91.19} (0.02) & {\bf 0.334} (0.003) / {\bf 91.19} (0.13) & 0.362 (0.002) / 91.08 (0.09) \\
& W4 & 0.430 / 89.68 & 0.337 (0.001) / 91.79 (0.09) & 0.335 (0.001) / 91.75 (0.07) & {\bf 0.330} (0.001) / {\bf 91.87} (0.08) \\
\midrule
\multirow{3}{*}{\begin{tabular}{c}56\\{\scriptsize (0.281 / 93.50)}\end{tabular}}
& W2 & 5.212 / 13.88 & 0.448 (0.010) / 88.83 (0.30) & {\bf 0.337} (0.006) / {\bf 90.50} (0.17) & 0.449 (0.015) / 89.07 (0.27) \\
& W3 & 0.618 / 86.21 & 0.314 (0.006) / 92.77 (0.15) & {\bf 0.288} (0.005) / 92.80 (0.10) & 0.299 (0.005) / {\bf 93.04} (0.14) \\
& W4 & 0.334 / 92.60 & 0.289 (0.001) / 93.30 (0.09) & {\bf 0.278} (0.001) / 93.42 (0.04) & 0.280 (0.001) / {\bf 93.46} (0.05) \\
\bottomrule
\end{tabular}
}
\end{table}

Figure~\ref{fig:cifar-profiles-extra} provides the additional ResNet-20 seed-1
interpolation diagnostics for W3/W4, and Figure~\ref{fig:cifar-profiles}
provides the corresponding ResNet-56 profiles.  The W2 diagnostics are shown in
the main body in Figure~\ref{fig:vision-w2-profiles}.

\begin{figure}[H]
\centering
\begin{minipage}[t]{0.36\linewidth}\centering
  \includegraphics[width=\linewidth]{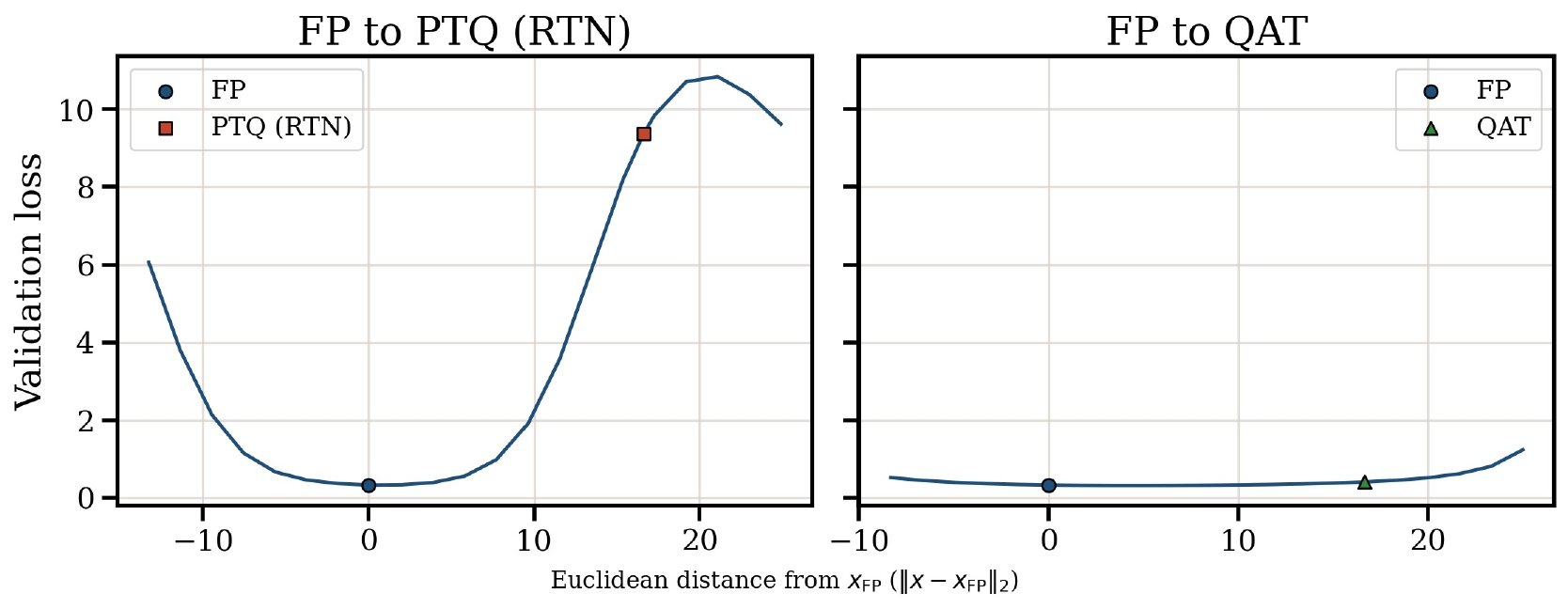}\\[-0.2em]
  {\tiny W2: FP-180 $\to$ RTN/QAT 1D profile}
\end{minipage}\hspace{0.006\linewidth}
\begin{minipage}[t]{0.36\linewidth}\centering
  \includegraphics[width=\linewidth]{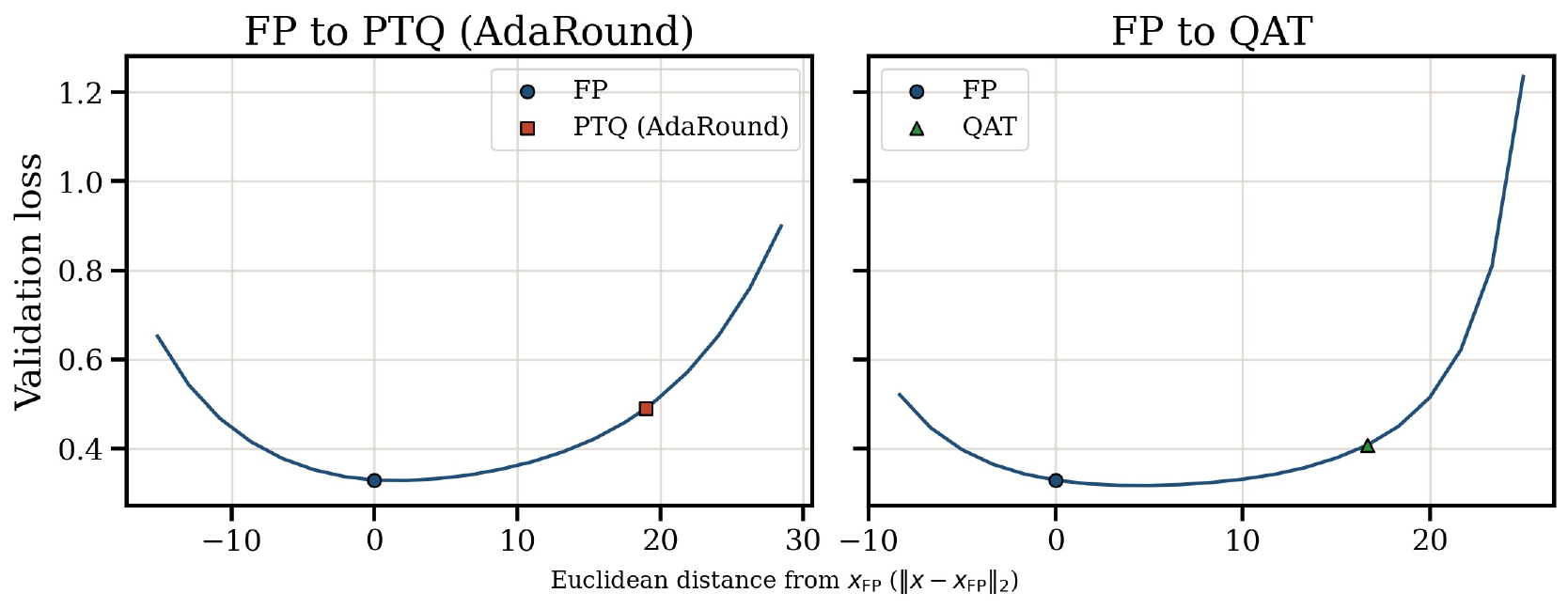}\\[-0.2em]
  {\tiny W2: FP-180 $\to$ AdaRound/QAT 1D profile}
\end{minipage}\hspace{0.025\linewidth}
\begin{minipage}[t]{0.2\linewidth}\centering
  \includegraphics[width=\linewidth]{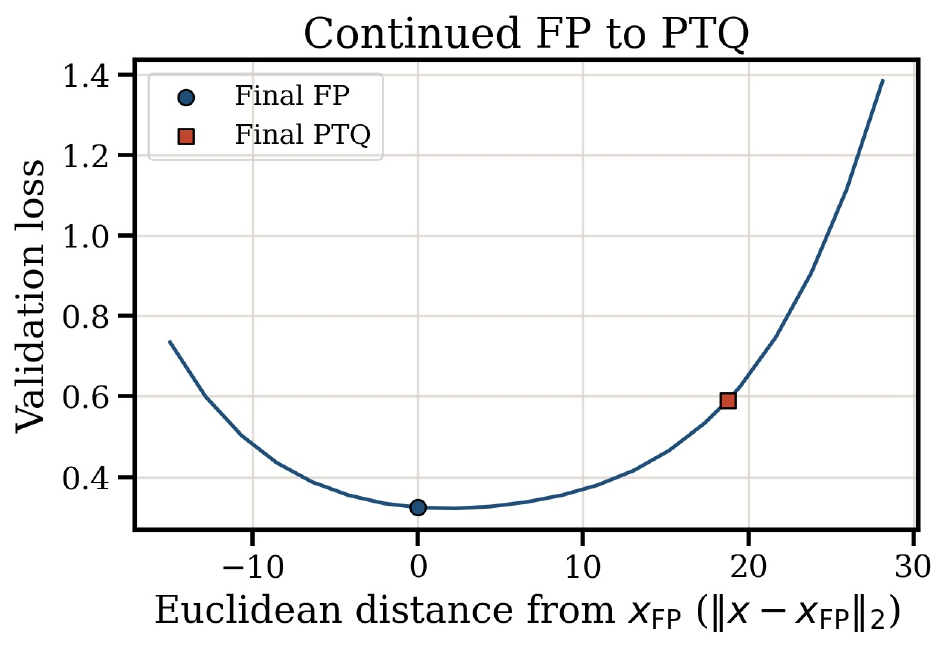}\\[-0.2em]
  {\tiny W2: FP-200 $\to$ AdaRound}
\end{minipage}\\[0.35em]
\begin{minipage}[t]{0.36\linewidth}\centering
  \includegraphics[width=\linewidth]{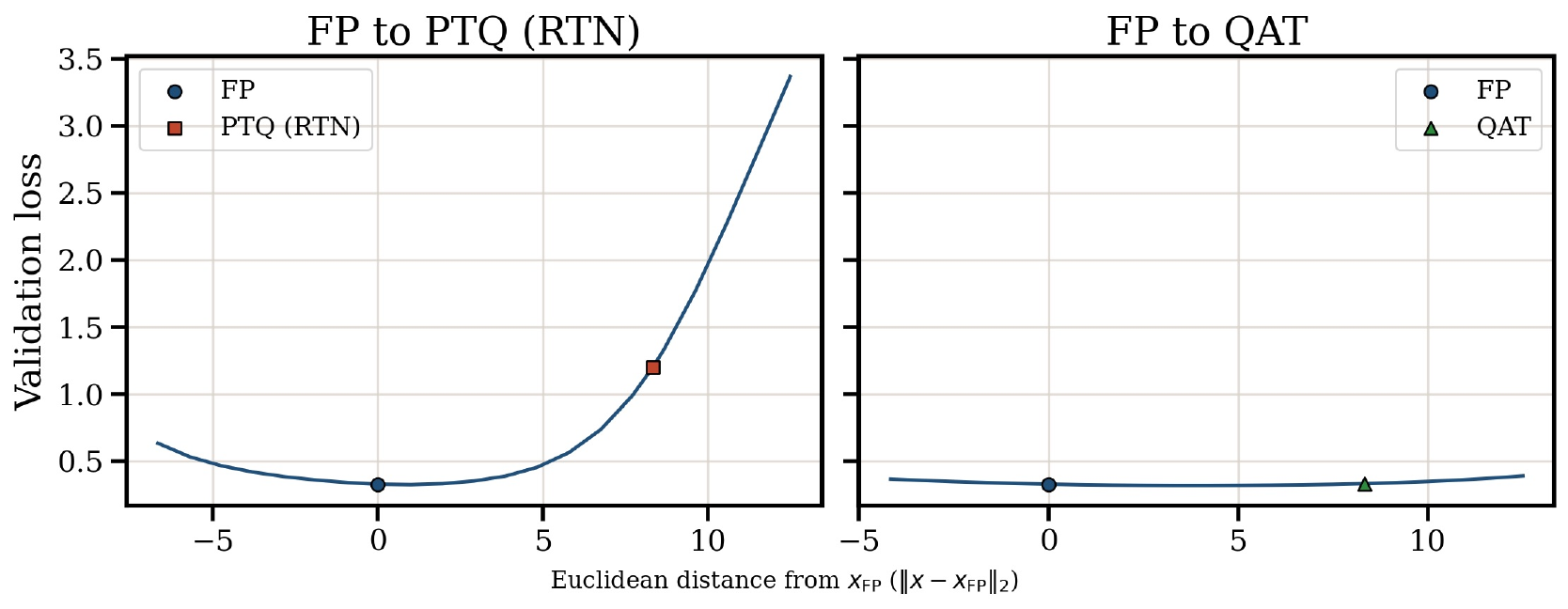}\\[-0.2em]
  {\tiny W3: FP-180 $\to$ RTN/QAT 1D profile}
\end{minipage}\hspace{0.006\linewidth}
\begin{minipage}[t]{0.36\linewidth}\centering
  \includegraphics[width=\linewidth]{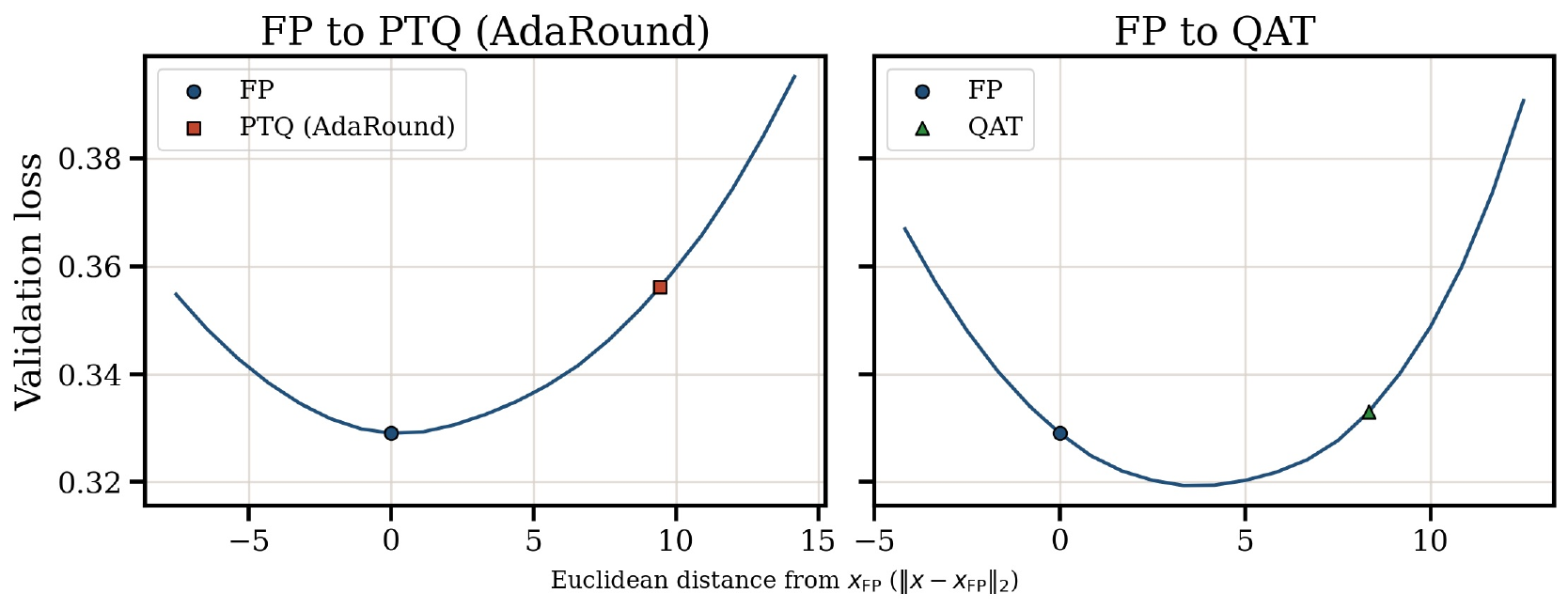}\\[-0.2em]
  {\tiny W3: FP-180 $\to$ AdaRound/QAT 1D profile}
\end{minipage}\hspace{0.025\linewidth}
\begin{minipage}[t]{0.2\linewidth}\centering
  \includegraphics[width=\linewidth]{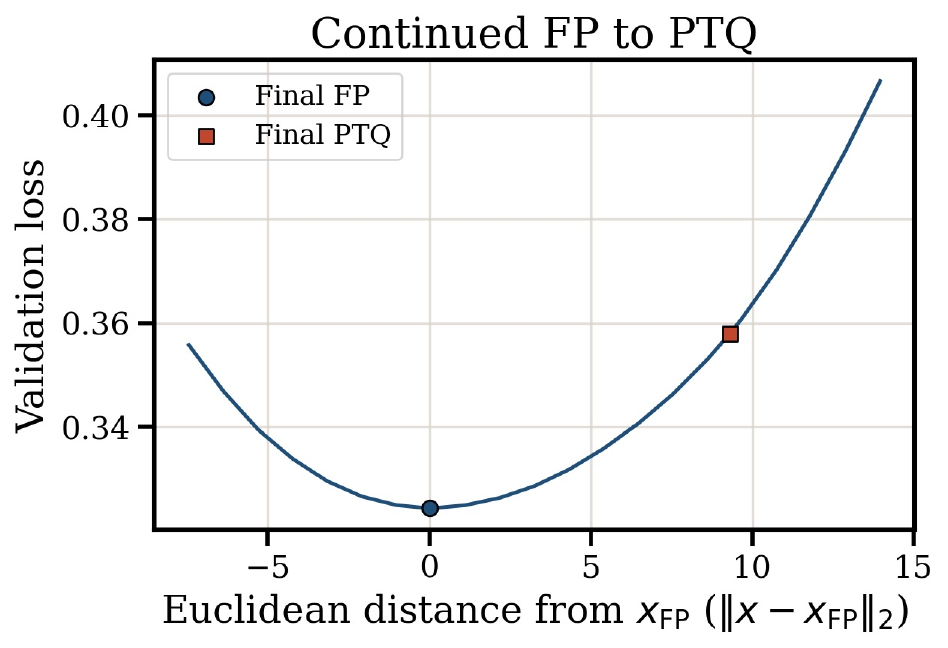}\\[-0.2em]
  {\tiny W3: FP-200 $\to$ AdaRound}
\end{minipage}\\[0.35em]
\begin{minipage}[t]{0.36\linewidth}\centering
  \includegraphics[width=\linewidth]{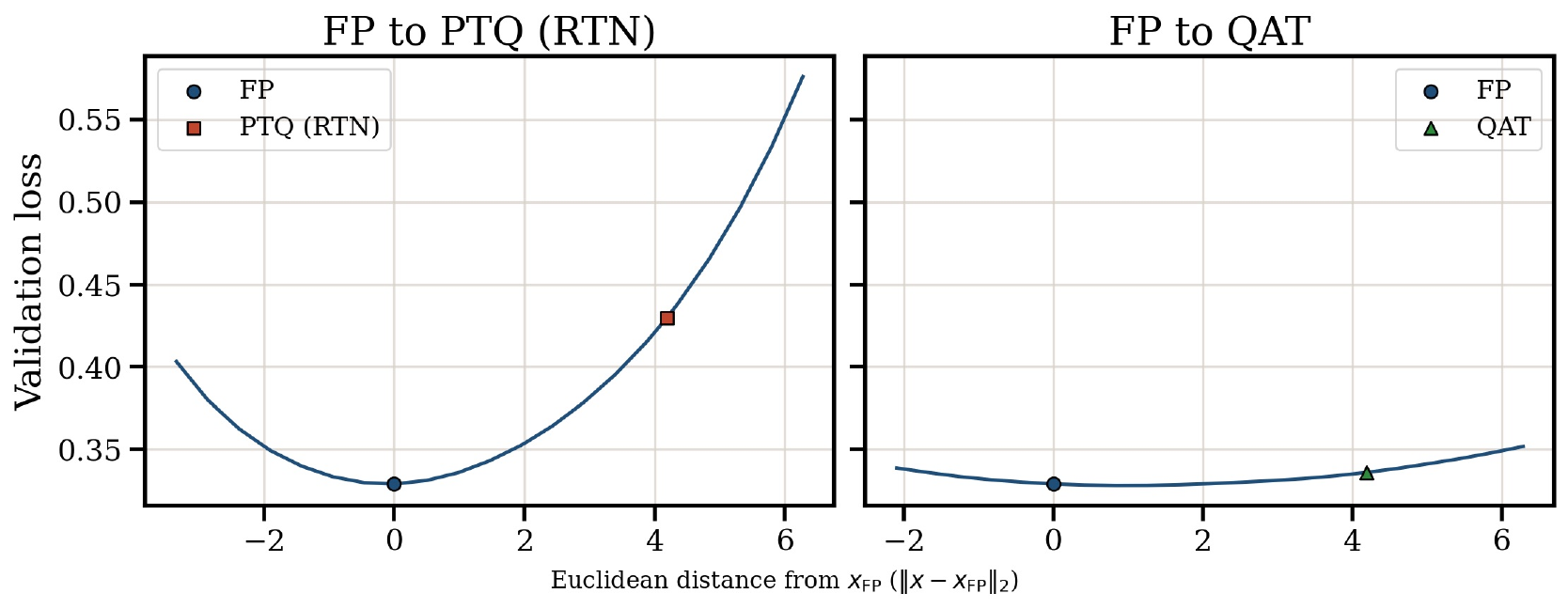}\\[-0.2em]
  {\tiny W4: FP-180 $\to$ RTN/QAT 1D profile}
\end{minipage}\hspace{0.006\linewidth}
\begin{minipage}[t]{0.36\linewidth}\centering
  \includegraphics[width=\linewidth]{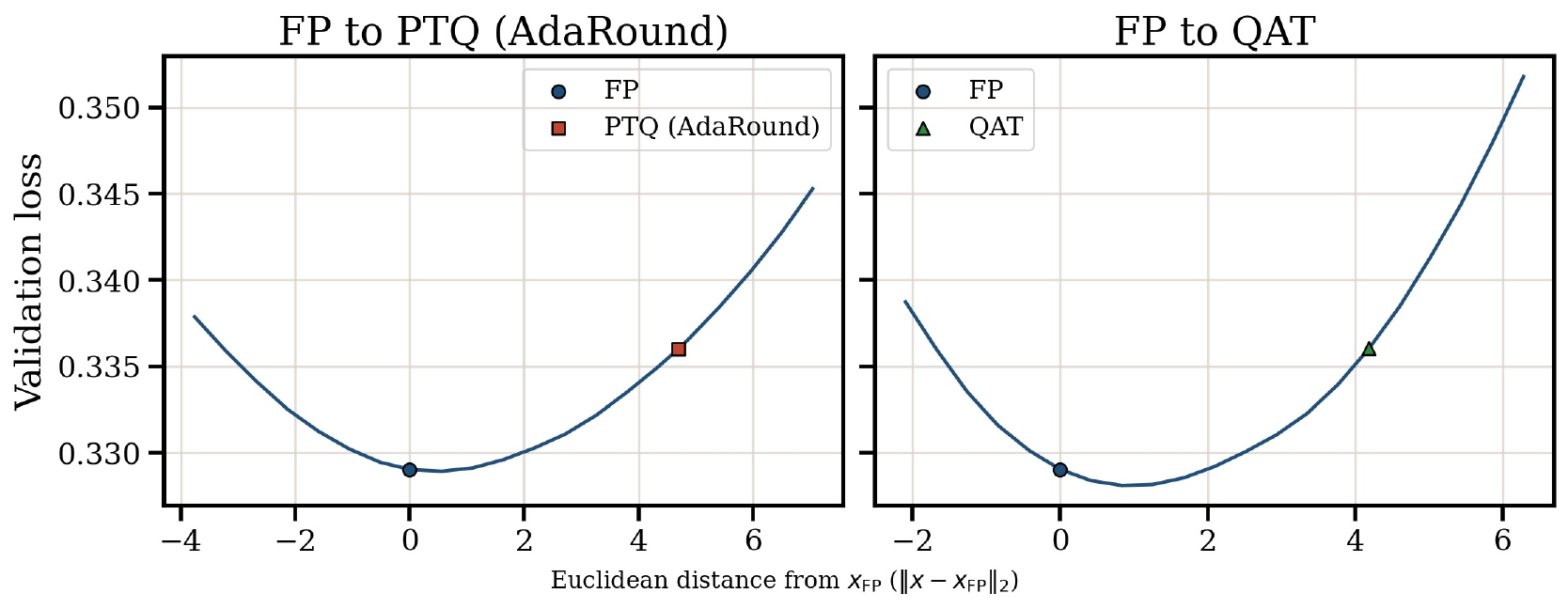}\\[-0.2em]
  {\tiny W4: FP-180 $\to$ AdaRound/QAT 1D profile}
\end{minipage}\hspace{0.025\linewidth}
\begin{minipage}[t]{0.2\linewidth}\centering
  \includegraphics[width=\linewidth]{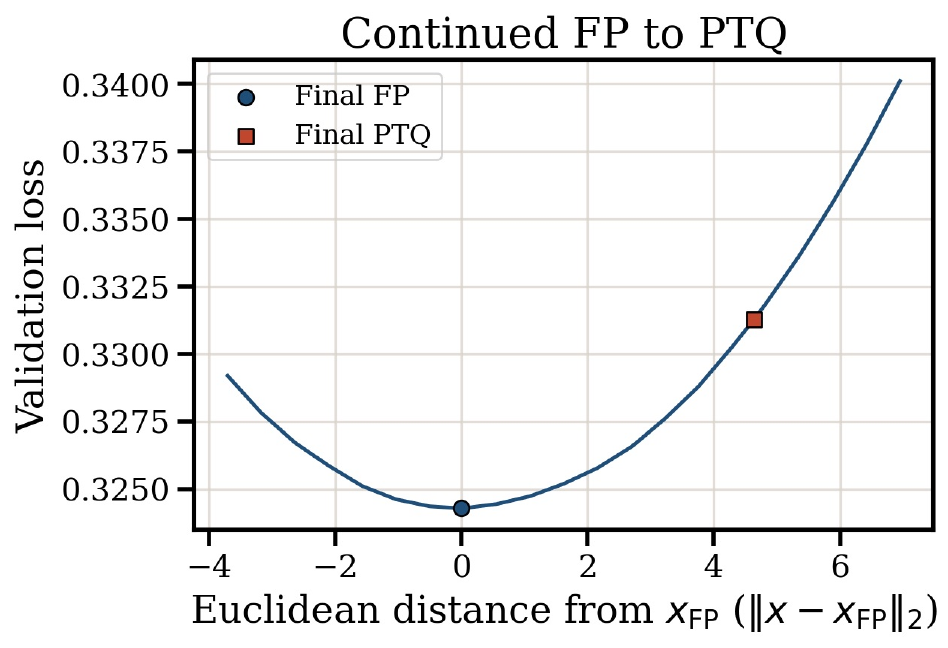}\\[-0.2em]
  {\tiny W4: FP-200 $\to$ AdaRound}
\end{minipage}
\caption{{Additional ResNet-20 on CIFAR-10 interpolation diagnostics.}
W3/W4 seed-1 interpolation profiles from the pretrained ResNet checkpoint
toward PTQ, QAT, and final-PTQ endpoints for ResNet-20.}
\label{fig:cifar-profiles-extra}
\end{figure}

\begin{figure}[H]
\centering
\begin{minipage}[t]{0.36\linewidth}\centering
  \includegraphics[width=\linewidth]{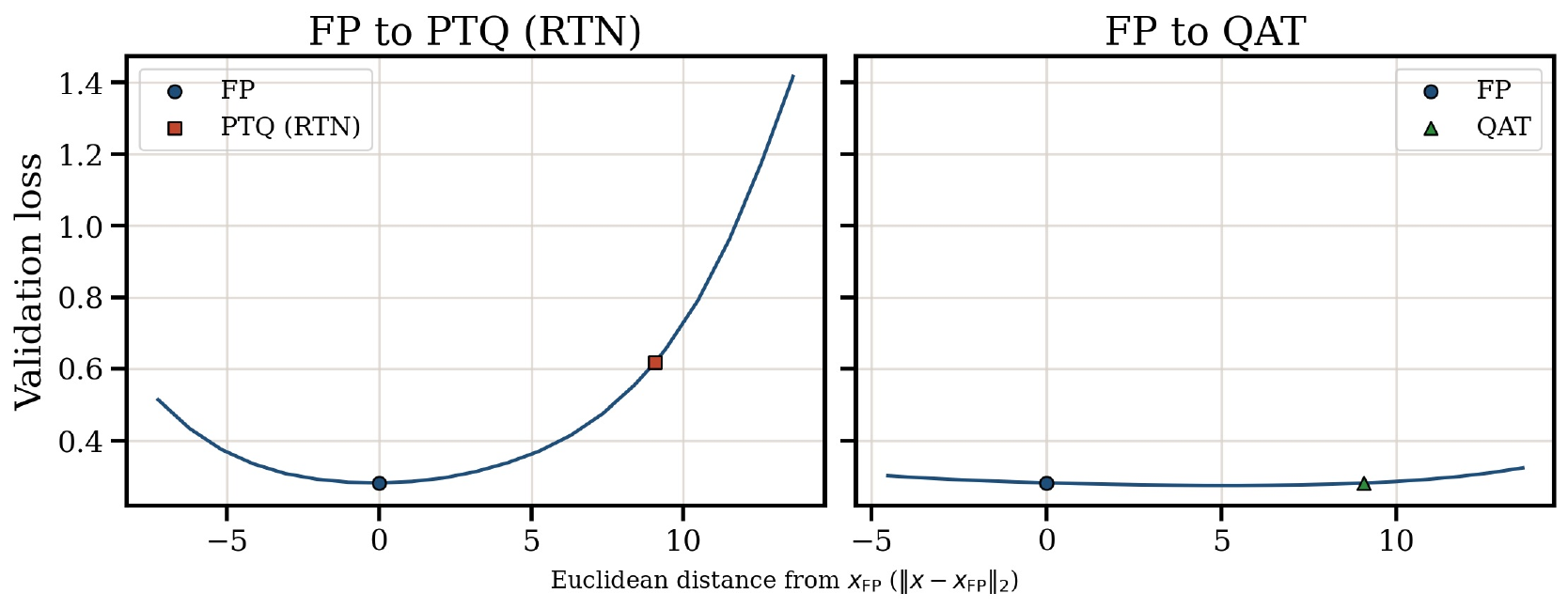}\\[-0.2em]
  {\tiny W3: FP-180 $\to$ RTN/QAT 1D profile}
\end{minipage}\hspace{0.006\linewidth}
\begin{minipage}[t]{0.36\linewidth}\centering
  \includegraphics[width=\linewidth]{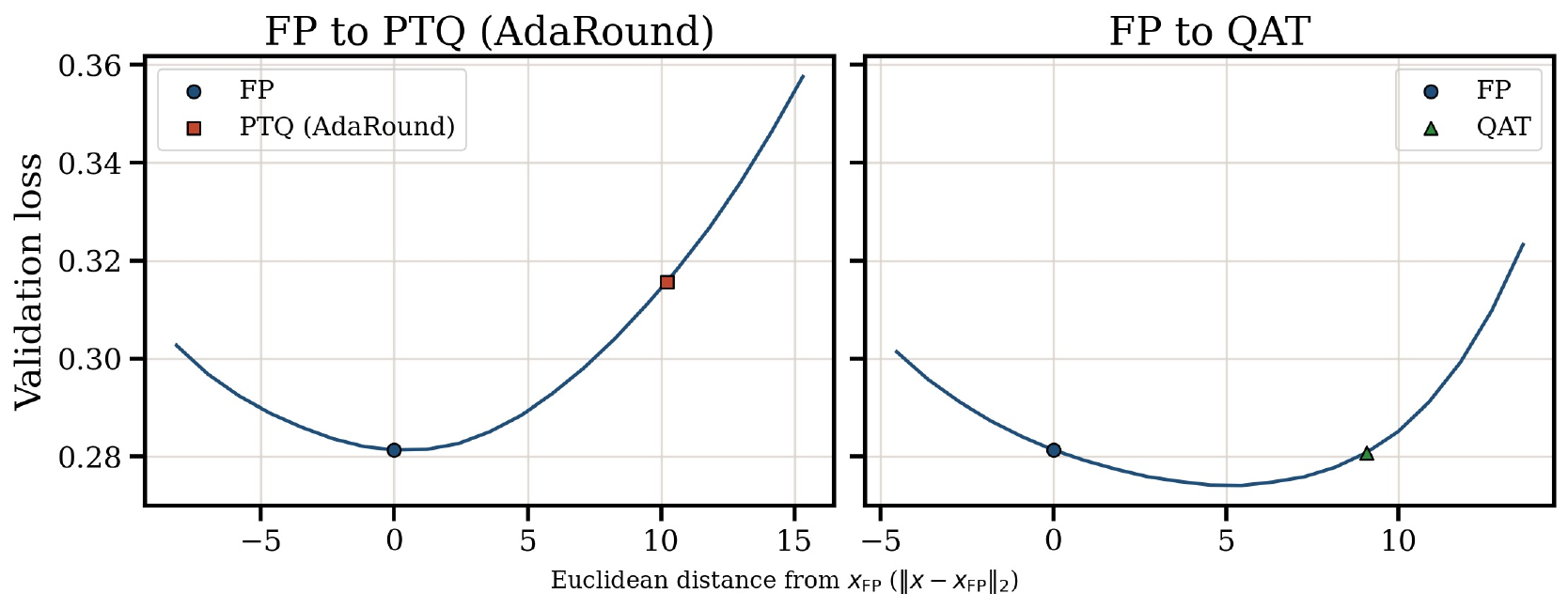}\\[-0.2em]
  {\tiny W3: FP-180 $\to$ AdaRound/QAT 1D profile}
\end{minipage}\hspace{0.025\linewidth}
\begin{minipage}[t]{0.2\linewidth}\centering
  \includegraphics[width=\linewidth]{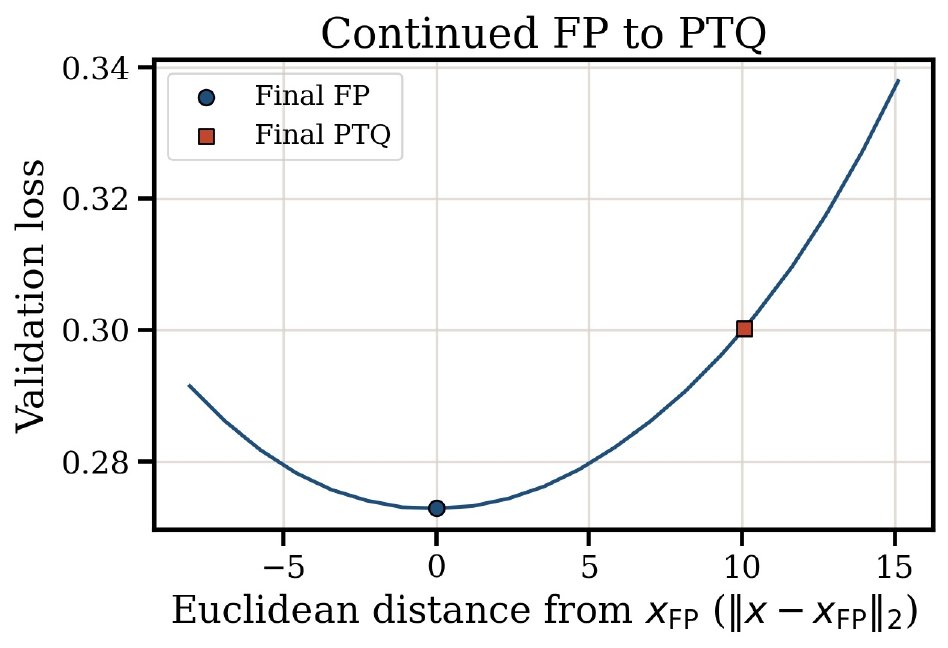}\\[-0.2em]
  {\tiny W3: FP-200 $\to$ AdaRound}
\end{minipage}\\[0.35em]
\begin{minipage}[t]{0.36\linewidth}\centering
  \includegraphics[width=\linewidth]{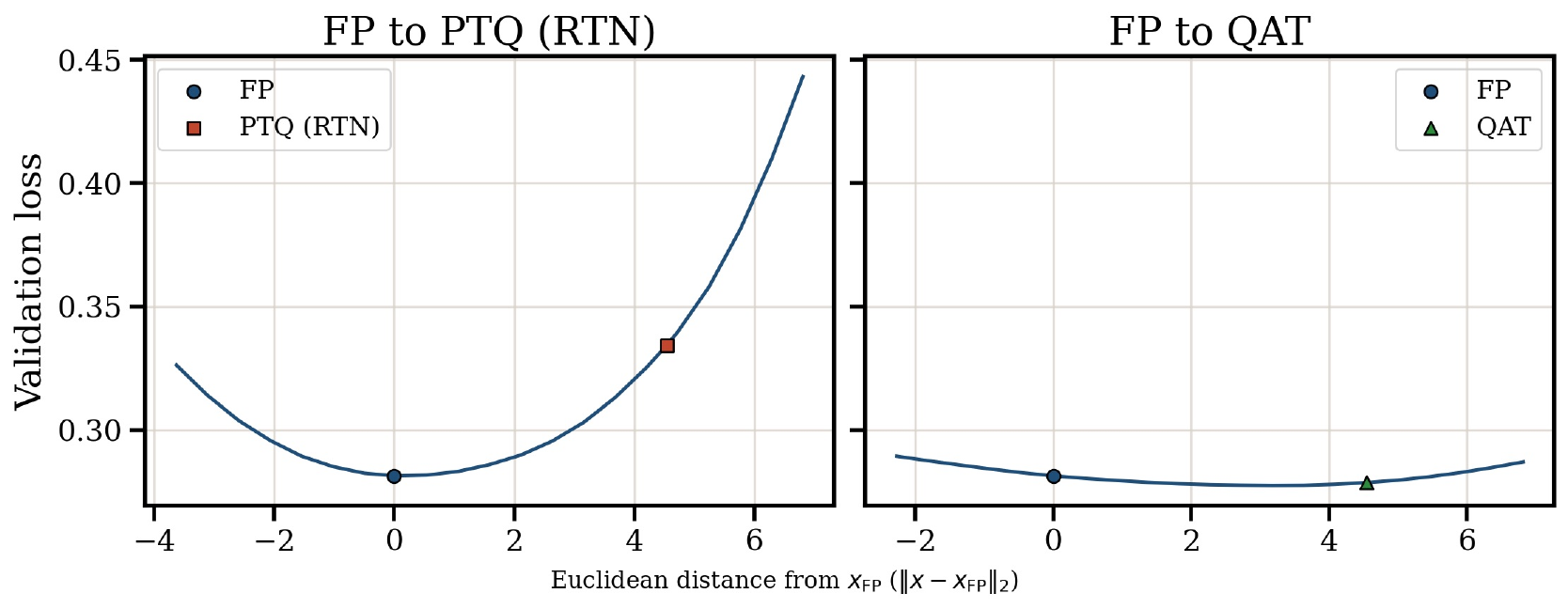}\\[-0.2em]
  {\tiny W4: FP-180 $\to$ RTN/QAT 1D profile}
\end{minipage}\hspace{0.006\linewidth}
\begin{minipage}[t]{0.36\linewidth}\centering
  \includegraphics[width=\linewidth]{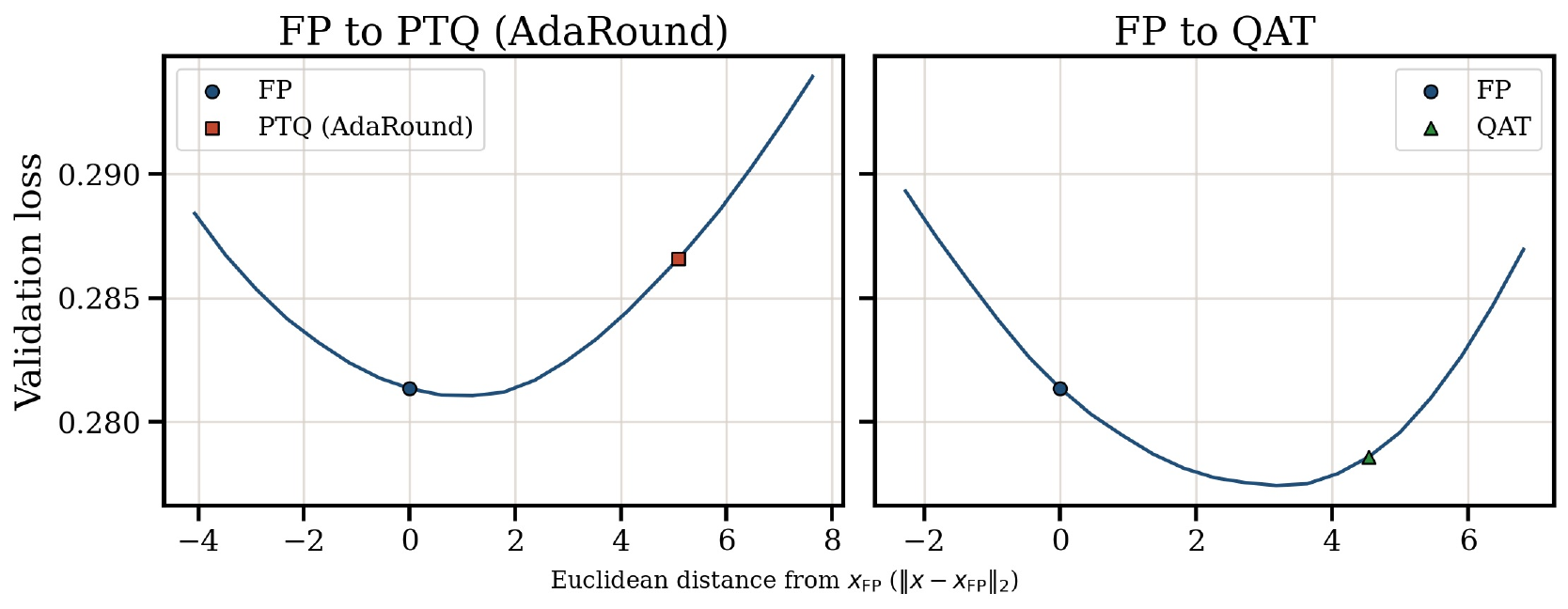}\\[-0.2em]
  {\tiny W4: FP-180 $\to$ AdaRound/QAT 1D profile}
\end{minipage}\hspace{0.025\linewidth}
\begin{minipage}[t]{0.2\linewidth}\centering
  \includegraphics[width=\linewidth]{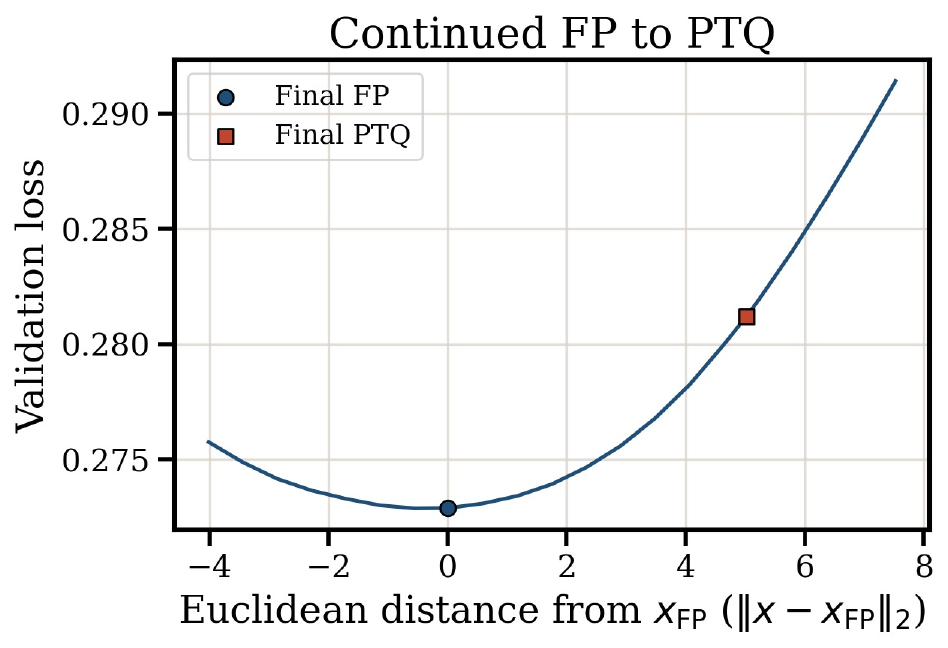}\\[-0.2em]
  {\tiny W4: FP-200 $\to$ AdaRound}
\end{minipage}
\caption{{Additional ResNet-56 on CIFAR-10 interpolation diagnostics.}
W3/W4 seed-1 interpolation profiles from the pretrained ResNet-56 checkpoint
toward PTQ, QAT, and final-PTQ endpoints.}
\label{fig:cifar-profiles}
\label{fig:cifar56-profiles-extra}
\end{figure}

\FloatBarrier

\section{Additional DeiT/ImageNet Diagnostics}
\label{app:deit-extra}

\begin{table}[H]
\centering
\caption{{DeiT/ImageNet equal-budget comparison.}
Each method entry reports validation loss / top-1 accuracy.  RTN, GPTQ, and QAT
report means over three random seeds with standard deviations in parentheses, while
FPFT+GPTQ is the available seed-0 equal-budget control.  RTN and QAT use the
same GPTQ-fitted weight grid as the GPTQ baseline.  Bold marks the best
quantized result in each row, separately for loss and accuracy.}
\label{tab:deit-main}
\scriptsize
\setlength{\tabcolsep}{2pt}
\resizebox{\linewidth}{!}{
\begin{tabular}{cc|cccc}
\toprule
pre-trained FP & Bits & RTN & GPTQ & QAT & fine-tuned FP + GPTQ \\
\midrule
\multirow{3}{*}{1.220 / 72.14} & W2 & 7.216 (0.071) / 0.16 (0.05) & 7.313 (0.084) / 0.19 (0.05) & {\bf 2.121} (0.011) / {\bf 52.78} (0.18) & 7.328 / 0.22 \\
& W3 & 3.166 (0.088) / 37.52 (1.28) & 2.062 (0.029) / 56.78 (0.39) & {\bf 1.439} (0.003) / {\bf 66.60} (0.12) & 2.021 / 56.73 \\
& W4 & 1.402 (0.003) / 68.60 (0.08) & 1.318 (0.002) / 70.17 (0.03) & {\bf 1.245} (0.003) / {\bf 71.01} (0.17) & 1.261 / 70.86 \\
\bottomrule
\end{tabular}
}
\end{table}

Figure~\ref{fig:deit-profiles-extra} provides the one-dimensional interpolation
profiles for the W3/W4 GPTQ and RTN anchors used in Table~\ref{tab:deit-main},
together with the equal-budget FPFT+GPTQ control.  The W2 diagnostics are shown
in the main body in Figure~\ref{fig:vision-w2-profiles}.

\begin{figure}[H]
\centering
\begin{minipage}[t]{0.36\linewidth}\centering
  \includegraphics[width=\linewidth]{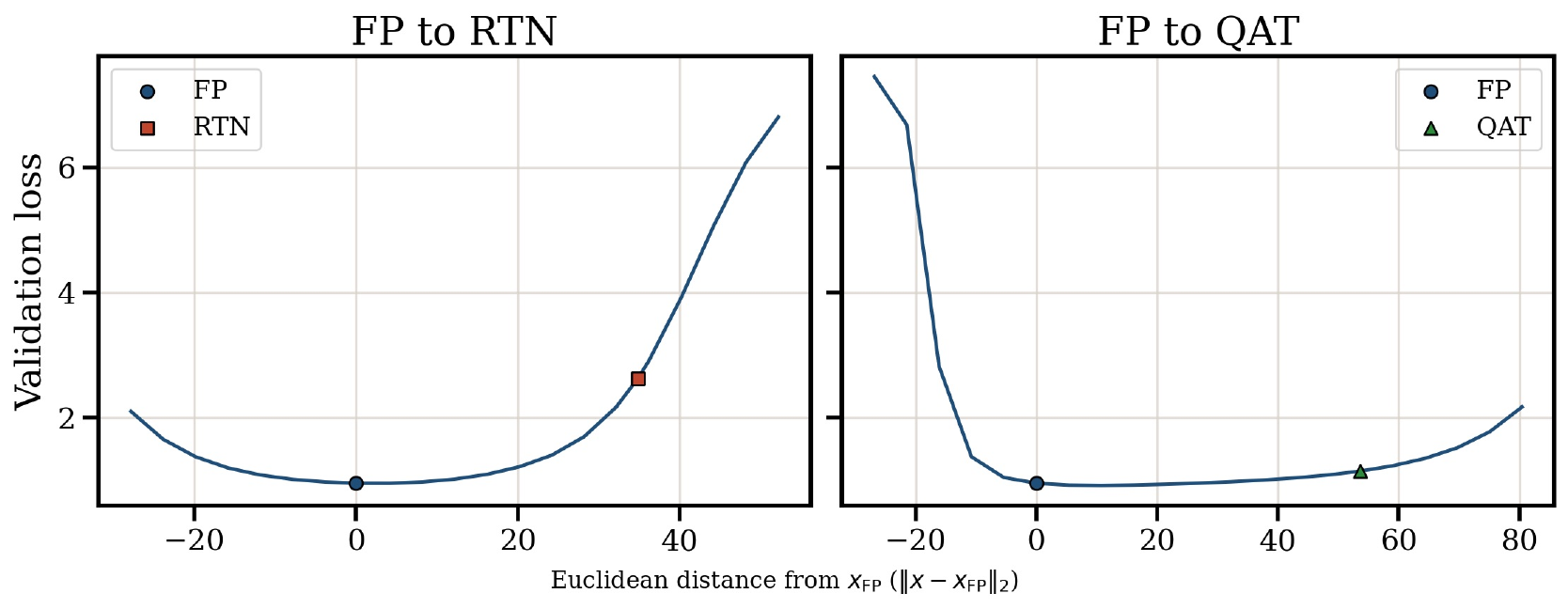}\\[-0.2em]
  {\tiny W3: pre-trained FP $\to$ RTN/QAT 1D profile}
\end{minipage}\hspace{0.006\linewidth}
\begin{minipage}[t]{0.36\linewidth}\centering
  \includegraphics[width=\linewidth]{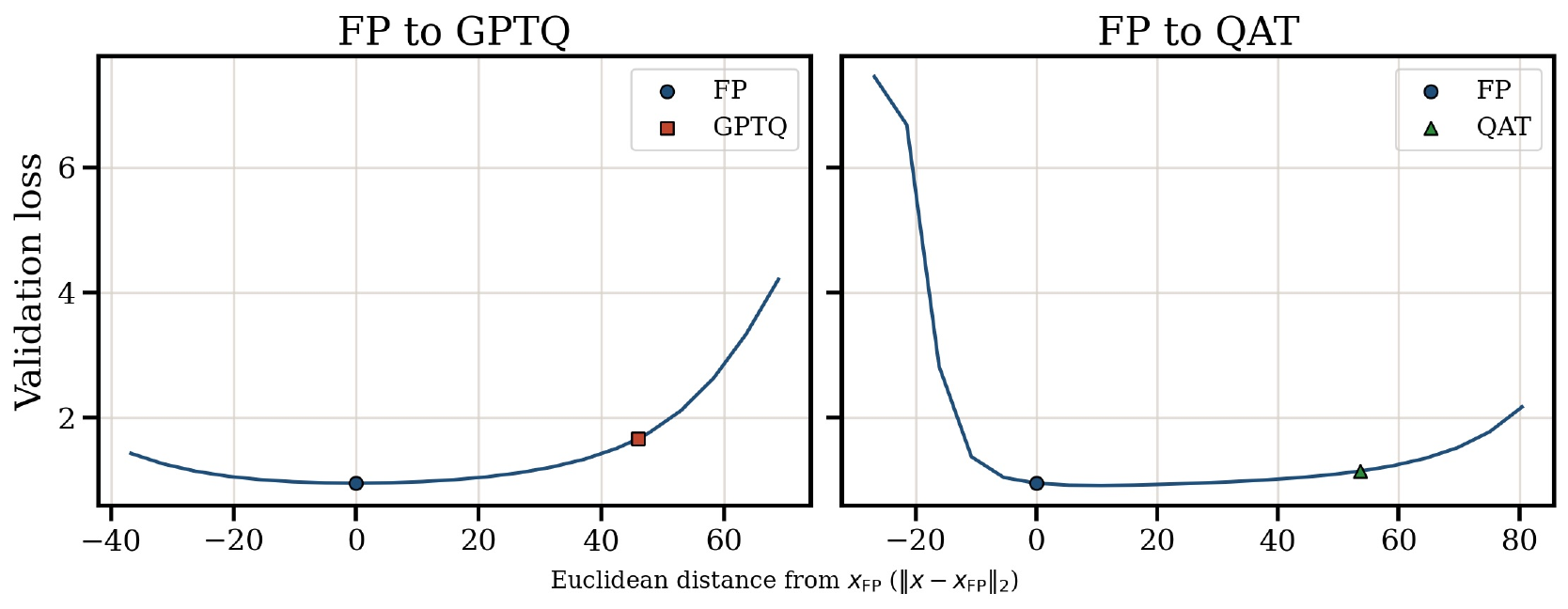}\\[-0.2em]
  {\tiny W3: pre-trained FP $\to$ GPTQ/QAT 1D profile}
\end{minipage}\hspace{0.025\linewidth}
\begin{minipage}[t]{0.2\linewidth}\centering
  \includegraphics[width=\linewidth]{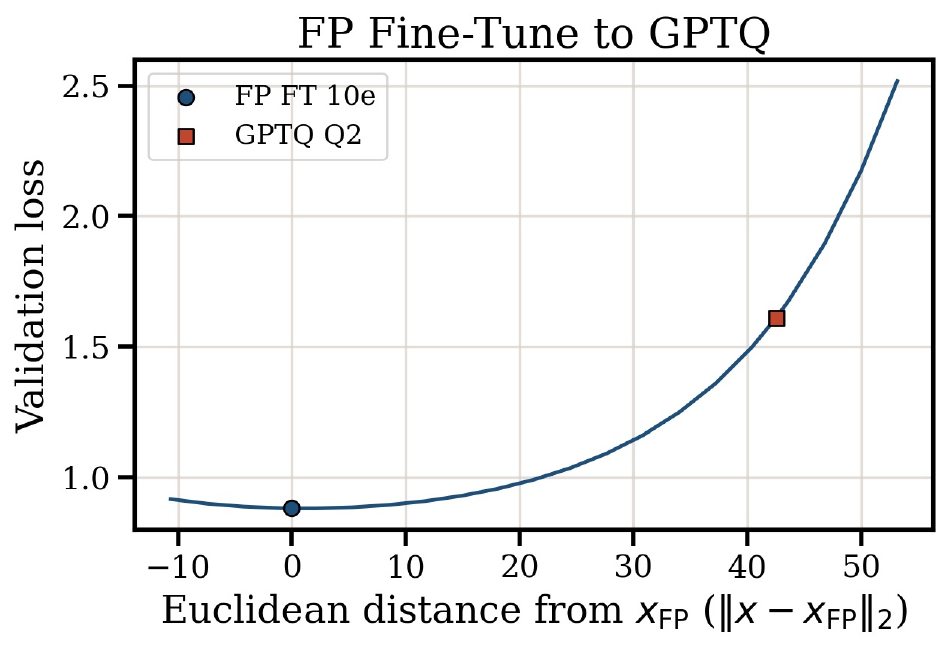}\\[-0.2em]
  {\tiny W3: fine-tuned FP $\to$ GPTQ}
\end{minipage}\\[0.35em]
\begin{minipage}[t]{0.36\linewidth}\centering
  \includegraphics[width=\linewidth]{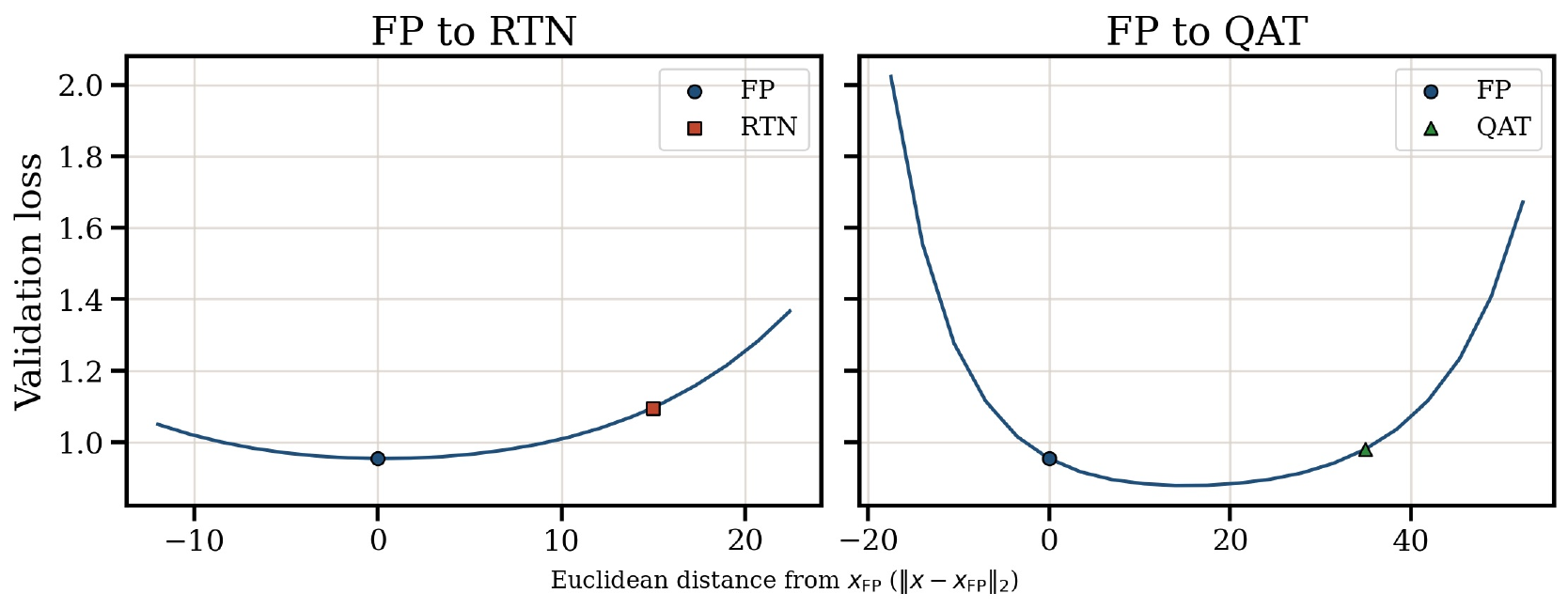}\\[-0.2em]
  {\tiny W4: pre-trained FP $\to$ RTN/QAT 1D profile}
\end{minipage}\hspace{0.006\linewidth}
\begin{minipage}[t]{0.36\linewidth}\centering
  \includegraphics[width=\linewidth]{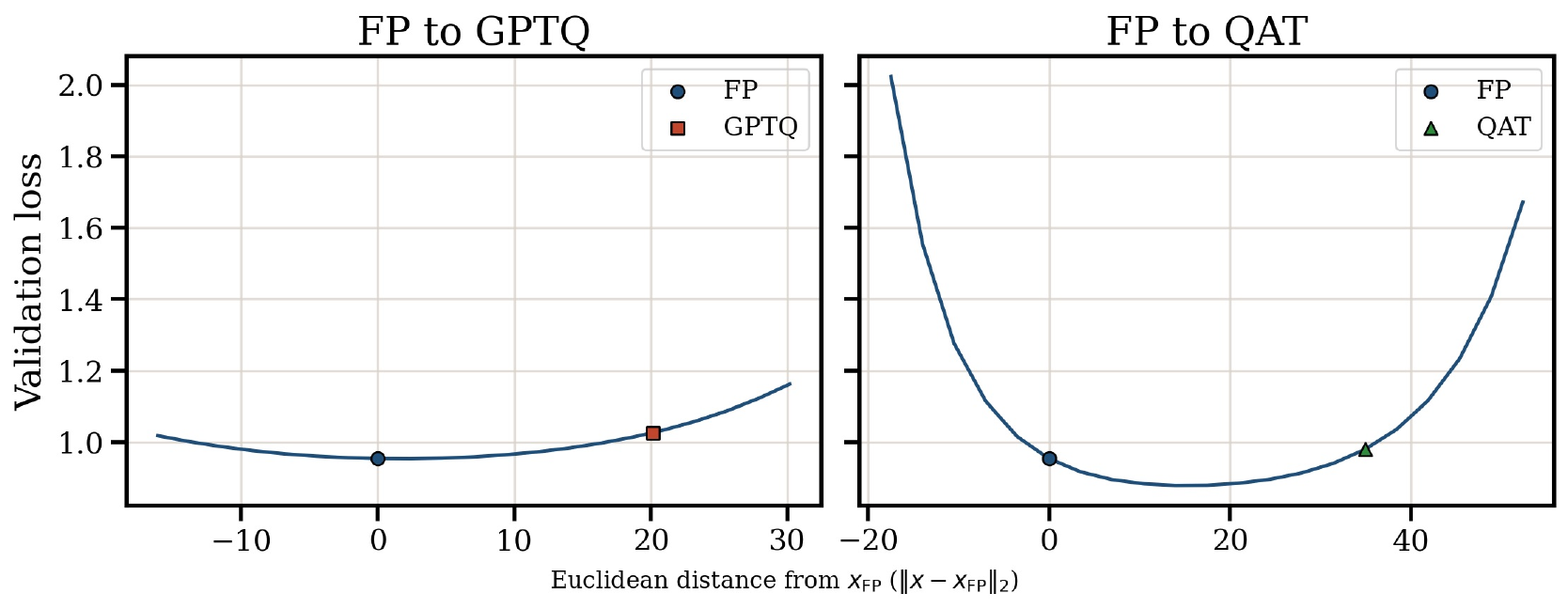}\\[-0.2em]
  {\tiny W4: pre-trained FP $\to$ GPTQ/QAT 1D profile}
\end{minipage}\hspace{0.025\linewidth}
\begin{minipage}[t]{0.2\linewidth}\centering
  \includegraphics[width=\linewidth]{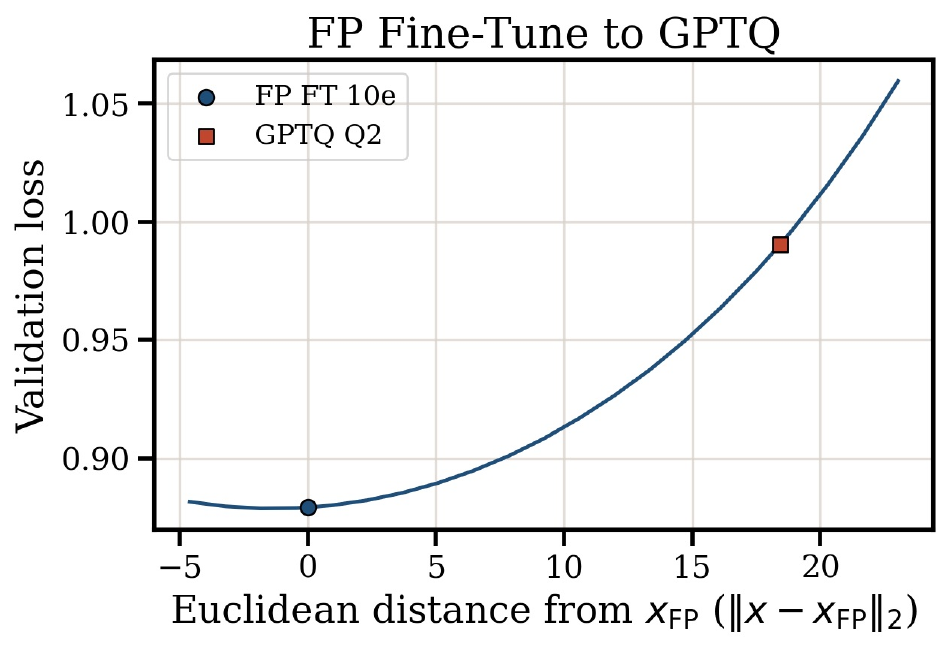}\\[-0.2em]
  {\tiny W4: fine-tuned FP $\to$ GPTQ}
\end{minipage}
\caption{{DeiT/ImageNet interpolation diagnostics.}
Rows correspond to W3/W4.  The first column uses RTN, the second uses GPTQ, and
the third column interpolates from the equal-budget full-precision
fine-tuned checkpoint toward the refit GPTQ endpoint.  These plots are
diagnostics for the PTQ-failure/QAT-recovery pattern and are not used as direct
proof of the full high-dimensional geometry.}
\label{fig:deit-profiles-extra}
\end{figure}

\FloatBarrier

\section{Toy examples}
\label{sec::toy-example}

\begin{example}[Generalization of the two-dimensional river-valley-basin loss]
\label{ex:running}
Fix an angle $\theta\in[\pi/6,\pi/4)$ and write $\tau=(\cos\theta,\sin\theta)$, $\nu=(-\sin\theta,\cos\theta)$, $\alpha=\cos\theta-\sin\theta$, and $\delta=\cos\theta+\sin\theta$. Let $w_0=(1,0)$ and define $t(w)=\langle \tau,w-w_0\rangle$ and $z(w)=\langle \nu,w-w_0\rangle$. 
For parameters $\mu, r, \epsilon>0$ and $u\in\R$, define $R= r+\sqrt{2\epsilon/\mu}$ and consider
\begin{equation}
\label{eq:fixed-running-loss}
    f_{\theta}(w) = \tfrac{1}{2}\bigl(t(w)+u\bigr)^2 + \tfrac{\mu}{2}\bigl(\max\{|z(w)|-r, 0\}\bigr)^2.
\end{equation}
The river center is $\cM=\{w_0+\lambda\tau \mid \lambda\in\R\}$. For any tubular neighborhood $U=\{w\in\R^2 \mid \dist(w,\cM)<R_U\}$ with $R_U>R$, the basin is $\cT=\{w\in U \mid |z(w)|\le R\}$.
\end{example}

\subsection{Verification of Assumption~\ref{ass:river-geometry}}
\label{app:toy-example-assum1}

\textbf{Claim:} Example~\ref{ex:running} satisfies Assumption~\ref{ass:river-geometry} on $U=\{w\in\R^2 \mid \dist(w,\cM)<R_U\}$ for $R_U>R$, with $\mathsf H=\mathbb{I}$, $L_P=1$,  $L=\max\{1,\mu\}$, $\epsilon_{\rm flat} = \epsilon$, and $c_\perp=\sqrt{2\mu\,\epsilon}$.

\begin{proof}
Since $(\tau,\nu)$ is an orthonormal basis, the nearest-point projection onto $\cM$ is $P_{\cM}(w)=w_0+t(w)\tau$, so $P_{\cM}$ is affine and $L_P=1$.
The loss in~\eqref{eq:fixed-running-loss} is $L$-smooth with $L=\max\{1,\mu\}$. 
For any $\pi=w_0+\lambda\tau\in\cM$, we have $t(\pi)=\lambda$ and $z(\pi)=0$, hence $\nabla f_{\theta}(\pi)=(\lambda+u)\tau\in T_{\cM}(\pi)$, which verifies the river-center condition. Next, if
$w=\pi+\xi\nu$ with $|\xi|\le R$, then $t(w)=t(\pi)$ and $z(w)=\xi$.
Therefore
$
    |f_{\theta}(w)-f_{\theta}(\pi)|
    = \tfrac{\mu}{2}(\max\{|\xi|-r, 0\})^2
    \le \tfrac{\mu}{2}(R-r)^2
    = \epsilon
$.
For $w\in U\setminus\cT$, i.e., $|z(w)|>R$, the outward normal direction is $\nu_{\mathsf H}(w)=\operatorname{sgn}(z(w))\nu$. Thus
\[
    \left\langle \nabla f_{\theta}(w),\nu_{\mathsf H}(w)\right\rangle
    = \mu(|z(w)|-r)
    \ge \mu(R-r)
    = \sqrt{2\mu\,\epsilon}
    = c_\perp .
\]
This proves that the toy example satisfies
Assumption~\ref{ass:river-geometry}.
\end{proof}

\subsection{Failure of Hessian-based PTQ}
\label{app:toy-example-PTQ}

\textbf{Claim:}
Consider the function $f_\theta$ in Example~\ref{ex:running} with $u\in(\alpha/2,\alpha)$ and $R<\delta$.
For a grid scale $\rho>0$ and bitwidth $B\ge2$, let
$\cQ_{\rho,B}=\rho\{0,\pm1,\ldots,\pm(2^{B-1}-1)\}^2$. Define a grid point $q_{\rm g}(\rho)=(\rho,0)$. There exists an open interval $\mathcal I_{\rm fail}$ containing $\rho=1$ such that, for every $\rho\in\mathcal I_{\rm fail}$ and every $B \ge 2$, the following hold.

\begin{enumerate}[leftmargin=0.3in]
    \item[(a)] The grid point $q_{\rm g}(\rho)$ lies in $\cT$. Moreover, fixing $w_{\rm fp}$ to be any stationary point of $f_\theta$ at which the Hessian exists, the minimizer $q_{\ast}(\rho)$ of the Hessian proxy in~\eqref{eq:hessian-ptq} lies outside $\cT$.

    \item[(b)] The loss gap $f_\theta(q_{\ast}(\rho))-f_\theta(q_{\rm g}(\rho))$ grows linearly in the sharpness parameter $\mu$.
\end{enumerate}

\begin{proof}
Write $c=\cos\theta$, $s=\sin\theta$, so that $\alpha=c-s$ and $\delta=c+s$. Fix $\zeta=(R+\delta)/2$. Since $R<\delta$, we have $R<\zeta<\delta$.
The stationary points of $f_\theta$ are precisely those satisfying
$t(w)=-u$ and $|z(w)|\le r$. At every stationary point with
$|z(w)|<r$, the Hessian exists and equals $\nabla^2 f_\theta(w_{\rm fp})=\tau\tau^\top$. Thus, for such a choice of $w_{\rm fp}$, the Hessian proxy is $S(q)=\tfrac12(t(q)+u)^2$.

For integer pairs $(i,j)$, write $q_{ij}(\rho)=\rho(i,j)$. Then
\[
    t_{ij}(\rho)=c(\rho i-1)+s\rho j,\qquad
    z_{ij}(\rho)=-s(\rho i-1)+c\rho j,
\]
and, since $\tau,\nu$ are orthonormal, $t_{ij}(\rho)^2+z_{ij}(\rho)^2=(\rho i-1)^2+(\rho j)^2$. Let $q_{\rm g}(\rho)=(\rho,0)$ and
$q_{\rm b}(\rho)=(0,\rho)$. Both are feasible for every signed
symmetric uniform quantizer $\cQ_{\rho,B}$ with $B\ge2$. 
Choose an open interval $\mathcal I_{\rm fail}$ containing $1$, sufficiently small such that for every $\rho\in\mathcal I_{\rm fail}$,
\[
    s|\rho-1|<r,\qquad s+c\rho>R,\qquad
    c-s\rho-u>0.
\]
These conditions can be imposed simultaneously due to our assumption that $u \in (\alpha/2, \alpha)$ and $R < \delta$.
For $q_{\rm g}(\rho)$, we have $|z(q_{\rm g}(\rho))|=s|\rho-1|<r$, so $q_{\rm g}(\rho)\in\cT$.
For $q_{\rm b}(\rho)$, we have $z(q_{\rm b}(\rho))=s+c\rho>R$, so $q_{\rm b}(\rho)\notin\cT$. Moreover,
\[
    |t(q_{\rm b}(\rho))+u|=c-s\rho-u,
    \qquad
    S(q_{\rm b}(\rho))=\tfrac12 (c-s\rho-u)^2.
\]
We next show that every grid point with $|z|<\zeta$ has strictly larger Hessian proxy value than $q_{\rm b}(\rho)$. Let $\mathcal N=\{(1,0),(0,0),(2,0),(1,1),(1,-1)\}$. At $\rho=1$, by direct calculation, for every $(i,j)\in\mathcal N$,
\[
    |t_{ij}(1)+u|>\alpha-u=|t(q_{\rm b}(1))+u|.
\]
Indeed, the only nontrivial comparisons use $u>\alpha/2$ and $s>\alpha$, where $s>\alpha$ follows from $\theta\in[\pi/6,\pi/4)$. Therefore, by continuity, after shrinking $\mathcal I_{\rm fail}$ if necessary,
\[
    |t_{ij}(\rho)+u|>c-s\rho-u,
    \qquad \forall (i,j)\in\mathcal N, \rho\in\mathcal I_{\rm fail}.
\]
It remains to handle $(i,j)\notin\mathcal N$. At $\rho=1$,
\[
    (i-1)^2+j^2\ge2,
    \qquad
    \zeta^2+\alpha^2<\delta^2+\alpha^2=2.
\]
Thus, by discreteness of the integer lattice and continuity in $\rho$, after
shrinking $\mathcal I_{\rm fail}$ once more,
\[
    (\rho i-1)^2+(\rho j)^2>\zeta^2+(c-s\rho)^2,
    \qquad
    \forall(i,j)\notin\mathcal N, \rho\in\mathcal I_{\rm fail}.
\]
Hence, if $(i,j)\notin\mathcal N$ and $|z_{ij}(\rho)|<\zeta$, then
\[
    t_{ij}(\rho)^2
    =(\rho i-1)^2+(\rho j)^2-z_{ij}(\rho)^2
    >(c-s\rho)^2.
\]
Therefore $|t_{ij}(\rho)|>c-s\rho$, and so
\[
    |t_{ij}(\rho)+u|
    \ge |t_{ij}(\rho)|-u
    > c-s\rho-u.
\]
Combining the cases of $(i,j) \in \mathcal N$ and $(i,j) \notin \mathcal N$, we have shown that
\[
    |z(q_{ij}(\rho))|<\zeta
    \quad\Longrightarrow\quad
    S(q_{ij}(\rho))>S(q_{\rm b}(\rho)).
\]
Thus, $S(q_{ij}(\rho))\le S(q_{\rm b}(\rho))$ implies $|z(q_{ij}(\rho))|\ge\zeta$. 
Since $q_{\rm b}(\rho) \in \cQ_{\rho,B}$, every Hessian-proxy minimizer $q_{\ast}(\rho)$ satisfies $S(q_{\ast}(\rho))\le S(q_{\rm b}(\rho))$. Hence $|z(q_{\ast}(\rho))|\ge\zeta>R$, so
$q_{\ast}(\rho) \notin\cT$. This proves part (a).

For part (b), since $|z(q_{\rm g}(\rho))|<r$, we have $f_\theta(q_{\rm g}(\rho)) = \tfrac12\bigl(u+c(\rho-1)\bigr)^2$.
On the other hand, every proxy minimizer satisfies $|z(q_{\ast}(\rho))|\ge\zeta>r$, so
$f_\theta(q_{\ast}(\rho))\ge \tfrac{\mu}{2}(\zeta-r)^2$.
Therefore,
\[
    f_\theta(q_{\ast}(\rho))-f_\theta(q_{\rm g}(\rho))
    \ge
    \tfrac{\mu}{2}(\zeta-r)^2
    -
    \tfrac12\bigl(u+c(\rho-1)\bigr)^2.
\]
Since $\zeta=(R+\delta)/2$ and $R>r$, we have
$\zeta-r>(\delta-r)/2$, and hence
\[
    f_\theta(q_{\ast}(\rho))-f_\theta(q_{\rm g}(\rho))
    \ge
    \tfrac{\mu}{8}(\delta-r)^2
    -
    \tfrac12\bigl(u+c(\rho-1)\bigr)^2,
\]
which grows linearly in $\mu$.
\end{proof}

\section{Matrix Factorization}

\subsection{Verification of Assumption~\ref{ass:river-geometry}}
\label{app:mf-assum1}

We verify Assumption~\ref{ass:river-geometry} for Example~\ref{ex:mf_sec2}.  The purpose is only to
show that, after shrinking the local patch $\Omega$, the matrix factorization
objective satisfies Assumption~\ref{ass:river-geometry} for a fixed anisotropic
metric.  All constants below are local constants depending on the chosen patch.

\begin{proof}%
Write $X=\binom{P}{Z}$ with $P\in\R^{r\times k}$ and
$Z\in\R^{(d-r)\times k}$.  Then
\[
    f(P,Z)
    =
    \|PP^\top-D\|_F^2
    +
    2\|PZ^\top\|_F^2
    +
    \|ZZ^\top\|_F^2 .
\]
Let
\[
    \cM=\left\{\binom{P}{0} \,\middle|\, P\in\Omega\right\},
    \qquad
    U=\left\{\binom{P}{Z} \,\middle|\, P\in\Omega,\ \|Z\|_F<R_U\right\}.
\]
On this local product neighborhood, the nearest-point projection is
$
    P_{\cM}\binom{P}{Z}=\binom{P}{0}
$,
so $L_P=1$ and $N_{\cM}(P,0)=\binom{0}{\R^{(d-r) \times k}}$.

Since $\Omega$ is bounded and every $P\in\overline{\Omega}$ has full row rank, define $\mu=\inf_{P\in\overline{\Omega}}\sigma_{\min}(P)>0$ and $M=\sup_{P\in\overline{\Omega}}\|P\|_{\rm op}<\infty$.
For each full-row-rank matrix $P$, let $\Pi_P=P^\top(PP^\top)^{-1}P$ be the orthogonal projector onto the row space of $P$.  Fix the base projector $\Pi_0=P_0^\top(P_0P_0^\top)^{-1}P_0$.
Choose positive weights $\lambda_{\rm s}\ge \lambda_{\rm f}>0$ as follows:
\[
    \lambda_{\rm f}
    =
    \max\left\{
        \frac{4R^2}{R_U^2},
        \frac{2R^2}{\sqrt{\epsilon_{\rm flat}}}
    \right\},
    \qquad
    \lambda_{\rm s}
    =
    \max\left\{
        \lambda_{\rm f},
        \frac{8M^2R^2}{\epsilon_{\rm flat}}
    \right\}.
\]
Define
$
    C_0=\lambda_{\rm s}\Pi_0+\lambda_{\rm f}(I-\Pi_0)
$.
After vectorizing matrix pairs, define the linear operator
\[
    \mathsf H(\Delta P,Z)=(\Delta P,ZC_0).
\]
Since $C_0\succeq \lambda_{\rm f}I$, the matrix $\mathsf H$ is symmetric positive
definite.  Moreover, it is block diagonal with respect to the splitting
$(\Delta P,Z)$, so $\mathsf H N_{\cM}(P,0)\subseteq N_{\cM}(P,0)$.

Since the map $P\mapsto \Pi_P$ is continuous on the full-row-rank set, we may shrink $\Omega$ around $P_0$ so that $\sup_{P\in\Omega}\|\Pi_P-\Pi_0\|_{\rm op}\le \delta_{\mathsf H}$,
where $\delta_{\mathsf H}>0$ will be chosen below. The bounds $\sigma_{\min}(P) \ge \mu_{\Omega}$ and $\|P\|_{\rm op} \le M_{\Omega}$ remain valid on the smaller patch.

\noindent\textbf{Smoothness.}
For the full variable $X$, $\nabla f(X)=4(XX^\top-M^\star)X$.
For any perturbation $E$,
\[
    \nabla^2 f(X)[E]
    =
    4(EX^\top+XE^\top)X
    +
    4(XX^\top-M^\star)E .
\]
Thus,
$
    \|\nabla^2 f(X)[E]\|_F
    \le
    (12\|X\|_{\rm op}^2+4\|M^\star\|_{\rm op})\|E\|_F
$.
On $U$, we have $\|X\|_{\rm op}^2 \le M^2+R_U^2$ and $\|M^\star\|_{\rm op}=\|D\|_{\rm op}$. Therefore $f$ is $L$-smooth on $U$ with $L=12(M^2+R_U^2)+4\|D\|_{\rm op}$.

\noindent\textbf{River center.}
At a point on the river center, the river-center condition holds because
\[
    \nabla f(P,0)
    =
    \binom{4(PP^\top-D)P}{0}
    \in T_{\cM}(P,0).
\]

\noindent\textbf{Metric comparison after freezing the row space.}
For each $P\in\Omega$, define the moving metric
\[
    C_P=\lambda_{\rm s}\Pi_P+\lambda_{\rm f}(I-\Pi_P).
\]
Let $D_0(Z)^2=\langle Z,ZC_0\rangle$ and $D_P(Z)^2=\langle Z,ZC_P\rangle$. Because
$
    C_0-C_P=(\lambda_{\rm s}-\lambda_{\rm f})(\Pi_0-\Pi_P)
$,
we have
$
    \|C_0-C_P\|_{\rm op}
    \le
    \lambda_{\rm s}\delta_{\mathsf H}
$.
Also $C_P\succeq \lambda_{\rm f}I$, hence $\lambda_{\rm f} \|Z\|_F^2\le D_P(Z)^2$.
Therefore
\[
    |D_0(Z)^2-D_P(Z)^2|
    = |\langle Z,Z(C_0-C_P)\rangle| 
    \le \lambda_{\rm s}\delta_{\mathsf H}\|Z\|_F^2 
    \le \frac{\lambda_{\rm s}\delta_{\mathsf H}}{\lambda_{\rm f}}D_P(Z)^2 .
\]
Choosing
$
    \delta_{\mathsf H}\le \frac{\lambda_{\rm f}}{4\lambda_{\rm s}},
$
we obtain the uniform comparison
\begin{equation}\label{eq:metric_compare}
    \frac34D_P(Z)^2
    \le
    D_0(Z)^2
    \le
    \frac54D_P(Z)^2.
\end{equation}

\paragraph{Anisotropic basin.}
Define the basin by
\[
    \cT
    =
    \left\{
        \binom{P}{Z}\in U \,\middle|\,
        D_0(Z)=\|Z\|_{C_0}\le R
    \right\}.
\]
The choice $\lambda_{\rm f}\ge 4R^2/R_U^2$ guarantees that
$D_0(Z)\le R$ implies $\|Z\|_F\le R/\sqrt{\lambda_{\rm f}}\le R_U/2$, so
$\cT\subset U$.
Suppose $D_0(Z)\le R$.  By the metric comparison \eqref{eq:metric_compare},
$
    D_P(Z)^2\le \frac43R^2
$.
Decompose $Z=Z_{\rm s}+Z_{\rm f}$ with $Z_{\rm s}=Z\Pi_P$ and $Z_{\rm f}=Z(I-\Pi_P)$. Since $P(I-\Pi_P)=0$, we have $PZ_{\rm f}^\top=0$, and hence
\[
    f(P,Z)-f(P,0)
    =
    2\|PZ_{\rm s}^\top\|_F^2+\|ZZ^\top\|_F^2.
\]
Using $\|P\|_{\rm op}\le M$, $D_P(Z)^2\ge
\lambda_{\rm s}\|Z_{\rm s}\|_F^2$, and
$D_P(Z)^2\ge \lambda_{\rm f}\|Z\|_F^2$, we obtain
\begin{align*}
    f(P,Z)-f(P,0)
    &\le 2M^2\frac{D_P(Z)^2}{\lambda_{\rm s}} + \frac{D_P(Z)^4}{\lambda_{\rm f}^2} \\
    &\le \frac{8M^2R^2}{3\lambda_{\rm s}} + \frac{16R^4}{9\lambda_{\rm f}^2}
    \;\le\; \frac{\epsilon_{\rm flat}}{3} + \frac{4\epsilon_{\rm flat}}{9}
    = \frac{7\epsilon_{\rm flat}}{9}
\end{align*}
Thus $|f(P,Z)-f(P,0)|\le \epsilon_{\rm flat}$ for all $(P,Z)\in\cT$. This verifies the anisotropic-basin condition.
It remains to verify the sharp-wall condition for the fixed metric $C_0$. 

\noindent\textbf{Step 1.} We first prove the estimate for the moving metric $C_P$. Notice that
\begin{align*}
    \langle \nabla_Z f(P,Z),ZC_P\rangle
    &= \langle 4ZP^\top P+4ZZ^\top Z,\;ZC_P\rangle \\
    &= 4\lambda_{\rm s} \langle Z_{\rm s} P^\top P,\;Z_{\rm s} \rangle 
     + 4\langle ZZ^\top Z, Z C_P \rangle \\
    &\ge 4\lambda_{\rm s}\|PZ_{\rm s}^\top\|_F^2 + 4\lambda_{\rm f}\|ZZ^\top\|_F^2 .
\end{align*}
Since $\sigma_{\min}(P)\ge \mu$ on $\Omega$,
$
    \|PZ_{\rm s}^\top\|_F\ge \mu\|Z_{\rm s}\|_F
$.
Also, with $m_0=\min\{d-r,k\}$, it follows from Cauchy-Schwarz inequality that
$
    m_0 \|ZZ^\top\|_F^2\ge \|Z\|_F^4
$.
Thus
\[
    \langle \nabla_Z f(P,Z),ZC_P\rangle
    \ge
    4\lambda_{\rm s}\mu^2\|Z_{\rm s}\|_F^2
    +
    \frac{4\lambda_{\rm f}}{m_0}\|Z\|_F^4.
\]
If $D_P(Z) = \sqrt{\lambda_{\rm s}\|Z_{\rm s}\|^2_F + \lambda_{\rm f}\|Z_{\rm f}\|^2_F} > R$, then either $\sqrt{2\lambda_{\rm s}}\|Z_{\rm s}\|_F \ge {D_P(Z)}$ or $\sqrt{2\lambda_{\rm f}}\|Z_{\rm f}\|_F \ge {D_P(Z)}$. In the first case,
\[
    \frac{\langle \nabla_Z f(P,Z),ZC_P\rangle}{D_P(Z)}
    \ge
    2\mu^2D_P(Z)
    \ge
    2\mu^2R.
\]
In the second case,
$
    \|Z\|_F
    \ge \|Z_{\rm f}\|_F
    \ge {D_P(Z)}/{\sqrt{2\lambda_{\rm f}}}
$,
and therefore
\[
    \frac{\langle \nabla_Z f(P,Z),ZC_P\rangle}{D_P(Z)}
    \ge
    \frac{D_P(Z)^3}{m_0\lambda_{\rm f}}
    \ge
    \frac{R^3}{m_0\lambda_{\rm f}}.
\]
Hence, whenever $D_P(Z)>R$,
\[
    \frac{\langle \nabla_Z f(P,Z),ZC_P\rangle}{D_P(Z)}
    \ge 
    \min\left\{
        2\mu^2R,\,
        \frac{R^3}{m_0\lambda_{\rm f}}
    \right\}
    \triangleq c_\ast.
\]

\noindent\textbf{Step 2.} Next, we turn to the estimate for the fixed metric $C_0$.
Now suppose $(P,Z)\in U\setminus\cT$, i.e. $D_0(Z)>R$.  By the metric comparison \eqref{eq:metric_compare},
$
    D_P(Z) \ge 2 D_0(Z)/\sqrt5 > 2R/\sqrt5
$.
Applying the moving-metric estimate with threshold $2R/\sqrt5$ gives
\[
    \frac{\langle \nabla_Z f(P,Z),ZC_P\rangle}{D_0(Z)}
    \ge
    \frac{16}{25} c_\ast.
\]
We next control the error caused by replacing $C_P$ with the frozen metric $C_0$.
On $U$,
\[
    \|\nabla_Z f(P,Z)\|_F
    \le
    4\|Z\|_F\|P\|_{\rm op}^2+4\|Z\|_F^3
    \le
    4R_U(M^2+R_U^2)
    \triangleq G.
\]
Since $D_0(Z)\ge \sqrt{\lambda_{\rm f}}\|Z\|_F$, we have
\[
    \left|
    \frac{\langle \nabla_Z f(P,Z),Z(C_0-C_P)\rangle}{D_0(Z)}
    \right|
    \le
    \frac{\|\nabla_Z f(P,Z)\|_F\|Z\|_F\|C_0-C_P\|_{\rm op}}{D_0(Z)}
    \le
    \frac{G\lambda_{\rm s}\delta_{\mathsf H}}{\sqrt{\lambda_{\rm f}}}.
\]
By further shrinking $\Omega$, we may ensure
$
    \delta_{\mathsf H} \le \frac{c_\ast\sqrt{\lambda_{\rm f}}}{4G\lambda_{\rm s}}.
$
Therefore
\[
    \frac{\langle \nabla_Z f(P,Z),ZC_0\rangle}{D_0(Z)}
    \ge
    \left(\frac{16}{25}-\frac14\right)c_\ast
    \ge
    \frac14c_\ast.
\]
Finally, Assumption~\ref{ass:river-geometry} normalizes the sharp-wall direction
as
\[
    \nu_{\mathsf H}(P,Z)
    =
    \frac{\mathsf H\bigl((P,Z)-P_{\cM}(P,Z)\bigr)}
         {\left\|\mathsf H\bigl((P,Z)-P_{\cM}(P,Z)\bigr)\right\|}
    =
    \frac{(0,ZC_0)}{\|ZC_0\|_F}.
\]
Since
$
    \|ZC_0\|_F\le \sqrt{\lambda_{\rm s}}\,D_0(Z)
$,
we obtain
\[
    \left\langle
        \nabla f(P,Z),
        \nu_{\mathsf H}(P,Z)
    \right\rangle
    = \frac{\langle \nabla_Z f(P,Z),ZC_0\rangle}{\|ZC_0\|_F}
    \ge \frac{c_\ast}{4\sqrt{\lambda_{\rm s}}}.
\]
Thus the sharp-valley condition holds with
$
    c_\perp = \frac{c_\ast}{4\sqrt{\lambda_{\rm s}}} >0
$.

Combining the smoothness, basin, river-center, and sharp-valley estimates proves
that Example~\ref{ex:mf_sec2} satisfies Assumption~\ref{ass:river-geometry} on
$U$ with the fixed metric $\mathsf H$.
\end{proof}

\subsection{Numerical Simulations of Matrix Factorization}
\label{app:mf-simulations}

We simulate the symmetric matrix factorization objective of
Example~\ref{ex:mf_sec2},
\[
    f(X) = \|XX^\top - M^\star\|_F^2,
    \qquad
    M^\star = \begin{pmatrix} D & 0 \\ 0 & 0 \end{pmatrix},
    \quad
    D \in \R^{r\times r} \text{ diagonal}, \; D \succ 0,
\]
with $X \in \R^{d\times k}$ and $k \ge r$, in the high-dimensional
setting $d = 100$, $r = k = 5$ ($D = I_r$, $dk = 500$ parameters).
The loss landscape is no longer drawable directly, so we apply the
same family of diagnostics used on Llama in the main text: a 1D loss
profile from FP toward each anchor, a 2D contour on the
FP/PTQ/QAT affine plane, and a 3D river-cross surface in the plane
spanned by FP$\to$QAT and a random perpendicular direction.

\paragraph{Pipeline.}
Stage~1 is full-precision gradient descent on $f(X)$ from a small
random initialization, $X_0 \sim 0.30 \,\mathcal{N}(0, I)$, with
step size $\eta_{\rm FP} = 0.020$ for $4000$ iterations; this drives
$f(X_{\rm fp})$ to numerical zero ($\approx 10^{-30}$).  Stage~2
applies PTQ by per-entry rounding,
$X_{\rm PTQ} = Q(X_{\rm fp})$ with $Q(X) = \rho\,\mathrm{round}(X/\rho)$
and $\rho = 0.30$, chosen so that no FP optimum lies on the grid
($f(X) > 0$ for every $B \in \rho \,\Z^{d\times k}$).  Stage~3 runs
STE-QAT continuing from $X_{\rm fp}$,
\[
    X_{k+1} = X_k - \eta_k \, \nabla f\bigl(Q(X_k)\bigr),
\]
for $1500$ iterations with cosine-decayed step size
$\eta_k = \tfrac{1}{2}\bigl(1 + \cos(\pi k / 1500)\bigr) \cdot 0.010$,
mirroring the cosine schedule used for the LLM runs.  We report the
deployed loss $f\bigl(Q(X_{\rm QAT})\bigr)$ on the final QAT
iterate.

\paragraph{Results.}
The deployed PTQ loss is $f(X_{\rm PTQ}) = 0.7502$, a substantial gap
from the FP optimum because the FP solution lies strictly off the
grid in every coordinate.  STE-QAT lowers the deployed loss to
$f\bigl(Q(X_{\rm QAT})\bigr) = 0.2543$, a $66\%$ reduction over the
PTQ initialization at the \emph{same} grid resolution and weight
scope.  The displacements
$\Delta_{\rm PTQ} = X_{\rm PTQ} - X_{\rm fp}$ and
$\Delta_{\rm QAT} = Q(X_{\rm QAT}) - X_{\rm fp}$ have Frobenius norms
$0.474$ and $0.594$ and meet at an angle of $95.3^\circ$ --- they are
nearly orthogonal, so QAT is not a small correction along the
rounding direction but a separate, comparably-sized move.

Figure~\ref{fig:mf-highdim-1d} shows one-dimensional loss profiles
along these two directions.  The FP$\to$PTQ profile rises smoothly
through the PTQ anchor at $t=1$ to its $f = 0.7502$ value, while the
FP$\to$QAT profile reaches a lower $f = 0.2543$ at $t=1$; both rise
monotonically because $X_{\rm fp}$ is the global minimum, but the
slope along $\Delta_{\rm QAT}$ is markedly gentler, indicating a
basin-aligned displacement compared to the rounding direction.
Figure~\ref{fig:mf-highdim-landscape} renders the 2D loss landscape
on the affine plane through the three anchors: FP, PTQ, and QAT fall
at distinct corners with QAT placed inside the warm-color (low-loss)
region while PTQ sits on a steeper part of the surface.
Figure~\ref{fig:mf-highdim-river} tests the river-valley hypothesis
directly: in the plane spanned by FP$\to$QAT and a random
perpendicular direction (rescaled to $\|\Delta_{\rm FP\to QAT}\|$ and
averaged over $5$ random seeds, same construction as the
ResNet/DeiT/Llama river-cross plots), the valley along $b = 0$ is
narrow and deep; loss stays low along FP$\to$QAT but rises sharply
with $|b|$, consistent with
Assumption~\ref{ass:river-geometry} on this finite-dimensional
matrix factorization instance.

\begin{figure}[H]
\centering
\subfigure[FP $\to$ PTQ 1D profile]{
  \includegraphics[width=0.3\linewidth]{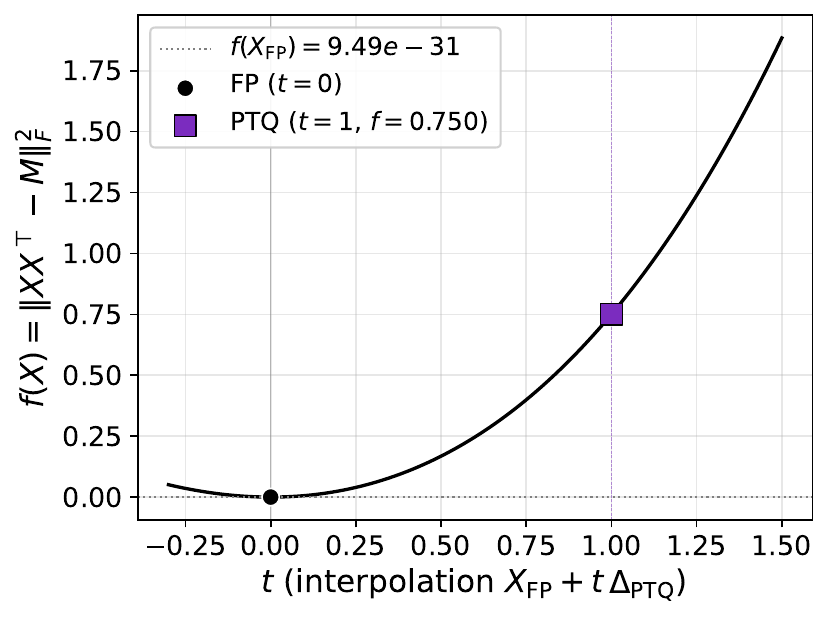}
  \label{fig:mf-highdim-1d-ptq}
}\hspace{0.5in}%
\subfigure[FP $\to$ QAT 1D profile]{
  \includegraphics[width=0.3\linewidth]{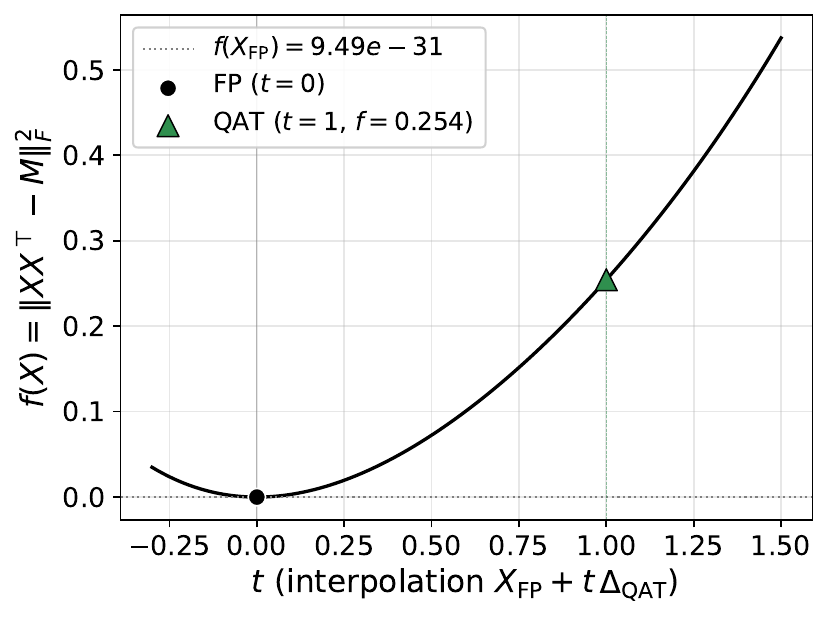}
  \label{fig:mf-highdim-1d-qat}
}
\caption{\textbf{One-dimensional loss profiles for the high-dimensional
Matrix Factorization simulation (Appendix~\ref{app:mf-simulations}).} Loss along the
linear interpolations $X_{\rm fp} + t\,\Delta_{\rm PTQ}$ (left) and
$X_{\rm fp} + t\,\Delta_{\rm QAT}$ (right).  The PTQ anchor at $t=1$
sits at $f=0.7502$; the QAT anchor at $t=1$ sits at $f=0.2543$.
$X_{\rm fp}$ is the global minimum so both profiles start at zero,
but the slope along $\Delta_{\rm QAT}$ is gentler than along
$\Delta_{\rm PTQ}$ --- the QAT direction follows a basin-aligned
displacement rather than an off-river rounding move.  Within each
panel the dotted line marks $f(X_{\rm fp})$ and the dashed vertical
line marks the anchor location at $t=1$.}
\label{fig:mf-highdim-1d}
\end{figure}

\begin{figure}[H]
\centering
\includegraphics[width=0.4\linewidth]{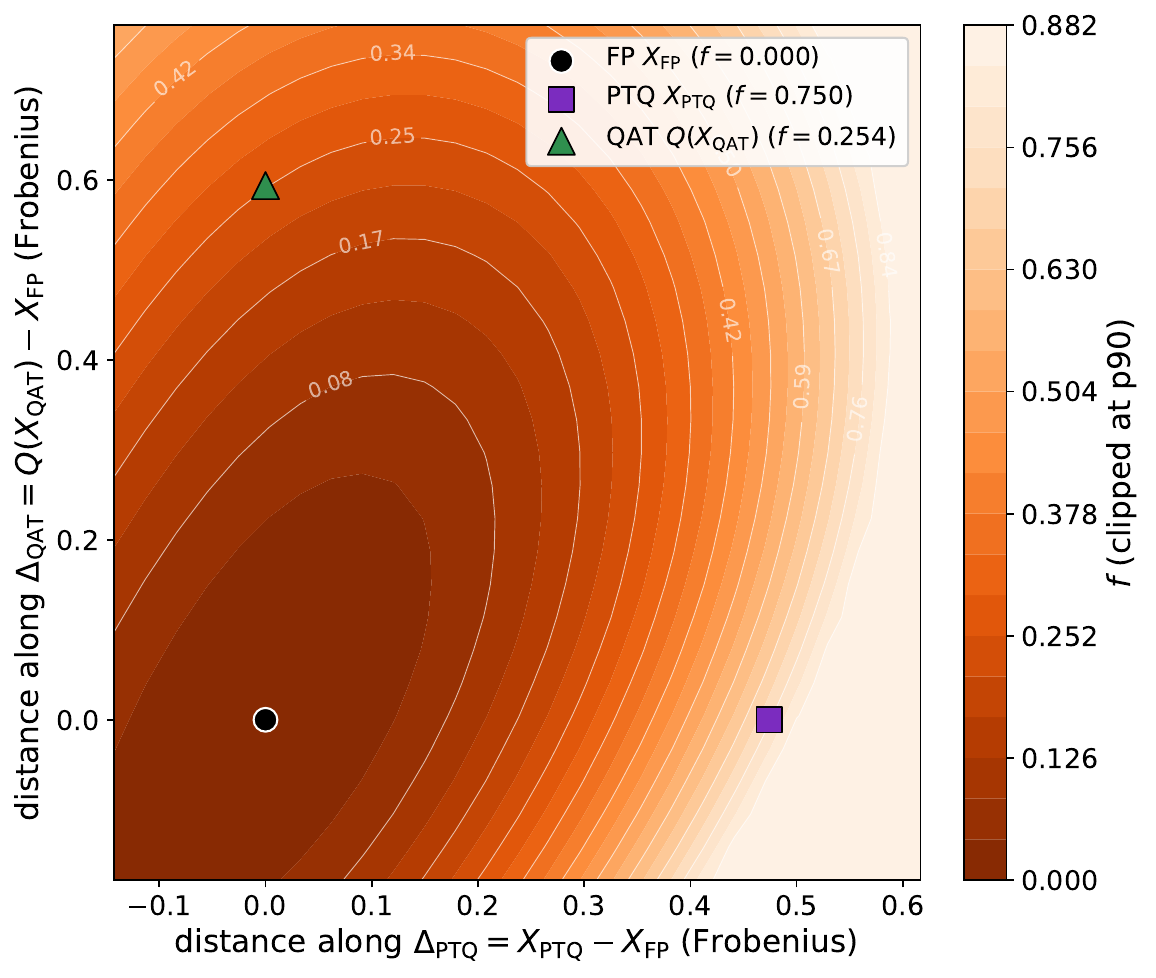}
\caption{\textbf{FP/PTQ/QAT 2D loss landscape for the high-dimensional
Matrix Factorization simulation.}  Loss on the affine plane
$X(a, b) = X_{\rm fp} + a\,\Delta_{\rm PTQ} + b\,\Delta_{\rm QAT}$
through the three anchors (FP at $(0,0)$, PTQ at $(1,0)$, QAT at
$(0,1)$).  Axes are in Frobenius distance units
($\|\Delta_{\rm PTQ}\| = 0.474$, $\|\Delta_{\rm QAT}\| = 0.594$,
angle $95.3^\circ$).  PTQ sits on a steeper part of the surface than
QAT; the QAT anchor lands inside the warm-color (low-loss) basin.}
\label{fig:mf-highdim-landscape}
\end{figure}

\begin{figure}[H]
\centering
\includegraphics[width=0.5\linewidth]{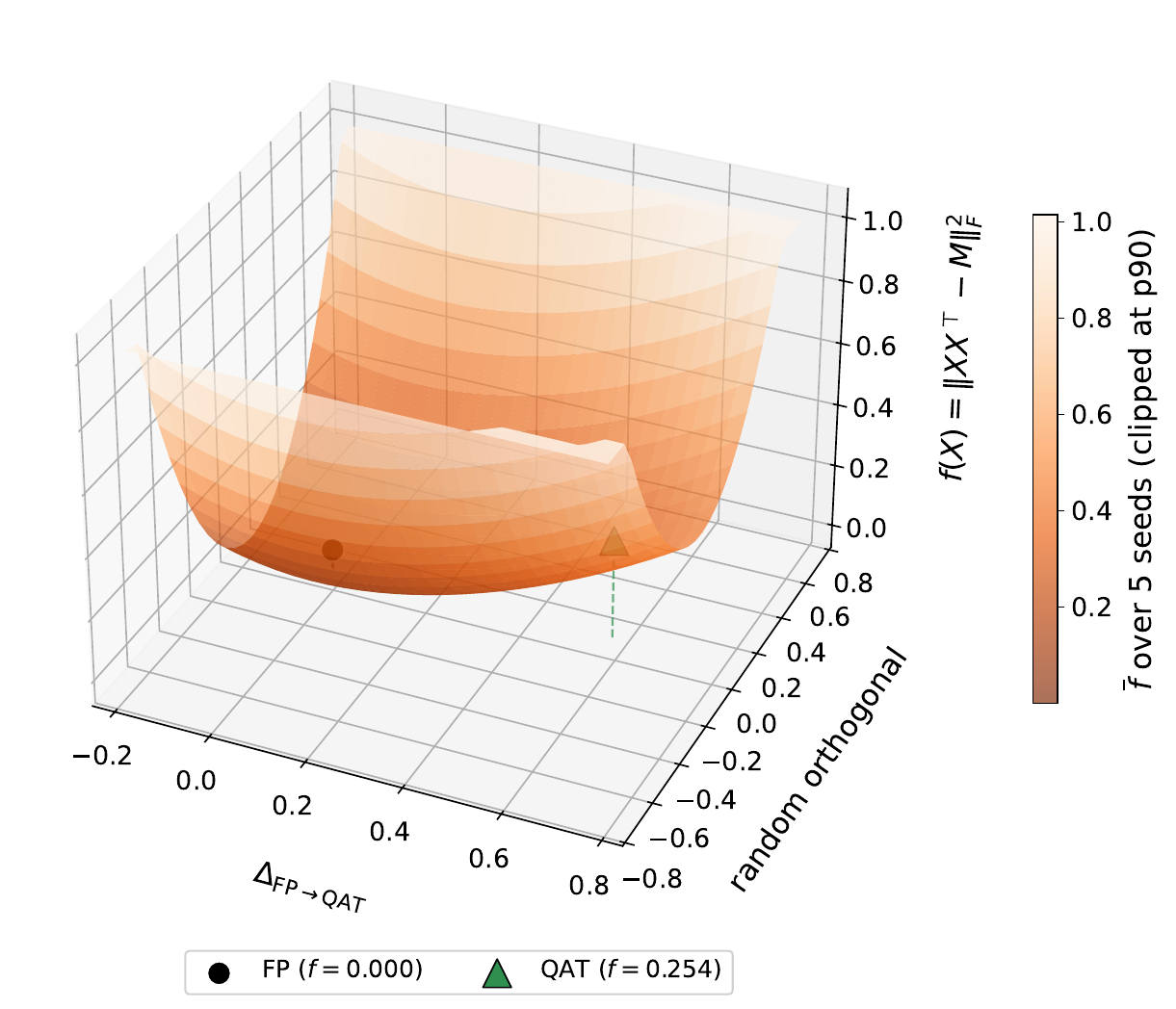}
\caption{\textbf{River-cross 3D loss landscape for the
high-dimensional Matrix Factorization simulation.}  Loss $f$ on the plane spanned by
$\Delta_{\rm FP\to QAT} = Q(X_{\rm QAT}) - X_{\rm fp}$ and a random
direction projected orthogonal to it and rescaled to
$\|\Delta_{\rm FP\to QAT}\|$, averaged pointwise over $5$ random
seeds (same construction as the ResNet/DeiT/Llama river-cross plots
in the main text).  The valley along $b = 0$ is narrow and deep:
loss stays low along $\Delta_{\rm FP\to QAT}$ (the river axis) but
rises sharply with $|b|$, directly consistent with
Assumption~\ref{ass:river-geometry} on this finite-dimensional
matrix factorization instance.}
\label{fig:mf-highdim-river}
\end{figure}

\FloatBarrier

\section{Proofs in Section~\ref{sec:Quant_under_River-valley-basin}}

\subsection{Proof of Theorem~\ref{thm:QAT-recovery}}
\begin{proof}[Proof of Theorem~\ref{thm:QAT-recovery}]
For brevity, write $q_k = Q(w_k)$ and $\pi_k = P_{\cM}(w_k)$, and $d_k = \|w_k-\pi_k\|_{\mathsf H}$. We prove the theorem in two steps.

\noindent\textbf{Step 1:} We prove by induction that for every $k \le T$,%
\begin{enumerate}
    \item[(a)] $w_k \in \cT$ and $q_k \in U$,
    \item[(b)] $\{\delta w_k + (1-\delta) q_k \mid \delta \in [0,1]\} \subset U$,
    \item[(c)] $\{w_k - \delta \,\eta\nabla f(q_k) \mid \delta \in [0,1]\} = \{(1-\delta) w_k + \delta w_{k+1} \mid \delta \in[0,1]\} \subset U$.
\end{enumerate}
For the base case $k=0$, it follows from our assumption that $w_0 = w_{\rm fp} \in \cT$ and $q_0 \notin \cT$. By Assumption~\ref{ass:quantizer-compatibility}(i), $\|q_0 - w_0\| \le \rho$ and, thus, for any $\delta \in [0,1]$,
\[
    \dist(\delta w_0 + (1-\delta) q_0, \cM) 
    \le (1-\delta)\|q_0-w_0\| + \dist(w_0, \cM) 
    \le \rho + R/\sqrt{\lambda_d} 
    < R_U.
\]
So the open set $U$ contains the segment $\{\delta w_k + (1-\delta) q_k \mid \delta \in [0,1]\}$ and, in particular, $q_0 \in U\setminus\cT$. The remaining property (c) is a direct consequence of {$\eta \le \rho/G$} and $\|\nabla f(q_0)\| \le G$.

For the inductive step with $k+1 \le T$, suppose $w_k \in \cT$, $q_k \in U$, and that $U$ covers the line segment connecting $w_k$ and $q_k$. Since $k \le T-1$, we have $q_k \in U\setminus\cT$.

Below we will show that $q_k \notin \cT$ implies progress in the normal direction, i.e., $d_{k+1} \le d_k$ and thus, $w_{k+1} \in \cT$ given $d_k \le R$. It then follows from $\|q_{k+1}-w_{k+1}\| \le \rho$ that
\begin{align*}
    \dist(\delta w_{k+1} + (1-\delta) q_{k+1}, \cM)
    \le&\; (1-\delta) \|q_{k+1} - w_{k+1}\| + \dist(w_{k+1}, \cM) \\
    \le&\; \rho + R/\sqrt{\lambda_d}
    < R_U.
\end{align*}
Hence, $\{\delta w_{k+1} + (1-\delta) q_{k+1} \mid \delta \in [0,1]\} \subset U$ and $q_{k+1} \in U$. Similarly, $\|\eta \nabla f(q_{k+1})\| \le \eta G \le \rho$ implies that $w_{k+2} \in U$ and $\{w_{k+1} - \delta \,\eta\nabla f(q_{k+1}) \mid \delta \in [0,1]\} \subset U$. So properties (a)-(c) hold for $k+1 \le T$, completing the induction.

\noindent\textbf{Progress in the normal direction when $q_k \notin \cT$.}
Consider $\Phi(w)= \frac12 \|w - \pi(w)\|^2_{\mathsf H}$. Since $P_{\cM}$ is $C^1$ on $U$ and the range of $\nabla P_{\cM}(w)$ lies in $T_{\cM}(P_{\cM}(w))$ while $\mathsf H(w - P_{\cM}(w)) \in \mathsf H N_{\cM}(P_{\cM}(w)) \subset N_{\cM}(P_{\cM}(w))$, we have
\[
    \nabla \Phi(w) = {\mathsf H}(w-P_{\cM}(w)) - \nabla P_{\cM}(w)^\top[\mathsf H(w-P_{\cM}(w))]
    = \mathsf H(w-P_{\cM}(w)).
\]
It follows from the $L_P$-Lipschitz continuity of $P_{\cM}(\cdot)$ that $\nabla\Phi(\cdot)$ is $\lambda_1(1+L_P)$-Lipschitz continuous on $U$. Thus, for any $w, w^\prime \in U$ such that $\{(1-\delta)w_k + \delta w_{k+1} \mid \delta \in [0,1]\} \subset U$, it follows from the descent lemma that
\[
    \Phi(w^\prime) \le \Phi(w) + \langle \nabla\Phi(w), w^\prime - w\rangle + \frac{\lambda_1(1+L_P)}{2} \|w^\prime - w\|^2.
\]
Notice that the segment $\{(1-\delta) w_k + \delta w_{k+1} \mid \delta \in[0,1]\} \subset U$ due to
\begin{align*}
    \dist((1-\delta)w_k + \delta w_{k+1}, \cM) 
    &\le \delta \|w_{k+1} - w_k\| + \dist(w_k, \cM) \\
    &\le \delta \eta G + R/\sqrt{\lambda_d} 
    \,\le\, \rho + R/\sqrt{\lambda_d} 
    \,<\, R_U.
\end{align*}
Therefore, we can take $w=w_k$ and $w^\prime=w_{k+1}$ in the descent lemma to derive that
\begin{equation}\label{eq:taylor-Phi}
\begin{split}
    \Phi(w_{k+1}) 
    &\le \Phi(w_k) + \langle \nabla\Phi(w_k),w_{k+1} - w_k\rangle + \frac{\lambda_1(1+L_P)}{2} \|w_{k+1} - w_{k}\|^2 \\
    &\le \Phi(w_k) - \eta \langle \mathsf H(w_k - P_{\cM}(w_k)),\nabla f(q_k)\rangle + \frac{\lambda_1(1+L_P)G^2}{2} \eta^2.
\end{split}
\end{equation}
Next we bound the inner product:
\begin{align*}
    &\langle \mathsf H(w_k - P_{\cM}(w_k)), \nabla f(q_k)\rangle \\[0.06in]
    =&\; \langle \mathsf H(q_k - P_{\cM}(q_k)), \nabla f(q_k)\rangle + \left\langle \mathsf H(w_k - q_k) - \mathsf H(P_{\cM}(w_k) - P_{\cM}(q_k)), \nabla f(q_k)\right\rangle \\[0.06in]
    \ge&\; c_{\perp} \|\mathsf H(q_k - P_{\cM}(q_k))\| \\
       &\qquad - \lambda_1(\|w_k - q_k\| + \|P_{\cM}(w_k) - P_{\cM}(q_k)\|)\|\nabla f(q_k)\| \tag*{(Assumption~\ref{ass:river-geometry}(iii))} \\[0.06in]
    \ge&\; c_{\perp} \sqrt{\lambda_d} \|q_k - P_{\cM}(q_k)\|_{\mathsf H} - \lambda_1(1+L_{P}) \|w_k - q_k\| \|\nabla f(q_k)\| \tag*{($L_P$-Lipschitz continuity of $P_{\cM}$)} \\[0.06in]
    \ge&\; c_{\perp} \sqrt{\lambda_d} \|q_k - P_{\cM}(q_k)\|_{\mathsf H} - \lambda_1(1+L_{P}) \rho G \tag*{(Assumptions~\ref{ass:quantizer-compatibility}(i), (ii))} \\[0.06in]
    \ge&\; c_{\perp} \sqrt{\lambda_d} R - \lambda_1(1+L_{P}) \rho G. \tag*{($q_k \notin \cT$)}
\end{align*}
Substituting this into \eqref{eq:taylor-Phi} and using {$\eta \le \frac{c_\perp \sqrt{\lambda_d} R - \lambda_1(1+L_P)\rho G}{\lambda_1(1+L_P)G^2}$}, we obtain
\begin{align*}
    \Phi(w_{k+1}) 
    &\le \Phi(w_k) - \eta \left(c_{\perp} \sqrt{\lambda_d}R - \lambda_1(1+L_{P}) \rho G\right) + \frac{\lambda_1(1+L_P)G^2}{2} \eta^2 \\
    &\le \Phi(w_k) - \frac{1}{2} \eta \left(c_{\perp} \sqrt{\lambda_d}R - \lambda_1(1+L_{P}) \rho G\right).
\end{align*}
In particular, $d_{k+1} \le d_k$ for all $k < T$.

\noindent\textbf{Step 2: Eventually, the quantized iterate enters $\cT$.}
If $\|w_k - \pi_k\|_{\mathsf H} \leq R- \sqrt{\lambda_1}\rho(1 + L_P)$, then, by the triangle inequality,
\begin{align*}
    \|q_k - P_{\cM}(q_k)\|_{\mathsf H}
    \le&\; \|q_k - w_k\|_{\mathsf H} + \|w_k - \pi_k\|_{\mathsf H} + \|\pi_k - P_{\cM}(q_k)\|_{\mathsf H} \\
    \le&\; \sqrt{\lambda_1} \|q_k - w_k\| + \|w_k - \pi_k\|_{\mathsf H} + \sqrt{\lambda_1} \|P_{\cM}(w_k) - P_{\cM}(q_k)\| \\
    \le&\; \|w_k - \pi_k\|_{\mathsf H} + \sqrt{\lambda_1} \rho(1+L_P) 
    \,\le\,R,
\end{align*}
 $q_k \in \cT$. Therefore, while $q_k \notin \cT$, the squared distance $\Phi(w_k) = \frac{1}{2} \|w_k - \pi_k\|^2_{\mathsf H}$ decreases by at least $\frac{\eta}{2} \big(c_{\perp}\sqrt{\lambda_d}R - \lambda_1(1+L_{P}) \rho G\big)$. Thus,
\[
    T \le 
    1 + \frac{\max\left\{d_0^2 - \left(R- \sqrt{\lambda_1}\rho(1+L_P)\right)^2, 0\right\}}
    {\eta \left( c_{\perp}\sqrt{\lambda_d}R - \lambda_1(1+L_{P}) \rho G\right)}.
\]
\end{proof}

\subsection{Proofs of Corollary~\ref{cor:QAT_loss1} and Corollary~\ref{cor:QAT_loss2}}
Before the proofs, we first introduce a standard property of the nearest-point projection.
For any $\pi\in\mathcal M$, it holds that
\[
    \nabla P_{\mathcal M}(\pi)=P_{T_{\mathcal M}(\pi)}.
\]
Indeed, for any $z\in\mathbb R^d$, decompose
$z=z_T+z_N$, where $z_T\in T_{\mathcal M}(\pi)$ and
$z_N\in N_{\mathcal M}(\pi)$. For the tangent component, choose a
smooth curve $\gamma\subset\mathcal M$ with
$\gamma(0)=\pi$ and $\dot\gamma(0)=z_T$. Since
$P_{\mathcal M}$ restricts to the identity map on $\mathcal M$,
differentiating $P_{\mathcal M}(\gamma(t))=\gamma(t)$ at $t=0$
gives $\nabla P_{\mathcal M}(\pi)z_T=z_T$.
For the normal component, since $P_{\mathcal M}(\pi+t z_N)=\pi$ for all sufficiently small $t$, differentiating at $t=0$ gives $\nabla P_{\mathcal M}(\pi)z_N=0$.
Therefore $\nabla P_{\mathcal M}(\pi)z = z_T = P_{T_{\mathcal M}(\pi)}z$.

We also have  Lemma \ref{lem:projection} about projection linearization around the river.

\begin{lemma}%
\label{lem:projection}
Suppose $P_{\cM}:U\to\cM$ is $C^{1,1}$ with $\|\nabla P_{\cM}(z)-\nabla P_{\cM}(z')\|_{\rm op} \le \kappa_{\cM}\|z-z'\|$ for any $z,z'\in U$.
Let $w \in U$ and set $\pi=P_{\cM}(w)$. If $\{w+t\delta \mid t\in[0,1]\} \subset U$, then
\[
    \left\|P_{\cM}(w+\delta)-P_{\cM}(w) - P_{T_{\cM}(\pi)}\delta\right\| 
    \le \kappa_{\cM} \left(\|w -\pi\| \|\delta\| + \frac12\|\delta\|^2\right).
\]
\end{lemma}

\begin{proof}[Proof of Lemma~\ref{lem:projection}]
Note that 
$
    P_{\cM}(w+\delta)-P_{\cM}(w)
    =
    \int_0^1 \nabla P_{\cM}(w+t\delta)\delta\,dt
$.
Since $\pi=P_{\cM}(w)$ and $\nabla P_{\cM}(\pi)=P_{T_{\cM}(\pi)}$, we have
\[
    P_{\cM}(w+\delta)-P_{\cM}(w) - P_{T_{\cM}(\pi)}\delta 
    = \int_0^1 \left[\nabla P_{\cM}(w+t\delta)-\nabla P_{\cM}(\pi)\right]\delta\,dt.
\]
Using the $C^{1,1}$ regularity of $P_{\cM}$ gives
\[
\begin{aligned}
    \left\|P_{\cM}(w+\delta)-P_{\cM}(w) - P_{T_{\cM}(\pi)}\delta\right\|
    &\le \int_0^1 \kappa_{\cM}\|w+t\delta-\pi\|\|\delta\|\,dt \\
    &\le \int_0^1 \kappa_{\cM} \left(\|w-\pi\|+t\|\delta\|\right)\|\delta\|\,dt \\
    &= \kappa_{\cM} \left(\|w-\pi\| \|\delta\| + \frac12\|\delta\|^2\right),
\end{aligned}
\]
completing the proof.
\end{proof}

Now we are ready to prove Corollary~\ref{cor:QAT_loss1}.
\begin{proof}[Proof of Corollary~\ref{cor:QAT_loss1}]
Write $q_k = Q(w_k)$ and $\pi_k = P_{\cM}(w_k)$, and $d_k = \|w_k-\pi_k\|_{\mathsf H}$.
By the definition of $T$, $q_k\notin\mathcal T$ for all $k<T$,
and $q_T\in\mathcal T$. Moreover, the proof of
Theorem~\ref{thm:QAT-recovery} gives $w_k\in\mathcal T$ for all
$k\le T$. Hence $d_k\le R$.

We first derive a common projected-loss estimate. Since the segment between $w_k$ and $w_{k+1}$ is contained in $U$, by Lemma~\ref{lem:projection},
\[
    \pi_{k+1}-\pi_k
    = -\eta P_{T_{\mathcal M}(\pi_k)}\nabla f(q_k) + r_k,
\]
where
\begin{align*}
    \|r_k\|
    \,\le\, \kappa_{\mathcal M}
    \left(
        \eta \|w_k-\pi_k\| \|\nabla f(q_k)\| + \frac12\eta^2\|\nabla f(q_k)\|^2
    \right) %
    \,\le\, \eta\kappa_{\mathcal M}G \left(R/\sqrt{\lambda_d} + \frac12\eta G\right).
\end{align*}
Letting $\Delta = \kappa_{\mathcal M} \left(R/\sqrt{\lambda_d} + \frac12\eta G\right)$, we have $\|r_k\| \le \eta G \Delta$. Consequently,
\begin{align*}
    \|\pi_{k+1}-\pi_k\|
    \,\le\, \eta \left\|P_{T_{\mathcal M}(\pi_k)}\nabla f(q_k)\right\| + \|r_k\| %
    \,\le\, \eta \|\nabla f(q_k)\| + \|r_k\|
    \,\le\, \eta G(\Delta + 1).
\end{align*}
By $L$-smoothness of $f$,
\[
\begin{aligned}
    f(\pi_{k+1})
    &\le f(\pi_k)
    + \langle \nabla f(\pi_k),\pi_{k+1}-\pi_k\rangle
    + \frac L2\|\pi_{k+1}-\pi_k\|^2 \\
    &= f(\pi_k)
    - \eta \left\langle \nabla f(\pi_k), P_{T_{\mathcal M}(\pi_k)}\nabla f(q_k) \right\rangle
    + \langle \nabla f(\pi_k),r_k\rangle
    + \frac L2\|\pi_{k+1}-\pi_k\|^2.
\end{aligned}
\]
Noting that
\[
    \langle \nabla f(\pi_k), P_{T_{\mathcal M}(\pi_k)}\nabla f(q_k) \rangle
    = \langle P_{T_{\mathcal M}(\pi_k)} \nabla f(\pi_k), \nabla f(q_k) \rangle
    = \langle \nabla f(\pi_k), \nabla f(q_k) \rangle,
\]
and using the bounds on $r_k$ and
$\|\pi_{k+1}-\pi_k\|$, we further obtain
\begin{equation}\label{eq:central-projected-loss}
    f(\pi_{k+1})
    \le f(\pi_k) - \eta
    \left\langle \nabla f(\pi_k), \nabla f(q_k)
    \right\rangle + \eta G\Delta\|\nabla f(\pi_k)\| + \frac{L}{2} \eta^2 G^2 (\Delta+1)^2.
\end{equation}

The additional assumption of Corollary~\ref{cor:QAT_loss1} gives us $\|\nabla f(\pi_k)\|\le \epsilon$ since $\pi_k \in \cM$. Therefore,
\[
    \left|\left\langle \nabla f(\pi_k), \nabla f(q_k) \right\rangle\right|
    \le \|\nabla f(\pi_k)\|\|\nabla f(q_k)\|
    \le \epsilon G.
\]
Plugging this bound into \eqref{eq:central-projected-loss} yields
\[
    f(\pi_{k+1})
    \le f(\pi_k)
    + \eta\epsilon G (\Delta+1)
    + \frac{L}{2} \eta^2 G^2 (\Delta+1)^2.
\]
Summing over $k=0,\dots,T-1$, we obtain
\begin{equation}\label{eq:case-a-piT}
    f(\pi_T)
    \le f(\pi_0) + \eta T G (\Delta+1) \left(\epsilon + \frac{L}{2} \eta G (\Delta+1)\right).
\end{equation}
We next compare $f(P_{\mathcal M}(q_T))$ with $f(\pi_T)$. Since the segment $\{\delta q_T + (1-\delta) w_T \mid \delta \in [0,1]\} \subset U$, by the mean value
theorem, there exists $\widetilde w_T \in \{\delta w_T+(1-\delta)q_T \mid \delta\in[0,1]\}$
such that
\[
    f(P_{\mathcal M}(q_T))-f(\pi_T)
    =
    \left\langle
    \nabla(f\circ P_{\mathcal M})(\widetilde w_T),
    q_T-w_T
    \right\rangle.
\]
Using the chain rule, $\nabla(f\circ P_{\mathcal M})(\widetilde w_T) = \nabla P_{\mathcal M}(\widetilde w_T)^\top [\nabla f(P_{\mathcal M}(\widetilde w_T))]$.
Note that
\[
\begin{aligned}
    \|\nabla(f\circ P_{\mathcal M})(\widetilde w_T)\|
    \le\;& \|\nabla P_{\mathcal M}(\widetilde w_T)-\nabla P_{\mathcal M}(\pi_T)\|_{\rm op} \|\nabla f(P_{\mathcal M}(\widetilde w_T))\| \\
    &\quad + \|\nabla P_{\cM}(\pi_T) [\nabla f(P_{\mathcal M}(\widetilde w_T))]\|.
\end{aligned}
\]
Since $\nabla P_{\mathcal M}(\pi_T)=P_{T_{\mathcal M}(\pi_T)}$, we have
\begin{align*}
    \|\nabla P_{\cM}(\pi_T) [\nabla f(P_{\mathcal M}(\widetilde w_T))]\| 
    \,=\, \|P_{T_{\mathcal M}(\pi_T)} [\nabla f(P_{\mathcal M}(\widetilde w_T))]\| %
    \,\le\, \|\nabla f(P_{\mathcal M}(\widetilde w_T))\|
    \,\le\, \epsilon.
\end{align*}
Thus,
\begin{align*}
    \|\nabla(f\circ P_{\mathcal M})(\widetilde w_T)\|
    \le\; \kappa_{\mathcal M}\|\widetilde w_T-\pi_T\|\epsilon+\epsilon %
    \,\le\, \epsilon(\kappa_{\cM} (R/\sqrt{\lambda_d}+\rho)+1).
\end{align*}
Therefore,
\begin{equation}\label{eq:case-a-from_Proj_qT_to_Proj_piT}
    f(P_{\mathcal M}(q_T))-f(\pi_T)
    \le \epsilon\rho\bigl(\kappa_{\mathcal M}(R/\sqrt{\lambda_d}+\rho)+1\bigr).
\end{equation}
Since $w_{\rm fp}, q_T\in\mathcal T$, Assumption~\ref{ass:river-geometry}(i) gives
\[
    f(q_T)\le f(P_{\mathcal M}(q_T))+\epsilon_{\rm flat}, \qquad
    f(\pi_0)\le f(w_{\rm fp})+\epsilon_{\rm flat}.
\]
Combining these two inequalities with \eqref{eq:case-a-piT} and \eqref{eq:case-a-from_Proj_qT_to_Proj_piT}, we obtain
\[
    f(q_T)
    \le f(w_{\rm fp})
    + 2\epsilon_{\rm flat}
    + \epsilon\rho\bigl(\kappa_{\mathcal M}(R/\sqrt{\lambda_d} + \rho)+1\bigr)
    + \eta T G (\Delta+1) \left(\epsilon + \frac{L}{2} \eta G (\Delta+1)\right).
\]
This completes the proof.
\end{proof}

We now proceed to Corollary~\ref{cor:QAT_loss2}. We will reuse the intermediate result \eqref{eq:central-projected-loss} in the proof of Corollary~\ref{cor:QAT_loss1}, which is independent of the additional assumptions in Corollary~\ref{cor:QAT_loss1}.

\begin{proof}[Proof of Corollary~\ref{cor:QAT_loss2}]
For $k<T$, the assumption gives $\left\langle g(P_{\cM}(q_k)), \nabla f(q_k)\right\rangle > c_\parallel$.
Nevertheless, we need to estimate $\langle \nabla f(\pi_k), \nabla f(q_k)\rangle = \|\nabla f(\pi_k)\| \langle g(\pi_k), \nabla f(q_k) \rangle$ in \eqref{eq:central-projected-loss}. So next we bound the difference. 
Taking $w=\pi \in \cM$ in the assumption yields $\|\nabla f(\pi)\| > c_\parallel$. Thus,
\[
    \min\{\|\nabla f(\pi_k)\|, \|\nabla f(P_{\cM}(q_k))\|\} > c_\parallel.
\]
Since the normalized gradient direction $g(\cdot)$ is $\kappa$-Lipschitz,
\begin{align*}
    \left\| g(\pi_k) - g(P_{\cM}(q_k))\right\|
    \,\le\; \kappa \|\pi_k - P_{\cM}(q_k)\| %
    \,\le\; \kappa L_P \|w_k-q_k\|
    \,\le\; \kappa \rho L_P.
\end{align*}
Therefore,
\begin{align*}
    \left\langle g(\pi_k), \nabla f(q_k)\right\rangle
    \,\ge\; \langle g(P_{\cM}(q_k)), \nabla f(q_k)\rangle - \|g(\pi_k) - g(P_{\cM}(q_k))\| \cdot \|\nabla f(q_k)\| %
    \,\ge\; c_\parallel - \kappa \rho L_P G.
\end{align*}
Plugging this into \eqref{eq:central-projected-loss} gives
\[
    f(\pi_{k+1})
    \le f(\pi_k) - \eta \left(c_\parallel - \kappa \rho L_P G - G\Delta\right) \|\nabla f(\pi_k)\|
    + \frac{L}{2} \eta^2 G^2 (\Delta+1)^2.
\]
Using our assumptions that $\kappa \leq \frac{c_\parallel}{4\rho L_P G}$ and $\kappa_{\cM} \le \frac{c_\parallel}{4G(R/\sqrt{\lambda_d} + \rho)}$, and the upper bound for the stepsize $\eta \le \frac{c_\parallel^2}{2LG^2(\kappa_\cM(R/\sqrt{\lambda_d} + \rho)+1)^2}$, we arrive at
\begin{align*}
    f(\pi_{k+1})
    \le\, f(\pi_k) - \frac{1}{2} \eta \,c_\parallel^2 + \frac{L}{2} \eta^2 G^2 (\Delta+1)^2 %
    \le\, f(\pi_k) - \frac{1}{4} \eta \,c_\parallel^2.
\end{align*}
Summing over $k=0, 1, \cdots, T-1$ yields
\begin{equation}\label{eq:case-b-piT}
    f(\pi_T) \le f(\pi_0) - \frac{1}{4} \eta T c_\parallel^2.
\end{equation}
We next compare $f(P_{\mathcal M}(q_T))$ with $f(\pi_T)$. Since
$\{\delta w_T + (1-\delta)q_T \mid \delta \in [0,1]\} \subset U$, applying the mean-value theorem to $f\circ P_{\mathcal M}$ gives
\[
    f(P_{\mathcal M}(q_T)) - f(\pi_T) \le \sup_{\delta\in[0,1]} \big\|\nabla (f \circ P_{\cM})\big(\delta w_T + (1-\delta)q_T\big)\big\| \cdot \|q_T - w_T\|.
\]
By the chain rule, $\nabla (f \circ P_{\cM})(w) = \nabla P_{\mathcal M}(w)^\top \nabla f(P_{\mathcal M}(w))$.
Since $P_{\mathcal M}$ is $L_P$-Lipschitz on $U$ and
$\|\nabla f(w)\|\le G$ for all $w \in U$, we have $\|\nabla (f \circ P_{\cM})(w)\|\le L_P G$.
Therefore,
\begin{equation}\label{eq:case-b-from_Proj_qT_to_Proj_piT}
    f(P_{\mathcal M}(q_T)) - f(\pi_T) \le L_P G \|q_T - w_T\|
    \le L_P G\rho.
\end{equation}
Combining \eqref{eq:case-b-piT}, \eqref{eq:case-b-from_Proj_qT_to_Proj_piT}, and the facts that $f(q_T)\le f(P_{\mathcal M}(q_T))+\epsilon_{\rm flat}$ and $f(\pi_0)\le f(w_{\rm fp})+\epsilon_{\rm flat}$ by Assumption~\ref{ass:river-geometry}(i), we conclude that
\[
    f(q_T)
    \le f(w_{\rm fp})
    + 2\epsilon_{\rm flat}
    + G L_P \rho
    - \frac{1}{4} \eta T\, c_\parallel^2.
\]
This completes the proof.
\end{proof}

\end{document}